\documentclass{article}

\usepackage{arxiv}

\usepackage[utf8]{inputenc} 
\usepackage[T1]{fontenc}    
\usepackage{hyperref}       
\usepackage{url}            
\usepackage{booktabs}       
\usepackage{amsfonts}       
\usepackage{nicefrac}       
\usepackage{microtype}      
\usepackage{xcolor}         

\usepackage{bm}
\usepackage{bbm}
\usepackage{comment}         
\usepackage{multicol}
\usepackage{multirow}
\usepackage{caption}
\usepackage{natbib}
\usepackage{braket}

\usepackage{amsmath,amssymb,amsthm}

\theoremstyle{plain}

\newtheorem{theorem}{Theorem}[section]
\newtheorem{lemma}[theorem]{Lemma}

\newtheorem{corollary}[theorem]{Corollary}
\newtheorem{proposition}[theorem]{Proposition}

\newtheorem{definition}[theorem]{Definition}
\newtheorem{assumption}[theorem]{Assumption}
\newtheorem*{example}{Example}
\newtheorem*{remark}{Remark}

\usepackage{color}

\usepackage{graphicx}

\newcommand{\Alt}{\mathrm{Alt}}

\DeclareMathOperator*{\argmax}{arg\,max}

\usepackage{times}

\newcommand{\iid}{ \stackrel{\mathrm{i.i.d}}{\sim} }

\newcommand{\pa}{\mathrm{\pa}}

\newcommand{\RN}[1]{%
  \textup{\uppercase\expandafter{\romannumeral#1}}%
}

\usepackage{xcolor}
\newcount\Comments  
\newcommand{\kibitz}[2]{\ifnum\Comments=1\textcolor{#1}{#2}\fi}

\usepackage{algorithm}
\usepackage{algorithmic}

\newcommand*{\email}[1]{\texttt{#1}}

\usepackage{authblk}
\allowdisplaybreaks

\title{Best Arm Identification with Contextual Information under a Small Gap}

\renewcommand{\shorttitle}{Best Arm Identification with Contextual Information under a Small Gap}

\author[1,2]{Masahiro Kato\thanks{\email{masahiro\_kato@cyberagent.co.jp}.}$\ \ $}
\author[1]{Masaaki Imaizumi}
\author[3]{Takuya Ishihara}
\author[4]{Toru Kitagawa}

\affil[1]{Department of Basic Science, the University of Tokyo}
\affil[2]{AI Lab, CyberAgent, Inc.}
\affil[3]{Graduate School of Economics and Management, Tohoku University}
\affil[4]{Department of Economics, Brown University and Department of Economics, University College London}

\begin{document}

\maketitle

\begin{abstract}
We study the best-arm identification (BAI) problem with a fixed budget and contextual (covariate) information.  
In each round of an adaptive experiment, after observing contextual information, we choose a treatment arm using past observations and current context. Our goal is to identify the best treatment arm, which is a treatment arm with the maximal expected reward marginalized over the contextual distribution, with a minimal probability of misidentification. In this study, we consider a class of nonparametric bandit models that converge to location-shift models when the gaps go to zero. First, we derive lower bounds of the misidentification probability for a certain class of strategies and bandit models (probabilistic models of potential outcomes) under a small-gap regime. A small-gap regime is a situation where gaps of the expected rewards between the best and suboptimal treatment arms go to zero, which corresponds to one of the worst cases in identifying the best treatment arm. We then develop the ``Random Sampling (RS)-Augmented Inverse Probability weighting (AIPW) strategy,'' which is asymptotically optimal in the sense that the probability of misidentification under the strategy matches the lower bound when the budget goes to infinity in the small-gap regime. The RS-AIPW strategy consists of the RS rule tracking a target sample allocation ratio and the recommendation rule using the AIPW estimator. 
\end{abstract}

\section{Introduction}
This paper considers an adaptive experimental design to accurately identify the best treatment arm among a finite set of candidates after the last round of the experiment.
The stochastic multi-armed bandit (MAB) problem is a classical abstraction of the sequential decision-making problem \citep{Thompson1933,Robbins1952,Lai1985}. Best arm identification (BAI) is an instance of the MAB problem. In BAI, we consider pure exploration to identify the best treatment arm, which is a treatment arm that yields the highest expected reward. In this study, we study \emph{BAI with a fixed budget and contextual information}, whose goal is to identify the best treatment arm minimizing the probability of misidentifying the best treatment arm after a fixed number of rounds of adaptive experiments, called a sample size or a \textit{budget}
\citep{Bubeck2009,Bubeck2011,Audibert2010}.  To gain efficiency in this task, during an adaptive experiment, one strategy is to employ a random covariate (side information) that characterizes the features of treatment arms, which is referred to as a context. A context can be observed before drawing one of treatment arms in an adaptive experiment. Based on observed context, a strategy chooses a treatment arm. Our setting is a generalization of BAI with a fixed budget \citep{Carpentier2016}. The main focus of this paper is to derive asymptotically optimal strategies for the purpose of efficiently identifying the best treatment arm, rather than how to learn optimal context-specific bandit strategies, as studied in the literature of contextual bandit. 

One of the main research interests in bandit problems is clarifying a tight lower bound on the probability of misidentification (a theoretical limit of performance). In this study, we call a strategy asymptotically optimal if, under the strategy, a probability of misidentification matches a lower bound as the budget goes to infinity. 
However, in BAI, it is unknown whether  a strategy exists under which a probability of misidentification matches the lower bound in BAI with a fixed budget\citep{Kaufman2016complexity}. 
Furthermore, when we can employ contextual information, even the lower bound has been unknown as well as an asymptotically optimal strategy, unlike BAI with fixed confidence \citep{Russac2021,Kato2021Role}, another setting of BAI \citep{Garivier2016}.

We develop an asymptotically optimal strategy for BAI with a fixed budget and contextual information under a small-gap regime, where the gaps of the expected rewards of the best and suboptimal treatment arms converge to zero.\footnote{This regime is also considered by \citet{Kato2022small}, and some of our results are employed and extended from their results.} This regime is one of the hardest situation to identify the best treatment arms among the other suboptimal treatment arms. First, we derive lower bounds for the probability of misidentification by extending the lower bound of \citet{Kaufman2016complexity} under the small-gap regime. Based on this lower bound, we derive an optimal target sample allocation ratio, which is the limit of the ratio of the number of samples allocated to each treatment arm within an adaptive experiment under an optimal strategy. Then, we propose our BAI strategy, the Random Sampling (RS)-Augmented Inverse Probability Weighting (AIPW) strategy, which consists of the RS rule using an estimated target sample allocation ratio and 
a recommendation rule using the AIPW estimator. We prove the asymptotic optimality of the proposed strategy 
when the budget goes to infinity under the small-gap regime.

The existence of an asymptotically optimal strategy has long been an open problem in this field.
\citet{glynn2004large} proposes optimal strategies based on optimally selected target sample allocation ratios. However, 
they assume that the optimal target sample allocation ratio is known in advance and do not consider the issue of estimating it. Based on the change-of-measure arguments popularized by \citet{Lai1985}, \citet{Kaufman2016complexity} derives lower bounds for the misidentification probability, which are agnostic to the optimal target sample allocation ratios. 
Despite the seminal result, optimal strategies has not been proposed for the lower bounds
\citet{kaufmann2020hdr}. In fact, when the gaps between the expected rewards of the best and suboptimal treatment arms are fixed, \citet{Carpentier2016} shows that there is no strategy under which the probability of misidentificaton matches the lower bounds derived by \citet{Kaufman2016complexity}. Thus, the debate over tight lower bounds and optimal algorithms is not settled, and various approaches have been proposed \citep{kaufmann2020hdr,Kasy2021,Ariu2021,Komiyama2022}. 

We consider that a contributing factor on the mismatch between the upper and lower bounds is an estimation error of the optimal target sample allocation ratio, which affects the probability of misidentification. To address this issue, we propose a small-gap regime for the following two reasons.
The first is that an optimal strategy under a small gap has an important practical implication itself because it means asymptotic optimality under one of the worst cases. The second reason is more technical. Under the fixed-gap regime, it has been shown that there is no asymptotically optimal strategy in the sense that its upper bounds match lower bounds derived in the way of \citet{Lai1985} \citep{Carpentier2016}. 
This regime makes the asymptotic optimality argument in fixed-budget BAI tractable by allowing the evaluation of optimal allocation probabilities to be ignored. 
Thus, this study addresses the open question by showing an asymptotically optimal strategy under a small-gap regime. In addition, we investigate the efficiency gain by using contextual information, which is a generalization of conventional BAI.
Furthermore, we demonstrate an analytical solution for the target sample allocation ratio, which has long been unknown.


In both settings of BAI with a fixed budget and fixed confidence, strategies using contextual information have not been sufficiently explored. In BAI with fixed confidence, recent studies have proposed the use of contextual information to identify a treatment arm with the highest expected reward marginalized over contextual information \citep{Kato2021Role,Russac2021}. 

Related problems have been frequently considered in studies of causal inference, which mainly discuss the efficient estimation of causal parameters such as the gap between expected outcomes of two treatment arms marginalized over the covariate (contextual) distribution \citep{Laan2008TheCA,Hahn2011,Meehan2020,Kato2020adaptive}, rather than BAI. The gap is also called the average treatment effect (ATE) in this literature \citep{imbens_rubin_2015}. 
In efficient ATE estimation with an adaptive experiment, the asymptotic variance of the estimator can be reduced with the help of covariate information. 

In this study, we find that when identifying the best treatment arm marginalized over contextual information in BAI with a fixed budget, we can improve the performance of a strategy by employing contextual information. To the best of our knowledge, our study is the first to consider BAI with contextual information in a fixed-budget setting.

\paragraph{Organization.} This paper is organized as follows. In Section~\ref{sec:problem_setting}, we formulate our problem. In Section~\ref{sec:lower_bounds}, we derive the general lower bounds for BAI with a fixed budget and the target sample allocation ratio under a small-gap regime. 
In Section~\ref{sec:track_aipw}, we propose the RS-AIPW strategy. Then, in Section~\ref{sec:asymp_opt}, we show that the proposed strategy is optimal in a sense that the upper bound for the probability of misidentification matches the lower bound. 
We introduce related work in \ref{sec:related} and discuss several topics in Sections~\ref{sec:discuss}. Finally, we present the proof of the lower bound in Section~\ref{sec:proof}.

\section{Problem Setting}
\label{sec:problem_setting}
We consider the following setting of BAI with a fixed budget and contextual information. Given a fixed number of rounds $T$, also called a budget, for each round $t = 1,2,\dots, T$, an agent observes a context (covariate) $X_t\in\mathcal{X}$ and chooses a treatment arm $A_t \in [K] = \{1,2,\dots, K\}$, where $\mathcal{X}\subset \mathbb{R}^d$ denotes the context space. Then, the agent immediately receives a reward (or outcome) $Y_t$ linked to the chosen treatment arm $A_t$. This setting is called the bandit feedback or Rubin causal model \citep{Neyman1923,Rubin1974}; that is, a reward in round $t$ is $Y_t= \sum_{a\in[K]}\mathbbm{1}[A_t = a]Y^a_{t}$, where $Y^a_{t}\in\mathbb{R}$ is a potential independent (random) reward, and $Y^1_t,Y^2_t,\dots, Y^K_t$ are conditionally independent given $X_t$. We assume that $X_t$ and $Y^a_{t}$ are independent and identically distributed (i.i.d.) over $t \in [T] = \{1,2,\dots, T\}$. Our goal is to find a treatment arm with the highest expected reward marginalized over contextual distribution of $X_t$ with a minimal probability of misidentification after observing the reward in the round $T$.

We define our goal formally. Let $P$ be a joint distribution of $(Y^1, Y^2, \dots, Y^K, X)$, and $(Y^1_t, Y^2_t, \dots, Y^K_t, X_t)$ be an i.i.d. copy of $(Y^1, Y^2, \dots, Y^K, X)$ at round $t$. We call distributions of the potential random variables $(Y^1, Y^2, \dots, Y^K, X)$ full-data bandit models \citep{Tsiatis2007semiparametric,imbens_rubin_2015}. For $P$, let $\mathbb{P}_{P}$, $\mathbb{E}_{P}$, and $\mathrm{Var}_{P}$ be the probability, expectation, and variance in terms of $P$ respectively and $\mu^a(P) = \mathbb{E}_{P}[Y^a] = \mathbb{E}_P[\mu^a(P)(X) ]$ be the expected reward marginalized over the context $X$, where $\mu^a(P)(x) = \mathbb{E}_{P}[Y^a|X=x]$ is the conditional expected reward given $x\in\mathcal{X}$. Let $\mathcal{P}$ be a set of all joint distributions $P$ such that the the best treatment arm $a^*(P)$ uniquely exists; that is, there exists $a^*(P)\in[K]$ such that $\mu^{a^*(P)} > \max_{b \in [K]\backslash a^*(P)} \mu^b$. An algorithm in BAI is called a \emph{strategy}, which recommends a treatment arm $\widehat{a}_T \in [K]$ after sequentially sampling treatment arms in $t = 1,2,\dots,T$. With the sigma-algebras $\mathcal{F}_{t} = \sigma(X_1, A_1, Y_1, \ldots, X_{t}, A_t, Y_t)$, we define a BAI strategy as a pair $ ((A_t)_{t\in[T]}, \widehat{a}_T)$, where
\begin{itemize}
    \setlength{\parskip}{0cm}
  \setlength{\itemsep}{0cm}
    \item the sampling rule chooses a treatment arm $A_t \in [K]$ in each round $t$ based on the past observations $\mathcal{F}_{t-1}$ and observed context $X_t$. 
    \item the recommendation rule returns an estimator $\widehat{a}_T$ of the best treatment arm $\widehat{a}^*(P)$ based observations up to round $T$. Here, $\widehat{a}_T$ is $\mathcal{F}_T$-measurable.
\end{itemize}
Let $P_0$ be the ``true'' bandit model of the data generating process. 
Then, our goal is to find a BAI strategy that minimizes the probability of misidentification $\mathbb{P}_{P_0}( \widehat{a}_T \neq a^*(P_0))$.

\paragraph{Notation.} 
For all $a\in[K]$ and $x\in\mathcal{X}$, let $\nu^a(P)(x) = \mathbb{E}_{P}[(Y^a)^2|x]$ and $\mathrm{Var}_P(Y^a|x) = \left(\sigma^a(P)(x)\right)^2$. 
For the true bandit model $P_0\in\mathcal{P}$, we denote $\mu^a(P_0) = \mu^a_0$, $\mu^a(P_0)(x) = \mu^a_0(x)$, $\nu^a_0(x) = \nu^a(P_0)(x)$, and $\sigma^a_0(x) = \sigma^a(P_0)(x)$. Let $Y^{a^*_0}_t = Y^*_t$,  $a^*(P_0) = a^*_0$, $\mu^{a^*_0}_0 = \mu^*_0$, and $\nu^{a^*_0}_0 = \nu^*_0$. For the two Bernoulli distributions with mean parameters $\mu, \mu' \in [0, 1]$, we denote the KL divergence by $d(\mu, \mu') =\mu\log (\mu/\mu') + (1-\mu)\log((1-\mu)/(1-\mu'))$ with the convention that $d(0, 0) = d(1,1) = 0$. 

\section{Lower Bounds}
\label{sec:lower_bounds}
In this section, we derive lower bounds for the probability of misidentification $\mathbb{P}_{P_0}( \widehat{a}_T \neq a^*_0)$ under a small gap; that is, $\mu^*_0- \mu^{a}_0 \to 0$ for all $a\in [K]$. Our lower bounds are extensions of distribution-dependent lower bounds shown by \citet{Kaufman2016complexity}. We call them the lower bounds under a small-gap. We derive the lower bounds for two-armed bandits and bandits with more than three arms separately.

\subsection{Lower Bounds for Locally Location-shift Bandit Models}
First, the following conditions for a class of the bandit model $\mathcal{P}$ are assumed throughout this study.
\begin{assumption} 
\label{asm:bounded_mean_variance}
For all $P, Q \in \mathcal{P}$ and $a\in[K]$, let $P^a$ and $Q^a$ be the joint distributions of $(Y^a, X)$ of an treatment arm $a$ under $P$ and $Q$, respectively. The distributions $P^a$ and $Q^a$ are mutually absolutely continuous and have density functions with respect to some Lebesgue measure $m$. The potential outcome $Y^a$ has the first and second moments conditioned on $x\in\mathcal{X}$. 
There exist known constants $C_{\mu}, C_\nu, C_{\sigma^2} > 0$ such that, for all $P \in \mathcal{P}$, $a \in [K]$, and $x\in\mathcal{X}$, $| \mu^a(P)(x)| \le  C_\mu$, $|\nu^a(P)(x)| < C_\nu$ and $\max\{ 1/\left(\sigma^a(P)(x)\right)^2, \left(\sigma^a(P)(x)\right)^2\} \leq C_{\sigma^2}$ for all $x\in\mathcal{X}$. 
\end{assumption}

For a class of bandit models, we 
consider the locally location-shift class class defined as follows.
\begin{definition}[Locally location-shift bandit class]
\label{def:ls_bc}
A class of bandit models $\mathcal{P}$ is a location-shift bandit class if (i) for any $x\in\mathcal{X}$, there exist constants $\mu(x)\in\mathbb{R}$ and $\sigma^a(x) > 0$ such that for any $x\in\mathcal{X}$, any $P\in\mathcal{P}$, and all $a\in[K]$, as $\mu^{a}(P) \to \mu(x)$, 
\[\left|\left(\sigma^a(P)(x)\right)^2 - \left(\sigma^a(x)\right)^2\right| = o\left(\mu^{a^*(P)}(P) - \mu(x) \right),\] and (ii) for any $P\in\mathcal{P}$, the distribution of $X_t$ is the same.
\end{definition}
Let $\zeta(x)$ be a density of $X_t$ under a location-shift bandit class. Then, according to the law of total variance, the (unconditional) variance $\left(\sigma^{a}(P)\right)^2$ of $Y^a_t$ given as 
\[\left(\sigma^{a}(P)\right)^2 = \int \left(\sigma^a(x)\right)^2 \zeta(x)\mathrm{d}x + \int \left(\mu^a(P)(x) - \mu^a(P)\right)^2\zeta(x)\mathrm{d}x\geq \int \left(\sigma^a(x)\right)^2 \zeta(x)\mathrm{d}x. \]
for all $a\in[K]$. For simplicity, $\sigma^{a^*_0}$ is denoted by $\sigma^*$.

Here, we raise two examples as members of this class.
\begin{example}[Gaussian Distribution]
Gaussian bandits, which are bandit models whose potential outcomes follow Gaussian distributions, with fixed variances belong to the locally location-shift bandit class. 
\end{example}
\begin{example}[Bernoulli Distribution]
Beroulli bandits, which are bandit models whose potential outcomes follow Bernoulli distributions, belong to the locally location-shift bandit class. If $\mu^a(P)(x) \to \mu(x)$ for all $a\in[K]$, the conditional variance  also converges to $\sigma^1(P)(x) = \cdots = \sigma^K(P)(x) = \mu(x)(1-\mu(x))$.
\end{example}
In this section, we derive the lower bounds for this class.

To derive the lower bound, we first restrict our BAI strategy to a consistent strategy, which is also considered in \citet{Kaufman2016complexity}. 
\begin{definition}[Consistent strategy]\label{def:consistent}
For each $P\in\mathcal{P}$, if $a^*(P)$ is unique, then $\mathbb{P}_P(\widehat{a}_T = a^*(P)) \to 1$ as $T\to \infty$.
\end{definition}
In large deviation efficiency of hypothesis testing, a similar consistency is assumed \citep{Vaart1998}. 

Although we can derive the lower bound for any consistent strategy for two-armed bandits ($K=2$), we need further restrictions on a class of strategies to derive the lower bound for multi-armed bandits with $K\geq 3$. In this paper, we restrict a class of strategies to an asymptotically invariant strategy defined as follows:
\begin{definition}[Asymptotically invariant strategy]\label{def:asymp_inv}
A strategy is called asymptotically invariant if for any pair $P, Q\in\mathcal{P}\times \mathcal{P}$, any $x\in\mathcal{X}$, and all $a\in[K]$, as $T\to \infty$,
\begin{align}
    \left|\frac{1}{T}\sum^T_{t=1}\mathbb{E}_P\left[\mathbbm{1}[A_t = a]|X_t = x\right] - \frac{1}{T}\sum^T_{t=1}\mathbb{E}_Q\left[\mathbbm{1}[A_t = a]|X_t = x\right]\right| \to 0. 
\end{align}
\end{definition}

Then, we present lower bounds for bandit models belonging to the location-shift bandit class.
Let $\mathcal{W}$ be a set of all measurable functions such that $w: [K]\times\mathcal{X} \to (0, 1)$ and $\sum_{a\in[K]} w(a|x) = 1$ for all $x\in\mathcal{X}$. We refer to $w\in\mathcal{W}$ an allocation ratio, which can be used to obtain the following lower bound.
The proof is shown in Appendix~\ref{appdx:lower}. 
\begin{theorem}[Lower bound for the locally location-shift bandit class]
\label{thm:semipara_bandit_lower_bound}
Suppose that $\mathcal{P}$ is locally location-shift bandit class. 
Let $C > 0$ be some constants independent from $\{\mu^a_0\}_{a\in[K]}$.
Suppose that for any $P_0 \in \mathcal{P}$, there exists a constant $\Delta_0$ such that $\mu^*_0 - \mu^a_0 \leq \Delta_0$. Then, 
for any $P_0 \in \mathcal{P}$, 
under Assumption~\ref{asm:bounded_mean_variance}, any consistent (Definition~\ref{def:consistent}) and asymptotically invariant (Definition~\ref{def:asymp_inv}) strategy satisfies, for all $a\in [K]$,
\begin{align*}
   \lim_{\Delta_0 \to 0}\limsup_{T \to \infty} -\frac{1}{\Delta^2_0T}\log\mathbb{P}_{ P_0 }(\widehat{a}_T \neq a^*_0)\leq \sup_{w \in \mathcal{W}}\min_{a\neq a^*_0} \frac{1}{2\Omega^{a}_0(w)}  + o(1),
\end{align*}
where
\[\Omega^{a}_0(w) = \mathbb{E}_{P_0}\left[\frac{\left(\sigma^*(X_t)\right)^2}{w(a^*_0| X_t)} + \frac{\left(\sigma^a(X_t)\right)^2}{w(a| X_t)} \right].\]
\end{theorem}
We refer to a statistical model as a semiparametric model if the distribution is characterized by both of the finite dimensional parameters (gaps of the expected rewards) and infinite dimensional parameters (e.g., the distribution of contextual information).
The denominator of the first term of the RHS in the lower bound corresponds to a semiparametric efficiency bound, which corresponds to a Cram\'er-Rao lower bound for semiparametric models \citep{bickel98,Vaart1998}, of the gap (ATE) between two treatment arms $a,b\in[K]$ $a\neq b$ under an allocation ratio $w \in \mathcal{W}$ in the supremum \citep{hahn1998role}\footnote{More precisely, the semiparametric efficiency bound of the asymptotic variance of the ATE  between two treatment arms $a,b\in[K]$ $a\neq b$ is given as $\mathbb{E}_{P_0}\left[\frac{\left(\sigma^a_0(X_t)\right)^2}{w(a| X_t)} + \frac{\left(\sigma^b_0(X)\right)^2}{w(b| X_t)} + \big\{(\mu^*_0(X_t) - \mu^a_0(X_t)) - (\mu^*_0 - \mu^a_0)\big\}^2\right]$ \citep{hahn1998role}, where $\big\{(\mu^*_0(X_t) - \mu^a_0(X_t)) - (\mu^*_0 - \mu^a_0)\big\}^2$ appears unlike ours.}. This result implies that the optimal BAI strategy chooses treatment arms so as to reduce the asymptotic variance of estimators for the gaps (ATEs) between the best and suboptimal treatment arms. Here, When the asymptotic variance of the gap estimators is small, the gaps can be estimated more accurately. Theorem \ref{thm:semipara_bandit_lower_bound} will make this implication clearer. 

Here, the supremum of the RHS in the lower bound can be replaced with the maximum as
\begin{align*}
    \sup_{w \in \mathcal{W}}\min_{a\neq a^*_0} \frac{1}{2\Omega^{a}_0(w)} = \max_{w \in \mathcal{W}}\min_{a\neq a^*_0} \frac{1}{2\Omega^{a}_0(w)}.
\end{align*}
In Theorem~\ref{thm:semipara_bandit_lower_bound_lsmodel}, we show the analytical solution of $w^* = \argmax_{w \in \mathcal{W}}\min_{a\neq a^*_0} \frac{1}{2\Omega^{a}_0(w)}$. Then, the analytical solution of this maximization problem and refined lower bound are shown in the following theorem. 

\begin{theorem}[Lower bounds for the location-shift bandit class]
\label{thm:semipara_bandit_lower_bound_lsmodel}
Suppose that $\mathcal{P}$ is a locally location-shift bandit class, and Assumption~\ref{asm:bounded_mean_variance} holds. For any $P_0 \in \mathcal{P}$, suppose that there exists a constant $\Delta_0$ such that $\mu^*_0 - \mu^a_0 \leq \Delta_0$. Let $C_1, C_2 > 0$ be some constants independent from $T$ and $\Delta_0$. 
Then, for any $P_0 \in \mathcal{P}$, any consistent (Definition~\ref{def:consistent}) and asymptotically invariant (Definition~\ref{def:asymp_inv}) strategy satisfies the following lower bounds hold for each case with $K=2$ and $K\geq 3$:
\begin{itemize}
\item when $K = 2$, $w^*(1|x) = \frac{\sigma^1(x)}{\sigma^1(x) + \sigma^2(x)}$ and $w^*(2|x) = \frac{\sigma^2(x)}{\sigma^1(x) + \sigma^2(x)}$ for any $x\in\mathcal{X}$, and 
\begin{align*}
   \lim_{\Delta_0 \to 0}\limsup_{T \to \infty} -\frac{1}{\Delta^2_0T}\log\mathbb{P}_{ P_0 }(\widehat{a}_T \neq a^*_0)\leq \frac{1}{2\mathbb{E}_{P_0}\left[\left(\sigma^1(X_t) + \sigma^2(X_t)\right)^2\right]} + o(1). 
\end{align*}
\item when $K\geq 3$, $w^*(a|x) = \frac{\left(\sigma^a(x)\right)^2}{\sum_{b\in[K]}\left(\sigma^b(x)\right)^2}$ for all $a\in[K]$ and any $x\in\mathcal{X}$, and 
\begin{align*}
   \lim_{\Delta_0 \to 0}\limsup_{T \to \infty} -\frac{1}{\Delta^2_0T}\log\mathbb{P}_{ P_0 }(\widehat{a}_T \neq a^*_0)\leq \frac{1}{2\sum_{b\in[K]}\mathbb{E}_{P_0}\left[\left(\sigma^b(X_t)\right)^2\right]} + o(1). 
\end{align*}
\end{itemize}
\end{theorem}

Because all gaps $\mu^*_0 - \mu^a_0$ are assumed to be upper bounded by $\Delta_a$, we consider a situation where the expected rewards of all suboptimal treatment arms are in $[\mu^*_0 - \Delta_0, \mu^*_0)$. To obtain lower bounds, it is sufficient to consider a case where $\mu^b = \mu^*_0 = \Delta_0$, under which the largest lower bounds are given (Figure~\ref{fig:concept:smallgap}). Based on implications obtained from Theorem~\ref{thm:semipara_bandit_lower_bound_lsmodel}, we construct our strategy in Section~\ref{sec:track_aipw}.

\begin{figure}[t]\centering
\includegraphics[scale=0.3]{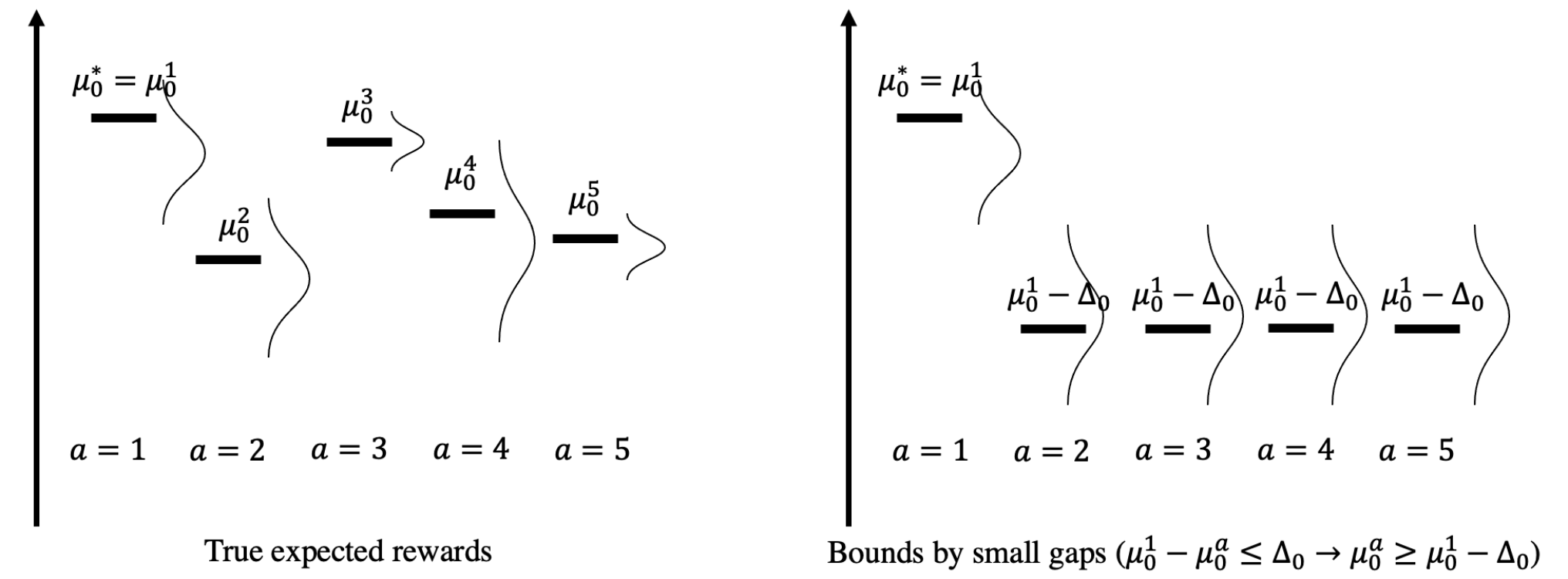}
\caption{An idea in the derivation of the lower bounds. To lower bound the probability of misidentification, or equivalently upper bound $ - \frac{1}{T}\log\mathbb{P}_{ P_0 }(\widehat{a}_T \neq a^*_0)$, it is sufficient to consider a case in the fight figure.}
\label{fig:concept:smallgap}
\end{figure}

Here, note that an allocation ratio $w$ in the supremum corresponds to an expectation of sampling rule $\frac{1}{T}\sum^T_{t=1}\mathbbm{1}[A_t = a]$ conditioned on $x$ under an alternative hypothesis $Q\in\mathcal{P}$ such that $Q\neq P_0$, which is used to derive the lower bound. From the definition of asymptotically invariant strategies, the maximizer $\widetilde{w}$ also work as a target sample allocation ratio
used in our proposed strategy; that is, n expectation of sampling rule $\frac{1}{T}\sum^T_{t=1}\mathbbm{1}[A_t = a]$ conditioned on $x$ under an alternative hypothesis $P_0\in\mathcal{P}$. 
In Sections~\ref{sec:track_aipw} and \ref{sec:asymp_opt}, we show that by allocation samples following the target sample allocation ratio, the upper bound for the probability of misidentification in our proposed strategy matches the lower bound. Thus, we can confirm that these maximizers correspond to the optimal target sample allocation ratio.

\subsection{Interpretations of Lower Bounds under the Small-gap Regime and Restrictions on Strategies}
In this section, we discuss the interpretations of lower bounds under the small-gap regime and restrictions on strategies. Although \citet{Kaufman2016complexity} and \citet{Carpentier2016} derive the lower bound for BAI with a fixed budget, their lower bounds do not employ a small-gap regime. Besides, \citet{Kaufman2016complexity} considers only a consistent strategy while we further restrict strategies to asymptotically invariant. We discuss the meanings of these elements in lower bounds. First, we review the lower bound for two-armed Gaussian bandits of \citet{Kaufman2016complexity}, which is derived for consistent strategies. However, this lower bound cannot be extended to multi-armed bandits. Therefore, we next consider restricting strategies to asymptotically invariant ones and the small-gap regime. We explain the meanings and benefits of these restrictions. Finally, we raise specific cases of our locally location-shift bandit class. 

\paragraph{Review of distribution-dependent lower bounds without contextual information.} 
First, we review the distribution-dependent lower bound for the probability of misidentification in two-armed bandits following Gaussian distributions with a fixed budget when there is no contextual information \citep{Kaufman2016complexity}. Let $\left(\sigma^a(P_0)\right)^2 = \left(\sigma^a_0\right)^2$. When the potential outcome of each treatment arm $a\in[K]$ follows the Gaussian distributions, the distribution-dependent lower bound is given as the following proposition. 
\begin{proposition}[Theorem~12 in \citet{Kaufman2016complexity}]
\label{prp:kauf_gaussian}
Suppose that $K = 2$, and consider a bandit class $\mathcal{P}$ such that for any $P\in\mathcal{P}$, $Y^a$ is generated from $\mathcal{N}\left(\mu^a(P), \left(\sigma^a\right)^2\right)$ for all $a\in[K]$, where $\mu^a(P) \in \mathbb{R}$ and $\left(\sigma^a_0\right)^2 > 0$ are constants, and $\mathcal{N}(\mu^a(P), \left(\sigma^a_0\right)^2)$ is a Gaussian distribution with a mean $\mu^a(P)$ and a variance $\left(\sigma^a_0\right)^2$ (variance is fixed for any $P\in\mathcal{P}$). Then, for any $P_0 \in\mathcal{P}$, any consistent strategy satisfies
\begin{align*}
    -\frac{1}{T}\log\mathbb{P}_{ P_0 }(\widehat{a}_T \neq a^*_0)\leq  \frac{\left(\mu^1_0 - \mu^2_0\right)^2}{2 \big(\sigma^1_0 + \sigma^2_0\big)^2}.
\end{align*} 
\end{proposition}
The proof is shown in \citet{Kato2022small}. 
This lower bound can be derived without restrictions of the asymptotically invariant strategy (Definition~\ref{def:asymp_inv}) and small gap ($|\mu^1_0 - \mu^2_0| \to 0$). Although the optimal target sample allocation ratio from this lower bound cannot be derived, by using $w^{*}(1) = \frac{\sigma^1_0}{\sigma^1_0 + \sigma^2_0}$ and $w^{*}(0) = 1- w^{*}(1)$ as a target sample allocation ratio, we can find an asymptotically optimal algorithm whose upper bound matches the lower bound when $|\mu^1_0 - \mu^2_0| \to 0$ when $\sigma^1_0$ and $\sigma^2_0$ are known \citep{glynn2004large,Kaufman2016complexity}.

\paragraph{Restrictions on a strategy class and the small-gap regime.} We restrict our strategies to asymptotic invariant strategies in addition to consistent strategies. As shown in Proposition~\ref{prp:kauf_gaussian}, this restriction is unnecessary to derive a lower bound for two-armed bandits with a Gaussian distribution without contextual information. However, without this restriction, we conjecture that we cannot derive lower bounds due to the reverse KL problem \citep{kaufmann2020hdr} when the number of treatment arms is larger than three or contextual information is available. We conjecture the following reason for why this restriction is a necessity. In BAI with a fixed budget, an estimation error of a target sample allocation ratio affects the evaluation of the probability of misidentification. In particular, if the target sample allocation ratio depends on $a^*_0$, we need to estimate $a^*_0$ to obtain the target sample allocation ratio. By restricting strategies to asymptotically invariant ones, we can avoid this estimation. On the other hand, when $K=2$, the target sample allocation ratio depends only on the standard deviation and not on the $a^*_0$ even if we use the restriction. Therefore, even without restriction, we can derive the lower bound when there are only two treatment arms.\footnote{In addition, we also do not have to use this restriction when considering the Equal-variance class (Definition~\ref{def:equal}), such as Bernoulli distributions, defined below. In those cases, the uniform sampling is target sample allocation ratio.} In fact, the lower bound for two-armed Gaussian bandits without contextual information with the restriction is the same as that without the restriction. 

In addition to the asymptotically invariant strategies, we consider the small-gap regime. We consider this regime mainly for the technical purpose of deriving the lower bounds, as explained in the next paragraph, but this regime has its own meaning. As we mentioned, this regime corresponds to one of the hardest (worst) cases to identify the best treatment arm. Therefore, we can interpret that asymptotically optimal strategies are 
a kind of asymptotically minimax optimal in the sense that the worst-case probability of misidentification matches the worst-case lower bound. 

From a technical perspective, asymptotically invariant strategies and the small-gap regime are important because they allow us to ignore the estimation error of a target sample allocation ratio. As shown by \citet{Carpentier2016}, we cannot develop strategies under which the probability of misidentification matches the lower bound of \citet{Kaufman2016complexity} when the gaps are large and strategies are restricted only to consistent ones. We consider that when gaps are large, an estimation error of the optimal target sample allocation ratio is a contributing factor of the probability of misidentification. Under the small-gap regime, we can ignore the estimation error relative to the probability of misidentification because identification of the best treatment arm becomes difficult when the gaps are sufficiently small. Thus, the small-gap regime makes the asymptotic optimality argument in fixed-budget BAI tractable by allowing the evaluation of optimal allocation probabilities to be ignored. 

As another advantage of the restrictions, we can obtain an analytical solution for the target sample allocation ratio. Under the small gap regime, if we constrain a strategy to be asymptotically invariant, we can express the target sample allocation ratio in terms of variance. This property is a great advantage in the computation of the target sample allocation ratio and in the interpretation of the algorithm.

\paragraph{Specific classes of the locally location-shift bandit class.}
By generalizing two cases, where all bandit models follow Gaussian distributions and those following Bernoulli bandit models, we define location-shift and equal-variance bandit classes as specific classes of the locally location-shift bandit class.

First, we consider the following location-shift bandit class, where the conditional variances are fixed, and only mean parameters vary across bandit models. This class is a generalization of a class of Gaussian distributions. 
\begin{definition}[Location-shift bandit class]
\label{def:ls_bc_g}
A class of bandit models $\mathcal{P}^{\mathrm{L}}$ is a location-shift bandit class if $\mathcal{P}^{\mathrm{L}} = \left\{P\in\mathcal{P} : \left(\sigma^a(P)(x)\right)^2 = \left(\sigma^a(x)\right)^2\right\}$, where $\sigma^a(x) > 0$ is a constant. 
\end{definition}
For this class, from Theorem~\ref{thm:semipara_bandit_lower_bound_lsmodel}, the lower bound given as $\frac{1}{2\sum_{b\in[K]}\mathbb{E}_{P_0}\left[\left(\sigma^b(X_t)\right)^2\right]} + o\left(1\right)$. 

As a generalization of bandit models whose potential outcomes follow 
one-parameter distributions such as Bernoulli, Binomial, and Gamma distributions, we define the following equal-variance bandit class. 
\begin{definition}[Equal-variance bandit class]
\label{def:equal}
A class of bandit models $\mathcal{P}^{\mathrm{E}}$ is an equal-variance bandit class if for a location-shift bandit model $\mathcal{P}$, $\sigma^{1}(x) = \sigma^{2}(x) = \cdots = \sigma^{K}(x) = \sigma(x)$ for any $x\in\mathcal{X}$, where $\sigma(x)$ is a constant.
\end{definition}
When outcomes follow Bernoulli distributions, the bandit model belongs to the equal-variance bandit class because the variances are the same when the expected rewards are the same. For this class, from Theorem~\ref{thm:semipara_bandit_lower_bound_lsmodel}, the lower bound given as $\frac{1}{K\mathbb{E}_{P_0}\left[\left(\sigma(X_t)\right)^2\right]} + o(1)$. Because the variances are equal across treatment arms, the target sample allocation ratio is also equal across treatment arms.
This lower bound and the target sample allocation ratio implies that the uniform-EBA strategy is optimal, where we choose each treatment arm with the same probability (the uniform sampling rule) and recommend a treatment arm with the highest sample average of observed rewards (the empirical best arm (EBA) recommendation rule). The fact that the uniform-EBA strategy is approximately optimal for two-armed Bernoulli bandits is also reported by \citet{Kaufman2016complexity}. 

\subsection{Efficiency Gain by using Contextual Information}
\label{sec:efficiency_gain}
We investigate when a strategy can gain efficiency by using contextual information; that is, how lower bounds are tightened by using contextual information. We first show lower bounds to investigate the efficiency gain when contextual information is unavailable. Let $w^*:[K]\to (0,1)$ such that $\sum_{a\in[K]}w^*(a) = 1$ be a target sample allocation ratio when contextual information is not available.
\begin{corollary}[Lower bounds for the location-shift bandit class]
\label{cor:lower_bound_lsmodel}
Suppose that $\mathcal{P}$ is a locally location-shift bandit class, and Assumption~\ref{asm:bounded_mean_variance} holds. For any $P_0 \in \mathcal{P}$, suppose that there exists a constant $\Delta_0$ such that $\mu^*_0 - \mu^a_0 \leq \Delta_0$. Let $C_1, C_2 > 0$ be some constants independent from $T$ and $\Delta_0$. 
Then, for any $P_0 \in \mathcal{P}$, any consistent (Definition~\ref{def:consistent}) and asymptotically invariant (Definition~\ref{def:asymp_inv}) strategy satisfies the following lower bounds hold for each case with $K=2$ and $K\geq 3$:
\begin{itemize}
\item when $K = 2$, $w^*(1) = \frac{\sigma^1_0}{\sigma^1_0 + \sigma^2_0}$ and $w^*(2) = \frac{\sigma^2_0}{\sigma^1_0 + \sigma^2_0}$, and 
\begin{align*}
   \limsup_{T \to \infty} -\frac{1}{\Delta^2_0T}\log\mathbb{P}_{ P_0 }(\widehat{a}_T \neq a^*_0)\leq \frac{1}{2\left(\sigma^1_0 + \sigma^2_0\right)^2} + o(1). 
\end{align*}
\item when $K\geq 3$, $w^*(a) = \frac{\left(\sigma^a\right)^2}{\sum_{b\in[K]}\left(\sigma^b\right)^2}$ for all $a\in[K]$, and 
\begin{align*}
   \limsup_{T \to \infty} -\frac{1}{\Delta^2_0T}\log\mathbb{P}_{ P_0 }(\widehat{a}_T \neq a^*_0)\leq \frac{1}{2\sum_{a\in[K]}\left(\sigma^a_0\right)^2} + o(1). 
\end{align*}
\end{itemize}
\end{corollary}
As we show the following theorem, lower bounds can be improved by using contextual information. The equality holds, for example, if $\left(\sigma^a_0\right)^2 = \left(\sigma^a(x)\right)^2$ for any $x\in\mathcal{X}$ and all $a\in[K]$ for $K=2$, and if  $\left(\sigma^a_0\right)^2 = \mathbb{E}_{P_0}\left[\left(\sigma^a(X_t)\right)^2\right]$ for all $a\in[K]$ for $K \geq 3$. 
\begin{theorem}
The lower bounds in Theorem~\ref{thm:semipara_bandit_lower_bound_lsmodel} is tighter than those in Corollary~\ref{cor:lower_bound_lsmodel}; 
that is, when $K=2$,
\begin{align*}
   &\frac{1}{2 \big(\sigma^1_0 + \sigma^2_0\big)^2} \leq \frac{1}{2\mathbb{E}_{P_0}\left[\left(\sigma^1(X_t) + \sigma^2(X_t)\right)^2\right]},
\end{align*}
where the equality holds if $\left(\sigma^a_0\right)^2 = \left(\sigma^a(x)\right)^2$ for any $x\in\mathcal{X}$ and all $a\in[K]$;
when $K \geq 3$,
\begin{align*}
   &\frac{1}{2\sum_{b\in[K]}\left(\sigma^b_0\right)^2} \leq \frac{1}{2\sum_{b\in[K]}\mathbb{E}_{P_0}\left[\left(\sigma^b(X_t)\right)^2\right]}, 
\end{align*}
where the equality holds if $\left(\sigma^a_0\right)^2 = \mathbb{E}_{P_0}\left[\left(\sigma^a(X_t)\right)^2\right]$ for all $a\in[K]$, which is a weaker condition than the condition for a case with $K=2$. 
\end{theorem}
We show the proof below. Here, we can find a case where $\frac{1}{2 \big(\sigma^1_0 + \sigma^2_0\big)^2} < \frac{1}{2\mathbb{E}_{P_0}\left[\left(\sigma^1(X_t) + \sigma^2(X_t)\right)^2\right]}$; that is,
$\left(\sigma^1_0 + \sigma^2_0\right)^2 > \mathbb{E}_{P_0}\left[\left(\sigma^1(X_t) + \sigma^2(X_t)\right)^2\right]$. Suppose that contextual information is discrete, and we observe context $X_{1}$ with probability $0.5$ and $X_{2}$ otherwise. Let $\left(\sigma^1(X_{1})\right)^2 = 7.5$, $\left(\sigma^1(X_{2})\right)^2 = 0.5$, $\left(\sigma^2(X_{1})\right)^2 = 0.5$, $\left(\sigma^2(X_{2})\right)^2 = 1.5$. We also suppose that $\mu^a(x) = \mu^a$ for $x\in\{X_{1}, X_{2}\}$. In this case, $\left(\sigma^a_0\right)^2 = \mathbb{E}_{P_0}\left[\left(\sigma^a(X_t)\right)^2\right] = 0.5\left(\sigma^a(X_1)\right)^2 + 0.5\left(\sigma^a(X_2)\right)^2$; therefore, $\left(\sigma^1_0\right)^2 = 4$ and $\left(\sigma^2_0\right)^2 = 1$. Then, we have $\mathbb{E}_{P_0}\left[\left(\sigma^1(X_t) + \sigma^2(X_t)\right)^2\right] \approx 7.8 < \left(\sigma^1_0 + \sigma^2_0\right)^2 = 9$. Therefore, we can gain efficiency by using contextual information. We discuss the efficiency gain in Section~\ref{sec:asymp_opt}. 
\begin{proof}Consider a case with $K=2$.
Recall that the lower bounds with contextual information are $\max_{w \in \mathcal{W}} \frac{1}{\mathbb{E}_{P_0}\left[\frac{\left(\sigma^1(X_t)\right)^2}{w(1| X_t)} + \frac{\left(\sigma^2(X_t)\right)^2}{w(2| X_t)} \right]}$ and those without contextual information are the $\max_{w \in (0, 1)} \frac{1}{\mathbb{E}_{P_0}\left[\frac{\left(\sigma^1_0\right)^2}{w} + \frac{\left(\sigma^2_0\right)^2}{1-w} \right]}$. Therefore, we compare $\mathbb{E}_{P_0}\left[\frac{\left(\sigma^1(X_t)\right)^2}{w(1| X_t)} + \frac{\left(\sigma^2(X_t)\right)^2}{w(2| X_t)} \right]$ and $\mathbb{E}_{P_0}\left[\frac{\left(\sigma^1_0\right)^2}{w} + \frac{\left(\sigma^2_0\right)^2}{1-w} \right]$. Here, from \[\left(\sigma^a_0\right)^2 = \mathbb{E}_{P_0}\left[\left(\sigma^a(X_t)\right)^2\right] + \mathbb{E}_{P_0}\left[ \left(\mu^a_0(X_t) - \mu^a_0\right)^2 \right] \geq \mathbb{E}_{P_0}\left[\left(\sigma^a(X_t)\right)^2\right],\] it holds that 
\begin{align*}
&\mathbb{E}_{P_0}\left[\left(\sigma^1(X_t) + \sigma^2(X_t)\right)^2\right] = \min_{w\in\mathcal{W}}\mathbb{E}_{P_0}\left[\frac{\left(\sigma^1(X_t)\right)^2}{w(1| X_t)} + \frac{\left(\sigma^2(X_t)\right)^2}{w(2| X_t)}\right]\leq \min_{w\in (0, 1)}\mathbb{E}_{P_0}\left[\frac{\left(\sigma^1(X_t)\right)^2}{w} + \frac{\left(\sigma^2(X_t)\right)^2}{1-w}\right]\\
&\ \ \ = \min_{w\in(0, 1)}\left\{\frac{\mathbb{E}_{P_0}\left[\left(\sigma^1(X_t)\right)^2\right]}{w} + \frac{\mathbb{E}_{P_0}\left[\left(\sigma^2(X_t)\right)^2\right]}{1- w}\right\} \leq \min_{w\in(0, 1)}\left\{\frac{\left(\sigma^1_0\right)^2}{w} + \frac{\left(\sigma^2_0\right)^2}{1-w}\right\} = \left(\sigma^1_0 + \sigma^2_0\right)^2.
\end{align*}

When $K \geq 3$, we can directly confirm that $\sum_{b\in[K]}\left(\sigma^b_0\right)^2 \geq \sum_{b\in[K]}\mathbb{E}_{P_0}\left[\left(\sigma^b(X_t)\right)^2\right]$. 

In both cases, the equality holds if $\left(\sigma^a_0\right)^2 = \left(\sigma^a(x)\right)^2$ for any $x\in\mathcal{X}$ and all $a\in[K]$. When $K\geq 3$, the equality also holds if $\left(\sigma^a_0\right)^2 = \mathbb{E}_{P_0}\left[\left(\sigma^a(X_t)\right)^2\right]$ for all $a\in[K]$, which is a weaker condition than the former condition. 
\end{proof}

\section{Proposed Strategy: the RS-AIPW Strategy}
\label{sec:track_aipw}
This section presents our strategy, which consists of sampling and recommendation rules. For each $t = 1,2,\dots, T$, our sampling rule randomly chooses a treatment arm with a probability identical to an estimated target sample allocation ratio. In final round $T$, our recommendation rule recommends a treatment arm with the highest-estimated expected reward. Based on these rules, we refer to this as the RS-AIPW strategy.\footnote{This strategy resembles ones in efficient ATE estimation via an adaptive experiment \citep{Laan2008TheCA,Hahn2011}. A sampling rule using the target sample allocation ratio \eqref{eq:target} for a case with $K=2$ is referred to as the Neyman allocation rule \citep{Armstrong2022,adusumilli2022minimax}. \citet{Kato2022small} discusses its asymptotic optimality in BAI with a fixed budget when the number of treatment arms is two, contextual information is not available, and the potential outcomes follow Gaussian distributions. Generalizing their result, we refine the Neyman allocation rule to cases where there are multiple treatment arms and contextual information.  Regarding the upper bound, we apply the results of  \citet{Fan2013,fan2014generalization} with the modifications by \citet{Kato2022small}.}

\subsection{Target Sample Allocation Ratio}
First, we define a target sample allocation, which is used to define a sampling rule. We estimate it during an adaptive experiment and employ the estimator as a probability of choosing a treatment arm. We call a target sample allocation worst-case optimal if the upper and lower bounds for the probability of misidentification match under a small-gap, one of the worst cases, when using our strategy using the allocation ratio. 
We conjecture the worst-case optimal target sample allocation ratio using the results of Section~\ref{sec:lower_bounds}. 
In particular, the results of Theorem~\ref{thm:semipara_bandit_lower_bound_lsmodel} yields the following conjectures for the worst-case optimal target sample allocation ratio $w^*\in\mathcal{W}$: when $K = 2$, for each $x\in\mathcal{X}$,
\begin{align}
\label{eq:target}
w^*(1|x) = \frac{\sigma^1(x)}{\sigma^1(x) + \sigma^2(x)}\ \mathrm{and}\ w^*(2|x) = \frac{\sigma^2(x)}{\sigma^1(x) + \sigma^2(x)}; 
\end{align}
when $K \geq 3$, for each $x\in\mathcal{X}$ and all $a\in[K]$,
\[w^*(a|x) = \frac{\left(\sigma^a(x)\right)^2}{\sum_{b\in[K]}\left(\sigma^b(x)\right)^2}.\]
Under this conjectured target sample allocation ratio, we can show that the upper and lower bounds for the probability of misidentification match under a small-gap regime in Section~\ref{sec:asymp_opt}; hence, we call this target sample allocation ratio worst-case optimal.  
This target sample allocation ratio is unknown when the variances are unknown; therefore, it must be estimated via observations during the bandit process.

\subsection{Sampling Rule with Random Sampling (RS) and Estimation}
We provide a sampling rule referred as to a \textit{random sampling} (RS) rule. For $a\in[K]$ and $t\in[T]$, let $\widehat{w}_{t}(a|x)$ be an estimated target sample allocation ratio at round $t$. In each round $t$, we obtain $\gamma_t$ from the uniform distribution on $[0,1]$ and choose a treatment arm $A_t = 1$ if $\gamma_t \leq \widehat{w}_{t}(1|X_t)$ and $A_t = a$ for $a \geq 2$ if $\gamma_t \in (\sum^{a-1}_{b=1}\widehat{w}_{t}(b|X_t), \sum^a_{b=1}\widehat{w}_{t}(b|X_t)]$.

As an initialization, we choose a treatment arm $A_t$ at round $t \leq K$ and set $\widehat{w}_t(a| x) = 1/K$ for $a\in[K]$ and $x\in\mathcal{X}$.
In a round $t > K$, for all $a\in[K]$, we estimate the target sample allocation ratio $w^{*}$ using past observations $\mathcal{F}_{t-1}$, such that for all $a\in[K]$ and $x\in\mathcal{X}$, $\widehat{w}_{t}(a|x) > 0$ and $\sum_{a\in[K]} \widehat{w}_{t}(a|x) = 1$. 
Then, in round $t$, we choose a treatment arm $a$ with a probability $\widehat{w}_{t}(a|X_t)$. 
To construct an estimator $\widehat{w}_{t}(a| x)$ for all $x\in\mathcal{X}$ in each round $t$, we denote a bounded estimator of the conditional expected reward $\mu^a_0(x)$ by $\widehat{\mu}^a_{t}(x)$, that of the conditional expected squared reward $\nu^a_0(x)$ by $\widehat{\nu}^a_{t}(x)$, and that of the  conditional variance $\left(\sigma^a_0(x)\right)^2$ by $(\widehat{\sigma}^a_t(x))^2$. All estimators are constructed only from samples up to round $t$. More formally, they are constructed as follows.
For $t=1,2,\dots, K$, we set $\widehat{\mu}^a_{t} = \widehat{\nu}^a_{t} = (\widehat{\sigma}^a_t(x))^2 = 0$. For $t > K$, we estimate $\mu^a(x)$ and $\nu^a(x)$ using only past samples $\mathcal{F}_{t-1}$ and converge to the true parameter almost surely (Assumption~\ref{asm:almost_sure_convergence}). For the estimators, we can use nonparametric estimators, such as the nearest neighbor regression estimator and kernel regression estimator, which are prove to converge to the true function almost surely under a bounded sampling probability $\widehat{w}_t$ by \citet{yang2002} and \citet{qian2016kernel}. 
As long as these conditions are satisfied, any estimators can be used. Note that we do not assume specific convergence rates for these estimators because we can show the asymptotic optimality without them owing to the unbiasedness of the AIPW estimator \citep{Kato2021adr}. Let $(\widehat{\sigma}^{\dagger a}_t(x))^2 = \widehat{\nu}^a_{t}(x) - \left(\widehat{\mu}^a_{t}(x)\right)^2$ for all $a\in[K]$ and $x\in\mathcal{X}$. 
Then, we estimate the variance $\left(\sigma^a_0(x)\right)^2$ for all  $a\in[K]$ and $x\in\mathcal{X}$ in a round $t$ as $\left(\widehat{\sigma}^a_{t}(x)\right)^2 =\max\{ \min\{((\widehat{\sigma}^{\dagger a}_t(x))^2,C_{\sigma^2}\},1/C_{\sigma^2}\}$ and define
$\widehat{w}_{t}$ by replacing the variances in $w^*$ with corresponding estimators; that is,
\begin{align*}
\widehat{w}_{t}(1|x) = \frac{\widehat{\sigma}^1_t(X_t)}{\widehat{\sigma}^1_t(X_t) + \widehat{\sigma}^2_t(X_t)}\ \mathrm{and}\ \widehat{w}_{t}(2|x) = \frac{\widehat{\sigma}^2_t(X_t)}{\widehat{\sigma}^1_t(X_t) + \widehat{\sigma}^2_t(X_t)}; 
\end{align*}
when $K \geq 3$, for each $x\in\mathcal{X}$,
\[\widehat{w}_{t}(a|x) = \frac{(\widehat{\sigma}^a_t(X_t))^2}{\sum_{b\in[K]}(\widehat{\sigma}^b_t(X_t))^2}\] 
If there are multiple elements in $\argmax_{a\in[K]} \widehat{\mu}^a_{t}(x)$, we choose one of them as $\widehat{a}_t$ in some way. 

We employ this strategy to apply the large deviation expansion for martingales to the estimator of the expected reward, which is the core of our theoretical analysis in Section~\ref{sec:asymp_opt}.

\subsection{Recommendation Rule with the AIPW Estimator}
The following section presents our recommendation rule.
In the recommendation phase of round $T$, for each $a\in[K]$, we estimate $\mu^a$ for each $a\in[K]$ and recommend the maximum. 
To estimate $\mu^a$, the AIPW estimator is defined as
\begin{align}
\label{eq:aipw}
\widehat{\mu}^{\mathrm{AIPW}, a}_{T} =\frac{1}{T} \sum^T_{t=1}\varphi^a\Big(Y_t, A_t, X_t; \widehat{\mu}^a_t, \widehat{w}_t\Big),\qquad \varphi^a(Y_t, A_t, X_t; \widehat{\mu}^a_t, \widehat{w}_t) = \frac{\mathbbm{1}[A_t = a]\big(Y^a_{t}- \widehat{\mu}^a_{t}(X_t)\big)}{\widehat{w}_t(a| X_t)} + \widehat{\mu}^a_{t}(X_t).
\end{align}

In the final round $t=T$, we recommend $\widehat{a}_T \in [K]$ as
\begin{align}
\label{eq:recommend}
\widehat{a}_T = \argmax_{a\in[K]} \widehat{\mu}^{\mathrm{AIPW}, a}_{T}.
\end{align}
The AIPW estimator has the following properties: (i) its components $\{\varphi^a(Y_t, A_t, X_t; \widehat{\mu}^a_t, \widehat{w}_t)\}^T_{t=1}$ are a martingale difference sequence, thereby allowing us to use the large deviation bounds for martingales; (ii) it has the minimal asymptotic variance among the possible estimators.
For instance, we can use other estimators with a martingale property, such as the inverse probability weighting (IPW) estimator \citep{Horvitz1952}, but their asymptotic variance will be larger than that of the AIPW estimator. For the $t$-th element of the sum in the AIPW estimator, we use the nuisance parameters estimated from past observations up to the round $t-1$ to make the sequence in the sum a martingale difference sequence. This technique is often used in adaptive experiments \citep{Laan2008TheCA,hadad2019,Kato2020adaptive,Kato2021adr} and also has a similar motivation to double machine learning \citep{ChernozhukovVictor2018Dmlf}. Note that in double machine learning for a doubly robust (DR) estimator, we usually impose specific convergence rates for the estimators of the nuisance parameter, which are not required in our case owing to the unbiasedness of the AIPW estimator (Assumption~\ref{asm:almost_sure_convergence}). Also see \citet{Kato2021adr}. 

We present the pseudo-code in Algorithm~\ref{alg}. Note that $C_{\mu}$ and $C_{\sigma^2}$ are introduced for technical purposes to bound the estimators. Therefore, any large positive value can be used. 

\begin{remark}[Remark on the sampling rule]
Unlike the sampling rule of \citet{Garivier2016}, our proposed sampling rule does not choose the next treatment arm so that the empirical allocation ratio tracks the optimal target sample allocation ratio. This is due to the use of martingale properties under the AIPW estimator in the theoretical analysis of the upper bound. 
\end{remark}
\begin{remark}[Sampling for stabilization] In the pseudo-code, only the first $K$ rounds are used for initialization. To stabilize the performance, we can increase the number of samplings in initialization, similarly to the forced-sampling  approach employed by \citet{Garivier2016}. In Section~\ref{sec:asymp_opt}, to show the asymptotic optimality, we use almost sure convergence of $\widehat{w}_{t}$ to $w^{*}$. As long as $\widehat{w}_{t} \xrightarrow{\mathrm{a.s}} w^{*}$, we can adjust $\widehat{w}_{t}$ appropriately. For instance, we can use $\widetilde{w}_{t} = (1-r_t)\widehat{w}_t(a|X_t) + r_t 1/2$ as the sampling probability instead of $\widehat{w}_t$, where $r_t \to 0$ as $t\to \infty$. 
\end{remark}
\begin{remark}[The role of $C_{\sigma^2}$] 
Assumption~\ref{asm:bounded_mean_variance} implies that the sampling probability is bounded by a small constant, $1/(2C_{\sigma^2}) \leq w^{*}(a|x) \leq C_{\sigma^2}/2$. Thus, it ensures that the variance of the AIPW estimator is finite. Although the role of this constant appears to be similar to the forced sampling \citep{Garivier2016}, it is substantially different. We can set $C_{\sigma^2}$ sufficiently large so that it is almost negligible in implementation. 
\end{remark}

\begin{algorithm}[tb]
   \caption{RS-AIPW strategy}
   \label{alg}
\begin{algorithmic}
   \STATE {\bfseries Parameter:} Positive constants $C_{\mu}$ and $C_{\sigma^2}$.
   \STATE {\bfseries Initialization:} 
   \FOR{$t=1$ to $K$}
   \STATE Draw $A_t=t$. For each $a\in[K]$, set $\widehat{w}_{t}(a|x) = 1/K$.
   \ENDFOR
   \FOR{$t=K+1$ to $T$}
   \STATE Observe $X_t$. 
   \STATE Construct $\widehat{w}_{t}(1|X_t)$ by using the estimators of the variances.
   \STATE Draw $\gamma_t$ from the uniform distribution on $[0,1]$. 
   \STATE $A_t = 1$ if $\gamma_t \leq \widehat{w}_{t}(1|X_t)$ and $A_t = a$ for $a \geq 2$ if $\gamma_t \in \left(\sum^{a-1}_{b=1}\widehat{w}_{t}(b|X_t), \sum^a_{b=1}\widehat{w}_{t}(b|X_t)\right]$. 
   \ENDFOR
   \STATE Construct $\widehat{\mu}^{\mathrm{AIPW}, a}_{T}$ for each $a\in[K]$ following \eqref{eq:aipw}.
   \STATE Recommend $\widehat{a}_T$ following \eqref{eq:recommend}.
\end{algorithmic}
\end{algorithm} 

\section{Asymptotic Optimality of the RS-AIPW Strategy}
\label{sec:asymp_opt}
In this section, we derive the following upper bound of the misspecification probability of the RS-AIPW strategy, which implies that the strategy is asymptotically optimal.

\subsection{Asymptotic Optimality}
We derive the upper bounds for bandit models, where the rewards are sub-exponential random variables. 
\begin{assumption}
\label{asm:sub_exp}
For all $P\in\mathcal{P}$ and $a\in[K]$, $X_{t}$ is sub-exponential random variable and $Y^a_{t}$ is conditionally sub-exponential random variable given $X_t = x$; that is, there are constants $U, U' > 0$ such that for all $P\in\mathcal{P}$,  $a\in[K]$, $t, u, u' > 0$, and $x\in\mathcal{X}$, $\mathbb{P}_P(|X_t| > u) \leq 2\exp( - u/U)$ and $\mathbb{P}_P(|Y_t| > u| X_t = x) \leq 2\exp( - u'/U')$
\end{assumption}
We also suppose that estimators of $\mu^a_0(x)$ and $\nu^a_0(x)$ converges to the true values almost surely.
\begin{assumption}
\label{asm:almost_sure_convergence}
For all $a\in[K]$ and $x\in\mathcal{X}$, $\widehat{\mu}^a_{t}(x)$ and $\widehat{w}^a_{t}(a| x)$ are $\mathcal{F}_{t-1}$-measurable, $|\widehat{\mu}^a_{t}(x)| \leq C_\mu$ and $|\widehat{w}^a_{t}(a| x)| \leq C_w$, and 
\begin{align*}
    t^{\alpha}\left|\widehat{\mu}^a_{t}(x) - \mu^a_0(x) \right|  \xrightarrow{\mathrm{a.s.}} 0\qquad \mathrm{and}\qquad t^{\alpha}\left|\widehat{w}^a_{t}(a| x) - w^*(a| x)\right|  \xrightarrow{\mathrm{a.s.}} 0\qquad \mathrm{as}\ t\to \infty,
\end{align*}
where $\alpha > 0$ is some constant, and $C_w > 0$ is a constant that depends on $C_{\sigma^2}$.
\end{assumption}
This assumption is satisfied when we sample each treatment arm with a probability larger than a positive constant and use appropriate estimation methods \citep{yang2002}. Let us define $\widetilde{V}^a = \mathbb{E}_{P_0}\left[\frac{\left(\sigma^*(X_t)\right)^2}{w^*(a^*_0| X_t)} + \frac{\left(\sigma^a(X_t)\right)^2}{w^*(a| X_t)}  + \left(\mu^*_0(X_t) - \mu^a_0(X_t) - (\mu^*_0 - \mu^a_0)\right)^2\right]$.
\begin{theorem}[Upper bound of the RS-AIPW strategy]\label{thm:optimal}
Suppose that $\mathcal{P}$ is a locally location-shift bandit class. If Assumptions~\ref{asm:bounded_mean_variance}, \ref{asm:sub_exp}, and \ref{asm:almost_sure_convergence} hold for any $P_{0}\in\mathcal{P}$, there exist constant $C_0, C_1 > 0$ such that $\sup_{1\leq t \leq T}\mathbb{E}_{P_0}[\exp(C_0 \sqrt{T}|\xi^a_t|) \;|\mathcal{F}_{t-1}]\leq C_1$ for any $P_{0}\in\mathcal{P}$.
Then for any $P_{0}\in\mathcal{P}$ such that $0 < (\mu^*_0- \mu^a_0) / {\sqrt{\widetilde{V}^a}} \le \min\{ C_0 / 4, \sqrt{{3 C_0^2} / ({8 C_1})}\}$ for all $a\in[K]$,
\begin{align*}
    &\liminf_{T \to \infty} - \frac{1}{T}\log \mathbb{P}_{P_0}\left(\widehat{a}_T \neq a^*_0\right)\geq \min_{a\neq a^*_0}\frac{\left(\mu^*_0 - \mu^a_0\right)^2}{2 \widetilde{V}^a} - c
    \left(\left(\frac{\mu^*_0- \mu^a_0}{\sqrt{\widetilde{V}^a}}\right)^3 + \left(\frac{\mu^*_0- \mu^a_0}{\sqrt{\widetilde{V}^a}}\right)^4 \right),
\end{align*}
where $c$ is a constant, independent from $T$ and $\mu^a_0$.
\end{theorem}
This theorem allows us to evaluate the exponentially small probability of misidentification up to the constant term when $\Delta_0\to 0$. Moreover, this result also implies that the estimation error of the target sample allocation ratio $w^{*}$ is negligible when $\Delta_0\to 0$. This is because  the upper bound matches the performance of strategies for Gaussian bandit models developed by \citet{glynn2004large} given the optimal target sample allocation ratio. This also means that the estimation error of the target sample allocation ratio is insensitive to the probability of misidentification in situations where identifying the best treatment arm is difficult due to the small gap.

\paragraph{Asymptotic optimality.}
When $\mathcal{P}$ is the locally location-shift bandit class, this upper bound matches the lower bounds in Theorems~\ref{thm:semipara_bandit_lower_bound} under a small-gap regime. 
\begin{corollary}
Suppose that there exists a constant $C >0$ such that $\left|\mu^*_0(x) - \mu^a_0(x)\right| \leq C\left(\mu^*_0 - \mu^a_0\right)$ for all $a\in[K]$ and $x\in\mathcal{X}$, then under the same conditions as those in Theorem~\ref{thm:optimal}, 
\begin{align*}
    &\liminf_{\widetilde{\Delta}_0 \to 0}\liminf_{T \to \infty} - \frac{1}{\widetilde{\Delta}^2_0T}\log \mathbb{P}_{P_0}\left(\widehat{a}_T \neq a^*_0\right)\geq \min_{a\neq a^*_0}\frac{1}{2 \mathbb{E}_{P_0}\left[\frac{\left(\sigma^*(X_t)\right)^2}{w^*(a^*_0| X_t)} + \frac{\left(\sigma^a(X_t)\right)^2}{w^*(a| X_t)}\right]} - o\left(1\right),
\end{align*}
where $\widetilde{\Delta}_0 = \min_{a\neq a^*_0}\left(\mu^*_0 - \mu^a_0\right)> 0$. 
\end{corollary}
Although the upper bound also matches the lower bound for the equal-variance bandit class, the uniform-EBA strategy is also obviously optimal.

\paragraph{Efficiency gain.} As well as Section~\ref{sec:efficiency_gain}, we investigate the efficiency gain by using contextual information from the viewpoint of upper bound. First, we show the upper bound when we cannot employ contextual information. Recall that we defined the target sample allocation ratios without contextual information as $w^*(a)$ for each $a\in[K]$ in Corollary~\ref{cor:lower_bound_lsmodel}. 
\begin{corollary}
Suppose that $\mathcal{P}$ is a locally location-shift bandit class. 
Suppose that Assumptions~\ref{asm:bounded_mean_variance}, \ref{asm:sub_exp}, and \ref{asm:almost_sure_convergence} hold. Then, for any $P_{0}\in\mathcal{P}$,
\begin{align*}
    &\liminf_{T \to \infty} - \frac{1}{\widetilde{\Delta}^2_0T}\log \mathbb{P}_{P_0}\left(\widehat{a}_T \neq a^*_0\right)\geq \min_{a\neq a^*_0}\frac{1}{2 \mathbb{E}_{P_0}\left[\frac{\left(\sigma^*\right)^2}{w^*(a^*_0)} + \frac{\left(\sigma^a\right)^2}{w^*(a)}\right]} - o\left(1\right),
\end{align*}
where $\widetilde{\Delta}_0 = \min_{a\neq a^*_0}\left(\mu^*_0 - \mu^a_0\right)> 0$. 
\end{corollary}
By comparing upper bounds for cases where we can use contextual information and we cannot use contextual information, we obtain the following relationship on the efficiency gain. 
\begin{align*}
 \min_{a\neq a^*_0}\frac{1}{2 \mathbb{E}_{P_0}\left[\frac{\left(\sigma^*(X_t)\right)^2}{w^*(a^*_0| X_t)} + \frac{\left(\sigma^a(X_t)\right)^2}{w^*(a| X_t)}\right]} &\geq \min_{a\neq a^*_0}\frac{1}{2 \mathbb{E}_{P_0}\left[\frac{\left(\sigma^*(X_t)\right)^2}{w^*(a^*_0| X_t)} + \frac{\left(\sigma^a(X_t)\right)^2}{w^*(a| X_t)}  + \left(\mu^*_0(X_t) - \mu^a_0(X_t) - (\mu^*_0 - \mu^a_0)\right)^2\right]}\\
 &\geq \min_{a\neq a^*_0}\frac{1}{2 \mathbb{E}_{P_0}\left[\frac{\left(\sigma^*\right)^2}{w^*(a^*_0)} + \frac{\left(\sigma^a\right)^2}{w^*(a)}\right]}.
\end{align*}

\subsection{Proof of the Upper Bound}
Owing to the dependency among samples in BAI, it is also difficult to apply the standard large deviation bound \citep{Dembo2009large} to a sample average of some random variable. 
For example, G\"{a}rtner-Ellis theorem \citep{Gartner1977,Ellis1984} provides a large deviation bound for dependent samples, but it requires the existence of the cumulant, a logarithmic moment generating function, which is not easily guaranteed for the samples in BAI. 

For these problems, we derive a novel Cram\'er-type large deviation bounds for martingales by extending the results of \citet{Grama2000} and \citet{Fan2013,fan2014generalization}. Note that their original large deviation bound is only applicable to martingales whose conditional second moment is bounded deterministically; that is, for some martingale difference sequence $\{W_s\}^n_{s=1}$ of some random variable $W_s$, for any $n > 0$, there exists a real number $0< \epsilon < 1/2$ such that $\mathbb{E}\left[\sum^n_{s=1}\mathbb{E}[W^2_s| \mathcal{F}_{s-1}] - 1\right] \leq \epsilon^2$; then, \citet{Fan2013,fan2014generalization} derive the upper bound for $\mathbb{P}\left(\sum^n_{s=1}W_s > z\right)$, where $\epsilon$ belongs to a range upper bounded by $\epsilon^{-1}$. Thus, their large deviation bound holds when $\mathbb{E}\left[\sum^n_{s=1}\mathbb{E}[W^2_s| \mathcal{F}_{s-1}] - 1\right]$ can be bounded by any $\epsilon$ for any $n > 0$. 
\citet{Kato2022small} modifies the results of \citet{Fan2013,fan2014generalization} by bounding the conditional second moment for large $T$ to apply the problem of BAI with a fixed budget. We basically follow \citet{Kato2022small} and generalize the result to the multi-armed bandit case. For the sake of completeness, we show a part of the results of \citet{Kato2022small}. Using the large deviation bound and AIPW estimator, under our proposed strategy, the upper and lower bounds for probability of misidentification match as the gaps converge to zero.

\subsubsection*{Step~1: Cram\'er's large deviation expansions for the AIPW estimator}
Here, we introduce key elements of our analysis. For each $t\in[T]$, we define the difference variable
\begin{align*}
\xi^a_t &= \frac{\varphi^{a^*_0}\Big(Y_t, A_t, X_t; \widehat{\mu}^{a^*_0}_t, \widehat{w}_t\Big) -\varphi^a\Big(Y_t, A_t, X_t; \widehat{\mu}^a_t, \widehat{w}_t\Big) -  (\mu^*_0 - \mu^a_0)}{\sqrt{T \widetilde{V}^a}}.
\end{align*}
We also define its sum $Z^a_t = \sum^t_{s=1}\xi^a_{s}$, and a sum of conditional moments $W_t = \sum^t_{s=1}\mathbb{E}_{P_0}[(\xi^a)^2_{s}| \mathcal{F}_{s-1}]$ with initialization $W_0 = 0$. Using the difference variable $\xi^a_t$, we can express the gap estimator as $\sqrt{T}(\widehat{\mu}^{\mathrm{AIPW}, a^*_0}_{T} - \widehat{\mu}^{ \mathrm{AIPW}, a}_T - (\mu^*_0 - \mu^a_0)) / \sqrt{\widetilde{V}^a} = \sum^T_{t=1}\xi^a_t = Z^a_T$. Here, $\left\{\left(\xi^a_t, \mathcal{F}_t\right)\right\}^T_{t=1}$ is a martingale difference sequence (Appendix~\ref{appdx:martingale}), 
using the fact that $\widehat{\mu}^a_{t}$ and $\widehat{w}_t(a|X_t)$ are $\mathcal{F}_{t-1}$-measurable random variables.
Let us also define $V_T = \mathbb{E}_{P_0} [ | \sum_{t=1}^T \mathbb{E}_{P_0}[(\xi^a_t)^2 | \mathcal{F}_{t-1}] -1 |]$
and denote the cumulative distribution function of the standard normal distribution by $\Phi(x) = ({\sqrt{2\pi}})^{-1} \int_{-\infty}^x \exp(- {t^2} / {2})\mathrm{d}t$. 
We obtain the following theorem on the tail probability of $Z^a_T$: 
\begin{theorem}
\label{thm:fan_refine}
Suppose that Assumptions~\ref{asm:bounded_mean_variance} and \ref{asm:sub_exp}, and  the following condition hold:\\
Condition~A: $\sup_{1\leq t \leq T}\mathbb{E}_{P_0}[\exp(C_0 \sqrt{T}|\xi^a_t|) \;|\mathcal{F}_{t-1}]\leq C_1$ for some positive constants $C_0,C_1$.\\
Then, for any $\varepsilon > 0$, there exist $T_0, c_1, c_2>0$ such that, for all $T\geq T_0$ and $1\leq u \leq \sqrt{T}\min\{ C_0/4, \sqrt{{3 C_0^2} / ({8 C_1})}\}$, 
\begin{align*}
\frac{\mathbb{P}_{P_0}\left(Z^a_T \leq - u\right)}{\Phi(-u)} &  \le c_1 u \exp\left(c_2\left( \frac{u^3}{\sqrt{T}} + \frac{u^4}{T} + u^2 (V_T+ \varepsilon / \{T^{\alpha}(1 - \alpha)\} ) + T_0\right)  \right),
\end{align*}
where the constants $c_1,c_2$ depend on $C_0$ and $C_1$ but do not depend on $\{(\xi^a_t, \mathcal{F}_t)\}^T_{t=1}$, $u$, and the bandit model $P$. 
\end{theorem}
As described by \citet{fan2014generalization}, if $T\mathbb{E}[(\xi^a_t)^2|\mathcal{F}_{t-1}]$ are all bounded from below by a positive constant, Condition~A implies the conditional Bernstein condition: for a positive constant $C$, $|\mathbb{E}[(\xi^a_t)^k|\mathcal{F}_{t-1}]| \leq \frac{1}{2} k!(C/\sqrt{T})^{k-2}\mathbb{E}[(\xi^a_t)^2|\mathcal{F}_{t-1}]$ for all $k\geq 2$ and all $t\in[T]$. 

For $u = \sqrt{T} (\mu^*_0- \mu^a_0) / {\sqrt{\widetilde{V}^a}}$ and 
\begin{align*}
    &\mathbb{P}_{P_0}\left(Z^a_T \leq - \sqrt{T} (\mu^*_0- \mu^a_0) / {\sqrt{\widetilde{V}^a}}\right) \\
    &= \mathbb{P}_{P_0}\left( \frac{\sum_{t =1}^T ({\varphi^{a^*_0}\Big(Y_t, A_t, X_t; \widehat{\mu}^{a^*_0}_t, \widehat{w}_t\Big) - \varphi^a\Big(Y_t, A_t, X_t; \widehat{\mu}^a_t, \widehat{w}_t\Big) - \mu^*_0- \mu^a_0})}{{\sqrt{T\widetilde{V}}}}\leq - \frac{\sqrt{T} (\mu^*_0- \mu^a_0)}{ {\sqrt{\widetilde{V}^a}} }\right)\\
    &= \mathbb{P}_{P_0}\left(\widehat{\mu}^{\mathrm{AIPW}, a^*_0}_{T} \le  \widehat{\mu}^{\mathrm{AIPW}, a}_{T}\right).
\end{align*}
Then, the probability that we fail to make the correct treatment arm comparison is bounded as 
\begin{align*}
    &\frac{\mathbb{P}_{P_0}\left(\widehat{\mu}^{\mathrm{AIPW}, a^*_0}_{T} \leq \widehat{\mu}^{\mathrm{AIPW}, a}_{T} \right)}{\Phi\left(-\sqrt{T} (\mu^*_0- \mu^a_0)/\sqrt{\widetilde{V}^a}\right)}\\
    &\leq c_1 \sqrt{T} \frac{\mu^*_0- \mu^a_0}{\sqrt{\widetilde{V}^a}} \exp\left(c_2 \left(T\left\{ \left(\frac{\mu^*_0- \mu^a_0}{\sqrt{\widetilde{V}^a}}\right)^3 +\left(\frac{\mu^*_0- \mu^a_0}{\sqrt{\widetilde{V}^a}}\right)^4 +\left(\frac{\mu^*_0- \mu^a_0}{\sqrt{\widetilde{V}^a}}\right)^2(V_T + \varepsilon  / \{T^{\alpha}(1 - \alpha)\})\right) + c_2T_0 \right)  \right).
\end{align*}

Here, we provide the proof sketch of Theorem~\ref{thm:fan_refine}. The formal proof is shown in Appendix~\ref{appdx:proof_large_deviation}. 
\begin{proof}[Proof sketch of Theorem~\ref{thm:fan_refine}.] Let us define $r_t(\lambda ) = \exp(\lambda \xi^a_t)/\mathbb{E}[\exp(\lambda \xi^a_t)]$. Then, we  apply the change-of-measure in \cite{Fan2013, fan2014generalization} to transform the bound. In \citet{Fan2013,fan2014generalization}, the proof is complete up to this procedure. However, in our case, the second moment is also a random variable. Because of the randomness, there remains a term $\mathbb{E}[\exp(\overline{\lambda}(u)\sum^T_{t=1} \xi^a_t)]/(\prod^T_{t=1}\mathbb{E}[\exp(\overline{\lambda}(u) \xi^a_t)])$, where $\overline{\lambda}(u)$ is some positive function of $u$. Therefore, we next consider the bound of the conditional second moment of $\xi^a_t$ to apply $L^r$-convergence theorem (Proposition~\ref{prp:lr_conv_theorem}). With some computation, the proof is complete.
\end{proof}

\subsubsection*{Step~2: Gaussian approximation under a small gap}
Finally, we consider an approximation of the large deviation bound. Here, $\Phi(-u)$ is bounded as $\frac{1}{\sqrt{2 \pi} (1 + u)} \exp(- \frac{u^2}{2})\le \Phi(-u) \le \frac{1}{\sqrt{\pi} (1 + u)} \exp( - \frac{u^2}{2}),\; u\ge 0$ (see  \citet[Section 2.2.,][]{Fan2013}).
By combining this bound with Theorem~\ref{thm:fan_refine} and Proposition~\ref{prp:rate_clt} in Appendix~\ref{appdx:prelim}, which shows the rate of convergence in the Central limit theorem (CLT) for $0 \leq u \leq 1$, we have the following corollary.
\begin{corollary}
\label{thm:fan_refine2}
Suppose that  Assumptions~\ref{asm:bounded_mean_variance} and \ref{asm:sub_exp}, Condition~A in Theorem~\ref{thm:fan_refine}, and the following conditions hold:\\
Condition~B: $\left(\mu^*_0 - \mu^a_0 \right) / {\sqrt{\widetilde{V}^a}} \le \min\{ C_0/4, \sqrt{3 C_0^2/(8 C_1})\}$;\\
Condition~C: $\lim_{T \to \infty}V_T = 0$.\\
Then, there exist a constant $c>0$ such that
\begin{align*}
    \liminf_{T \to \infty} - \frac{1}{T}\log\mathbb{P}_{P_0}\left(\widehat{\mu}^{\mathrm{AIPW}, a^*_0}_{T} \leq \widehat{\mu}^{\mathrm{AIPW}, a}_{T} \right)  \geq  \frac{(\mu^*_0- \mu^a_0)^2}{2\widetilde{V}^a} - c 
\left(\left(\frac{\mu^*_0- \mu^a_0}{\sqrt{\widetilde{V}^a}}\right)^3 + \left(\frac{\mu^*_0- \mu^a_0}{\sqrt{\widetilde{V}^a}}\right)^4 \right).
\end{align*}
\end{corollary}
This approximation can be considered a Gaussian approximation because the probability
is represented by $\exp(-{\left(\mu^*_0- \mu^a_0\right)^2} T / ({2 \widetilde{V}^a}))$. 
Condition~B is satisfied as $\mu^*_0- \mu^a_0 \to 0$. 
To use Corollary~\ref{thm:fan_refine2}, we need to show that Conditions~A and C hold.  
First, the following lemma states that Condition~A holds with the constants $C_0$ and $C_1$, which are universal to the problems in $\mathcal{P}$.
\begin{lemma}
\label{lem:condition1}
Suppose that Assumptions~\ref{asm:bounded_mean_variance} and ~\ref{asm:sub_exp} and hold. For each $C_0 \ge 0$, there exists a positive constant $C_1$ that depends on $ C_0, C_\mu, C_{\sigma^2}$, such that 
$\sup_{t \in [T]} \mathbb{E}_{P_0}[\exp(C_0 \sqrt{T} |\xi^a_t|) \;| \mathcal{F}_{t-1}] \le C_1$.
\end{lemma}
With regards to Condition~C, we introduce the following lemma for the convergence of $V_T$, which corresponds to the mean convergence of the variance of the AIPW estimator scaled with $\sqrt{T}$. 
\begin{lemma}
\label{lem:condition2}
 Suppose that Assumptions~\ref{asm:bounded_mean_variance} and ~\ref{asm:sub_exp} hold. For any $P \in \mathcal{P}$, $\lim_{T \to \infty}V_T = 0$; that is, for any $\delta > 0$, there exists $T_0$ such that for all $T>T_0$, $\mathbb{E}_{P_0} [| \sum_{t=1}^T \mathbb{E}_{P_0}[(\xi^a_t)^2 | \mathcal{F}_{t-1}] - 1 |] \le \delta.$
\end{lemma}

The proofs of Lemma~\ref{lem:condition1} and Lemma~\ref{lem:condition2} are shown in Appendix~\ref{appdx:lem:condition1} and \ref{appdx:lem:condition2}, respectively. 

Finally, the proof of Theorem~\ref{thm:optimal} is completed as follows: 
\begin{align*}
     &\liminf_{T \to \infty} - \frac{1}{T}\log \mathbb{P}_{P_0}(\widehat{a}_T \neq a^*_0)   \geq \liminf_{T \to \infty} - \frac{1}{T}\log \sum_{a\neq a^*_0}\mathbb{P}_{P_0}(\widehat{\mu}^{\mathrm{AIPW}, a}_{T}  \geq  \widehat{\mu}^{\mathrm{AIPW}, a^*_0}_{T}) 
     \\
     &\ge  \liminf_{T \to \infty} - \frac{1}{T}\log (K-1) \max_{a\neq a^*_0} \mathbb{P}_{P_0}(\widehat{\mu}^{\mathrm{AIPW}, a}_{T}  \geq  \widehat{\mu}^{\mathrm{AIPW}, a^*_0}_{T})\\
     &\geq  \min_{a\neq a^*_0} \frac{(\mu^*_0- \mu^a_0)^2}{2 \widetilde{V}^a} - c
    \left(\left(\frac{\mu^*_0- \mu^a_0}{\sqrt{\widetilde{V}^a}}\right)^3 + \left(\frac{\mu^*_0- \mu^a_0}{\sqrt{\widetilde{V}^a}}\right)^4 \right).
\end{align*}

\begin{remark}[CLT]
Note that the CLT cannot provide an exponentially small evaluation of the probability of misidentification. It gives an approximation around $1/\sqrt{T}$ of the expected reward, but we are interested in an evaluation with constant deviation from the expected reward. However, when the gap converges to zero with $1/\sqrt{T}$, our large deviation bound gives the CLT for martingale. In this sense, our result is a generalization of the martingale CLT.
\end{remark}

\section{Related work} 
\label{sec:related}

\subsection{Additional Literature on BAI}
The stochastic MAB problem is a classical abstraction of the sequential decision-making problem \citep{Thompson1933,Robbins1952,Lai1985}, and BAI is a paradigm of the MAB problem \citep{EvanDar2006,Audibert2010,Bubeck2011}. Though the problem of BAI itself goes back decades, its variants go as far back as the 1950s \citet{bechhofer1968sequential}. 

\citet{kaufmann14,Kaufman2016complexity} conjectures distribution-dependent lower bounds for BAI. In the BAI literature, there is another setting, known as BAI with fixed confidence \citep{Jenninson1982,Mannor2004,Kalyanakrishnan2012,wang2021fast}. For the fixed confidence setting, \citet{Garivier2016} solves the problem in the sense that they develop a strategy whose upper bound of the sample complexity, an expected stopping time, matches the distribution-dependent lower bound. The result is further developed by \citet{Degenne2019b} to solve the two-player game by the no-regret saddle point algorithm. Furthermore, \citet{Qin2017}, \citet{Shang2020}, and \citet{Jourdan2022} extend the Top Two Thompson Sampling (TTTS), proposed by \citet{Russo2016} and shows the asymptotic optimality of their strategies in the fixed confidence setting. \citet{wang2021fast} develops Frank-Wolfe-based Sampling (FWS) to characterize the complexity of fixed-confidence BAI with various types of structures among the arms. See \cite{wang2021fast} for techniques in the fixed-confidence setting and a further comprehensive survey.

\citet{Russo2016}, \citet{Qin2017}, and \citet{Shang2020} propose the Bayesian BAI strategies, which are optimal in the sense of the posterior convergence rate. Although the upper bounds of the sample complexity are shown to match the lower bounds of \citet{Kaufman2016complexity} in fixed-confidence BAI for some of the methods, the upper bounds for the probability of misidentification do not match that for fixed-budget BAI. Although the rate of the posterior convergence is also optimal in the fixed-budget setting, it does not imply the asymptotic optimality for the probability of misidentification \citep{Kasy2021,Ariu2021}. For example, the KL divergence in the lower and upper bounds is flipped between the evaluations of posterior convergence and probability of misidentification. In addition, for the posterior convergence, we consider a convergence of a random variable, while for the probability of misidentification, we consider a convergence of a non-random variable. 

In evaluation, we can use the simple regret. \citet{Bubeck2009} provides a non-asymptotic minimax lower and upper bound of simple regret for bandit models with a bounded support. Following their results, the uniform-EBA strategy is optimal for bandit models with a bounded support. This result is compatible with lower bounds under the equal-variance bandit class, which implies that the uniform sampling is asymptotically optimal for the equal-variance bandit class. Because \citet{Bubeck2009} does not use other parameters, such as variances, 
their result does not contradict with Theorem~\ref{thm:semipara_bandit_lower_bound}, which implies that the target sample allocation ratio using the variances is optimal. Recently,  \citet{adusumilli2022minimax,Adusumilli2022diffusion} consider another minimax and Bayes evaluations of BAI, by employing diffusion process approaches \citep{Fan2021,Wager2021}. 
\citet{Komiyama2021} discusses the optimality of Bayesian simple regret minimization, which is closely related to BAI in a Bayesian setting. They showed that parameters with a small gap make a significant contribution to Bayesian simple regret.

\subsection{Literature on Causal Inference}
The framework of bandit problems is closely related to the potential outcome framework of \citep{Neyman1923,Rubin1974}. In causal inference, the gap is often referred to as the average treatment effect, and the estimation is studied in this framework. To estimate the average treatment effect efficiently, \citet{Laan2008TheCA}, \citet{Hahn2011}, \citet{Meehan2020}, \citet{Kato2020adaptive}, and \citet{gupta2021efficient} propose adaptive strategies. The AIPW estimator, which is also referred to as a DR estimator, plays an important role in treatment effect estimation \citep{Robins1994,hahn1998role,bang2005drestimation,dudik2011doubly,Laan2016onlinetml,Luedtke2016}. The AIPW estimator also plays an important role in double/debiased machine learning literature because it mitigates the convergence rate conditions of the nuisance parameters \citep{ChernozhukovVictor2018Dmlf,Ichimura2022}.

In adaptive experiments for efficient ATE estimation, the AIPW estimator has also been used by \citet{Laan2008TheCA} and \citet{Hahn2011}. \citet{Karlan2014} applied the method of \citet{Hahn2011} to test how donors respond to new information regarding the effectiveness of a charity. These studies have been extended by \citet{Meehan2020} and \citet{Kato2020adaptive}.
However, the notion of optimality is based on the analogue of the efficient estimation of the ATE under i.i.d. observations and not complete in adaptive experiments. 

When constructing AIPW estimator with samples obtained from adaptive experiments, including BAI strategies, a typical construction is to use sample splitting and martingales \citep{Laan2008TheCA,hadad2019,Kato2020adaptive,Kato2021adr}.  \citet{Howard2020TimeuniformNN}, \citet{Kato2020adaptive}, and provide non-asymptotic confidence intervals of the AIPW or DR estimator, which do not bound a tail probability in large deviation as ours. The AIPW estimator is also used in the recent bandit literature, mainly in regret minimization \citep{dimakopoulou2021online,Kim2021}. \citet{hadad2019}, \citet{Bibaut2021}, and \citet{Zhan2021} consider the off-policy evaluation using observations obtained from regret minimization algorithms.

\subsection{Difference from Limit Experiments Frameworks}
\label{app_subsec:diff_limit_dec}
The small-gap regime is inspired by limit experiments framework \citep{LeCam1986,Vaart1998,Hirano2009}. 
For a parameter $\theta_0\in\mathbb{R}$ and $n$ i.i.d. observations for a sample size $n$, the limit experiments framework considers local alternatives $\theta = \theta_0 + h/\sqrt{n}$ for a constant $h\in\mathbb{R}$ \citep{Vaart1991,Vaart1998}. Then, we can approximate the statistical experiment by a Gaussian distribution and discuss the asymptotic optimality of statistical procedures under the approximation. \citet{Hirano2009} relates the asymptotic optimality of statistical decision rules \citep{Manski2000,Manski2002,Manski2004,DEHEJIA2005} to the limit experiment framework. This framework is further applied to policy learning, such as \citet{AtheySusan2017EPL}.

Independently, \citet{Armstrong2022} proposes an application of the local asymptotic framework to a setting similar to BAI by replacing the CLT used in the original framework, such as \citet{Vaart1998}, with that for martingales. In their analysis, the gaps converge to zero with $1/\sqrt{T}$, and a class of BAI strategies is restricted for the second moment of the score to converges to a constant, whereas our gaps converge to zero independently of $T$, and a class of BAI strategies is restricted to be consistent. 

Here, note that
taking the parameter $\theta = \theta_0 + h/\sqrt{T}$ does not produce the distribution-dependent analysis; that is, the instance is not fixed as $T$ increases. Therefore, a naive application of the distribution-dependent analysis like Proposition~\ref{lem:data_proc_inequality} does not provide a lower bounds for BAI in this setting.
To match the lower bound of  \citet{Kaufman2016complexity}, we need to consider the large deviation bound, rather than CLT. In other words, the limit experiment framework first applies a Gaussian approximation and then evaluates the efficiency under that approximation, where efficiency arguments are complete within the Gaussian distribution. In contrast, we derive the lower bounds of an event under the true distribution in our limit decision-making and approximate it by considering the limit of the gap. Therefore, in limit decision-making, we first consider the optimality for the true distribution and find the optimal strategy in the sense that the upper bound matches the lower bound when the gaps converge to zero.

\subsection{Other Related Work}
Our small-gap regime is also inspired by lil'UCB \citep{Jamieson2014}. 
\citet{Balsubramani2016} and \citet{Howard2020TimeuniformNN} propose sequential testing using the law of iterated logarithms and discuss the optimality of sequential testing based on the arguments of \citet{Jamieson2014}.

Ordinal optimization has been studied in the operation research community \citep{peng2016myopic, Dohyun2021}, and a modern formulation was established in the 2000s \citep{chen2000,glynn2004large}. Most of these studies consider the estimation of the optimal sampling rule separately from the probability of misidentification.

In addition to \citet{Fan2013,fan2014generalization}, several studies have employed martingales to obtain tight large deviation bounds \citep{Cappe2013,Juneja2019,Howard2020TimeuniformNN,Kaufmann2021}. Some of these studies have applied change-of-measure techniques. 

\citet{Tekin2015}, \citet{GuanJiang2018}, and \citet{Deshmukh2018} also consider BAI with contextual information, but their analysis and setting are different from those employed in this study.

\section{Discussion}
\label{sec:discuss}

\subsection{Asymptotic Optimally in BAI with a Fixed Budget}
\citet{Kaufman2016complexity} derives distribution-dependent lower bounds for BAI with a fixed confidence and budget, based on similar change-of-measure arguments to those found in \citet{Lai1985}. In BAI with fixed confidence, \citet{Garivier2016} develops a strategy whose upper bound and lower bounds for the probability of misidentification match. In contrast, in the fixed-budget setting, the existence of a strategy whose upper bound matches the lower bound of \citet{Kaufman2016complexity} was unclear. We consider that this is because the estimation error of an optimal target sample allocation ratio is negligible in BAI with a fixed budget, unlike BAI with fixed confidence, where we can draw each treatment arm until the strategy satisfies a condition. Furthermore, there are lower bounds different from \citet{Kaufman2016complexity}, such as \citet{Audibert2010}, \citet{Bubeck2011}, and \citet{Carpentier2016}.

\citet{Audibert2010} proposes the UCB-E and Successive Rejects (SR) strategies. Using the complexity terms
$H_1(P) = \sum_{a \in [K] \backslash\{a^*(P)\}} 1/(\Delta^a(P))^2$ and $H_2(P) = \max_{a \in [K] \backslash\{a^*(P)\}} a/(\Delta^a(P))^2$, where $\Delta^a(P) = \mu^{a^*(P)} - \mu^a$, they prove an upper bound for the probabilities of misidentification of the forms $\exp\left( - T/(18 H_1(P_0))\right)$ and $\exp \big( - T/( \log (K) H_2(P_0))\big)$, for UCB-E with the upper bound on $H_1(P_0)$ and SR, respectively. 

\citet{Kato2022small} shows that the upper bound for the probability of misidentification of the RS-AIPW strategy matches the lower bound derived by \citet{Kaufman2016complexity} (Proposition~\ref{prp:kauf_gaussian}) under the small-gap regime when the number of treatment arms is two, contextual information is not available, and the potential outcomes follow Gaussian distribution.  They approximate only the upper bound by the small gap but do not consider the approximation of the lower bound. 

\citet{Carpentier2016} discusses the optimality of the method proposed by \citet{Audibert2010} by an effect of constant factors in the exponents of certain bandit models. They proved the lower bound on the probability of misidentification of the form: $ \sup_{P\in\mathcal{P}^B}\Big\{ \mathbb{P}_{P_0}\big(\widehat{a}_T \neq a^*_0\big) \exp \big(400T/( \log(K) H_1(P))\big)\Big\}$, where for all $P\in\mathcal{P}^B$, there exists a constant $B>0$ such that $H_1(P) < B$. Our result does not contradict with the result that found by \citet{Carpentier2016}, as we consider a small-gap regime, rather than the large-gap regime employed by \citet{Carpentier2016}. In the other words, their results are complementary to ours because we consider situations with a small gap.

\subsection{Two-stage Sampling Rule}
Our RS-AIPW strategy is also applicable to a setting where we can update the sampling rule in batch, ratner than a sequential manner, as well as other BAI strategies in different settings. For example, even in a two-stage setting, where we are allowed to update the sampling rule only once, we can show the asymptotic optimality if the budgets separated into two-stages go to infinity simultaneously. Such a setting has frequently been adopted in the field of economics, such as \citet{Hahn2011} and \citet{Kasy2021}. 

\section{Conclusion}
In this study, we considered BAI with a fixed budget and contextual information under a small-gap regime. Subsequently, we derived lower bounds for the probability of misidentification by applying semiparametric analysis under the small-gap regime. Then, we proposed the RS-AIPW strategy. With the help of a new large deviation expansion we developed, we showed that the performance of our proposed RS-AIPW strategy matches the lower bound under a small gap. We also addressed a long-standing open issue in BAI with a fixed budget; even without contextual information, the existence of an asymptotically optimal BAI strategy was unclear. 
Because BAI with a fixed budget and without contextual information is a special case in our setting, we addressed this question. Furthermore, we demonstrated an analytical solution for the target sample allocation ratio, which has also been unknown for a long time. Thus, our study serves as a breakthrough in the field of BAI with a fixed budget. Our future direction is to develop BAI strategies for various settings, such as linear \citep{Hoffman2014,Liang2019,KatzSamuels2020}, combinatorial \citep{Chen2014}, and policy learning \citep{Kitagawa2018,AtheySusan2017EPL,Zhou2020}.   

\bibliographystyle{asa}
\bibliography{BAI.bbl}

\clearpage 

\tableofcontents

\appendix

\section{Preliminaries for the Proof}
\label{appdx:prelim}
\begin{definition}\label{dfn:uniint}[Uniform integrability, \citet{Hamilton1994}, p.~191] Let $W_t \in \mathbb{R}$ be a random variable with a probability measure $P$.  A sequence $\{W_t\}$  is said to be uniformly integrable if for every $\epsilon > 0$ there exists a number $c>0$ such that 
\begin{align*}
\mathbb{E}_{P}[|A_t|\cdot I[|A_t| \geq c]] < \epsilon
\end{align*}
for all $t$.
\end{definition}
The following proposition is from \citet{Hamilton1994}, Proposition~7.7, p.~191.
\begin{proposition}[Sufficient conditions for uniform integrability]\label{prp:suff_uniint} Let $W_t, Z_t \in\mathbb{R}$ are random variables. Let $P$ be a probability measure of $Z_t$. (a) Suppose there exist $r>1$ and $M<\infty$ such that $\mathbb{E}_{P}[|W_t|^r]<M$ for all $t$. Then $\{A_t\}$ is uniformly integrable. (b) Suppose there exist $r>1$ and $M < \infty$ such that $\mathbb{E}_{P}[|Z_t|^r]<M$ for all $t$. If $W_t = \sum^\infty_{j=-\infty}h_jZ_{t-j}$ with $\sum^\infty_{j=-\infty}|h_j|<\infty$, then $\{W_t\}$ is uniformly integrable.
\end{proposition}

\begin{proposition}[$L^r$ convergence theorem, p~165, \citet{loeve1977probability}]
\label{prp:lr_conv_theorem}
Let $Z_n$ be a random variable with probability measure $P$ and $z$ be a constant. 
Let $0<r<\infty$, suppose that $\mathbb{E}_{P}\big[|Z_n|^r\big] < \infty$ for all $n$ and that $Z_n \xrightarrow{\mathrm{p}}z$ as $n\to \infty$. The following are equivalent: 
\begin{description}
\item{(i)} $Z_n\to z$ in $L^r$ as $n\to\infty$;
\item{(ii)} $\mathbb{E}_{P}\big[|Z_n|^r\big]\to \mathbb{E}_{P}\big[|z|^r\big] < \infty$ as $n\to\infty$; 
\item{(iii)} $\big\{|Z_n|^r, n\geq 1\big\}$ is uniformly integrable.
\end{description}
\end{proposition}

Let $W_i$ be a random variable with probability measure $P$. Let $\mathcal{F}_n = \{W_1, W_2,\dots, W_n\}$.

\begin{proposition}[Strong law of large numbers for martingales, p 35,  \cite{hall1980martingale}]\label{prp:law_large_num_Hall}
Let $\{S_n = \sum^{n}_{i=1} W_i, \mathcal{F}_{n}, n\geq 1\}$ be a martingale and $\{U_n, n \geq 1\}$ a nondecreasing sequence of positive r.v. such that $U_n$ is $\mathcal{F}_{n-1}$-measurable. Then, 
\begin{align*}
    \lim_{n \to \infty}U^{-1}_n S_n= 0
\end{align*}
almost surely on the set $\{\lim_{n\to \infty} U_n = \infty, \; \sum_{i=1}^\infty U_i^{-1} \mathbb{E}[|W_i| | \mathcal{F}_{i-1}] < \infty\}$.
\end{proposition}

\begin{proposition}[Rate of convergence in the CLT, From Theorem~3.8, p 88,  \cite{hall1980martingale}]\label{prp:rate_clt}
Let $\{S_t = \sum^{t}_{s=1} X_s, \mathcal{F}_{t}, t\geq 1\}$ be a martingale with $\mathcal{F}_{t}$ equal to the $\sigma$-field generated by $X_1,\dots, X_t$. Let 
\begin{align*}
    V^2_t = \mathbb{E}\left[\left|\sum^t_{s=1}\mathbb{E}[Y^2_s| \mathcal{F}_{s-1}] - 1\right|\right]\qquad 1 \leq t \leq T.
\end{align*}
Suppose that for some $\alpha > 0$ and constants $M$, $C$ and $D$, 
\begin{align*}
    \max_{s\leq t}\mathbb{E}[\exp(|\sqrt{T} Y_t|^{\alpha})] < M,
\end{align*}
and 
\begin{align*}
    \mathbb{P}\left(|V^2_t - 1 | > D/\sqrt{t}(\log t)^{2+2/\alpha}\right) \leq C t^{-1/4}(\log t)^{1+1/\alpha}.
\end{align*}
Then, for $T\geq 2$,
\begin{align}
    \sup_{-\infty < x < \infty}\big|\mathbb{P}(S_T \leq x) - \Phi(x)\big|\leq A T^{-1/4} (\log T)^{1+1/\alpha},
\end{align}
where the constant $A$ depends only on $\alpha$, $M$, $C$, and $D$.
\end{proposition}

\section{Proof of Lower Bound (Theorem~\ref{thm:semipara_bandit_lower_bound})}
\label{sec:proof}
In this section, we provide proof of Theorem~\ref{thm:semipara_bandit_lower_bound}. Our argument is based on a change-of-measure argument, which has been applied to BAI without contextual information \citep{Kaufman2016complexity}. In this derivation, we relate the likelihood ratio to the lower bound. Inspired by \citet{Murphy1997}, we expand the semiparametric likelihood ratio, where the gap parameter $\mu^*_0 - \mu^a_0$ is regarded as a parameter of interest and the other parameters as nuisance parameters. By using a semiparametric efficient score function, we apply a series expansion to the likelihood ratio of the distribution-dependent lower bound around the gap parameter $\mu^*_0 - \mu^a_0$ under a bandit model of an alternative hypothesis. Then, when the gap parameter goes to $0$, the lower bound is characterized by the variance of the semiparametric influence function. Our proof is also inspired by \citet{Vaart1998} and \citet{hahn1998role}. Throughout the proof, for simplicity, $\mathcal{P}^{\mathrm{L}}$ is denoted by $\mathcal{P}$.

Precisely, our proof follows these steps. First, the goal is to express the lower bound of the probability of misidentification by using the gap parameter. In Proposition~\ref{lem:data_proc_inequality} of Appendix~\ref{sec:transport}, we introduce a bound for some event based on a change-of-measure argument \citep{Kaufman2016complexity}. We apply this bound to derive lower bounds for the probability of misidentification in the final step of the proof. Next, we consider distributions of observations. Although we defined distributions of the potential random variables $(Y^1_t, Y^2_t, \dots, Y^K_t, X_t)$ (full-data bandit models), we can only observe a reward of a chosen treatment arm, $Y^{A_t}_t$, and context, $X_t$, and cannot observe other rewards $(Y^a_t)_{a\in[K]\backslash\{A_t\}}$. Therefore, distributions of observations are different from the full-data bandit models. We induce the former from the latter in Appendix~\ref{sec:obs_data_bandit} to discuss optimality. With these preparations, in Appendix~\ref{sec:para_sub_full}, we introduce a parameter into the true nonparametric full-data bandit models to differentiate the log-likelihood around the gap parameter; that is, the gap parameter is introduced so that it corresponds to $\mu^*_0 - \mu^a_0$. This parameter is a technical device for the proof, and the parametrized models are called parametric submodels, which are subsets of $\mathcal{P}$. The derivative is then defined with respect to this parameter, and we consider applying the series expansion to the log likelihood. However, the derivative (score function) is not uniquely defined because it includes nuisance parameters other than the parameter of interest. Therefore, to specify a score function with the tightest lower bound, it is necessary to consider information on the distribution of the observations. To perform these operations, we associate the full-data bandit models with the distribution of the observed data in Appendix \ref{sec:mapping_obs}. Then, in Appendix \ref{sec:para_sub_obs}, we derive the parametric submodel of the distribution of observations from the parametric submodels of the full-data bandit models and define a score function for that the parametric submodel of the distribution of observations. For deriving lower bounds, an alternative hypothesis plays an important role, and we define a class of alternative hypotheses (alternative bandit models) in Appendix~\ref{sec:alter}. By using the alternative bandit models, we derive a lower bound of the probability of misidentification in Appendix~\ref{sec:deriv_lower}, which depends on the log-likelihood and is related to the gap parameter in the following arguments. For the lower bound, using the score function and alternative bandit models in Appendix~\ref{sec:alter}, we apply the series expansion to the log-likelihood in Appendix~\ref{sec:semiparametric_lratio} and characterize the bound in Proposition~\ref{lem:data_proc_inequality} of Appendix~\ref{sec:transport} with the gap parameter. Then, in Appendix~\ref{sec:oberved-data}, we derive the information bound of the second moment of the score function; then, in Appendix~\ref{sec:specification-score}, we specify a score function whose second moment is equal to the information bound in Appendix~\ref{sec:oberved-data}. Finally, combining them, we derive the lower bound for the probability of misidentification in Appendix~\ref{sec:final_step}. 

\subsection{Transportation Lemma}
\label{sec:transport}
Our lower bound derivation is based on change-of-measure arguments, which have been extensively used in the bandit literature \citep{Lai1985}.  
\cite{Kaufman2016complexity} derives the following result based on change-of-measure argument, which is the principal tool in our lower bound.
Let us define a density of $(Y^1, Y^2, \dots, Y^K, X)$ under  a bandit model $P\in\mathcal{P}$ as
\begin{align*}
    p_P(y^1, y^2, \dots, y^K, x) = \prod_{a\in[K]} f^a_{P}(y^a|x)\zeta_{P}(x)
\end{align*}
Let $f^{a^*_0}_P$ be denoted by $f^*_P$. 

\begin{proposition}[Lemma~1 in \cite{Kaufman2016complexity}]\label{lem:data_proc_inequality} Suppose that Assumption~\ref{asm:bounded_mean_variance} holds. 
Then, for any two bandit model $P,Q\in\mathcal{P}$ with $K$ treatment arms such that for all $a \in [K]$, $f^a_{P}(y^a|x)\zeta_{P}(x)$ and $f^{a}_{Q}(y^a|x)\zeta_{Q}(x)$ are mutually absolutely continuous, 
\begin{align*}
    &\mathbb{E}_{Q}\left[\sum^T_{t=1}  \mathbbm{1}[ A_t = a] \log \left(\frac{f^a_{Q}(Y^a_{t}| X_t)\zeta_{Q}(X_t)}{f^a_{P}(Y^a_{t}| X_t)\zeta_{P}(X_t)}\right) \right]\ge \sup_{\mathcal{E} \in \mathcal{F}_T} d(\mathbb{P}_{Q}(\mathcal{E}),\mathbb{P}_{P}(\mathcal{E})).
\end{align*}
\end{proposition}
Recall that $d(p, q)$ indicates the KL divergence between two Bernoulli distributions with parameters $p, q\in (0, 1)$. 

This ``transportation'' lemma provides the distribution-dependent characterization of events under a given bandit model $P$ and corresponding perturbed bandit model $P'$.

Between the true bandit model $P_0 \in \mathcal{P}$ and a bandit model $Q \in \mathcal{P}$, following the proof of Lemma~1 in \citet{Kaufman2016complexity}, we define the log-likelihood ratio as
\begin{align*}
    L_T = \sum^T_{t=1} \sum_{a\in[K]}\mathbbm{1}[ A_t = a] \log \left(\frac{f^a_{Q}(Y^a_{t}| X_t)\zeta_{Q}(X_t)}{f^a_{P_0}(Y^a_{t}| X_t)\zeta_{P_0}(X_t)}\right).
\end{align*}
For this log-likelihood ratio, from Lemma~\ref{lem:data_proc_inequality}, between the true model $P_0$, we have
\begin{align*}
    \mathbb{E}_{Q}[L_T] \ge \sup_{\mathcal{E} \in \mathcal{F}_T} d(\mathbb{P}_{Q}(\mathcal{E}),\mathbb{P}_{P_{0}}(\mathcal{E})).
\end{align*}

We consider an approximation of $\mathbb{E}_{Q}[L_T]$ under an appropriate alternative hypothesis $Q\in\mathcal{P}$ when the gaps between the expected rewards of the best treatment arm and suboptimal treatment arms are small.

\subsection{Observed-Data Bandit Models}
\label{sec:obs_data_bandit}
Next, we define a semiparametric model for observed data $(Y_t, A_t, X_t)$, as we can only observe the triple $(Y_t, A_t, X_t)$ and cannot observe the full-data $(Y^1_t, Y^2_t,\dots, Y^K_t, X_t)$. 

For each $x\in\mathcal{X}$, let us define the average allocation ratio under a bandit model $P\in\mathcal{P}$ and a BAI strategy as
\begin{align*}
    \frac{1}{T}\sum^T_{t=1}\mathbb{E}_{P}\left[ \mathbbm{1}[A_t = a]| X_t = x\right] = \kappa_{T, P}(a|x)
\end{align*}
This quantity represents the average sample allocation to each treatment arm $a$ under a strategy. Then, we first show the following lemma. We show the proof in Appendix~\ref{appdx:proof:lem_extnd_infinite}.
\begin{lemma}
\label{lem:kauf_lemma_extnd_infinite}
Suppose that Assumption~\ref{asm:bounded_mean_variance} holds. For $P_0, Q, P\in\mathcal{P}$, 
\begin{align*}
\frac{1}{T}\mathbb{E}_{P}[L_T] = \sum_{a\in[K]}\mathbb{E}_{P}\left[\mathbb{E}_{P}\left[\log \frac{f^a_{Q}(Y^a_t| X_t)\zeta_{Q}(X)}{f^a_{P_0}(Y^a_t| X)\zeta_{P_0}(X)}|X_t\right]\kappa_{T, P}(a| X_t)\right].
\end{align*}
\end{lemma}
Based on Lemma~\ref{lem:kauf_lemma_extnd_infinite}, for some $\kappa \in\mathcal{W}$, we consider the following samples $\{(\overline{Y}_t, \overline{A}_t, X_t)\}^T_{t=1}$, instead of $\{(Y_t, A_t, X_t)\}^T_{t=1}$, generated as
\begin{align*}
    \{(\overline{Y}_t, \overline{A}_t, X_t)\}^T_{t=1} \iid r(y, d, x) =  \prod_{a\in[K]}\left\{f^a_{P}(y^a| x)\kappa(a| x) \right\}^{\mathbbm{1}[d=a]}\zeta_{P}(x),
\end{align*}
where $\kappa(a| x)(a|x)$ corresponds to the conditional expectation of $\mathbbm{1}[\overline{A}_t = a]$ given $X_t$. The expectation of $L_T$ for $\{(\overline{Y}_t, \overline{A}_t, X_t)\}^T_{t=1}$ on $P$ is identical to that for $\{(Y_t, A_t, X_t)\}^T_{t=1}$ from the result of Lemma~\ref{lem:kauf_lemma_extnd_infinite} when $\kappa = \kappa_{T, P}$. Therefore, to derive the lower bound for $\{(Y_t, A_t, X_t)\}^T_{t=1}$, we consider that for $\{(\overline{Y}_t, \overline{A}_t, X_t)\}^T_{t=1}$. 
Note that this data generating process is induced by a full-data bandit model $P\in \mathcal{P}$; therefore, we call it an observed-data bandit model. 

Formally, for a bandit model $P\in\mathcal{P}$ and some $\kappa\in\mathcal{W}$, by using a density function of $P$, let $\overline{R}^{\kappa}_P$ be a distribution of an observed-data bandit model $\{(\overline{Y}_t, \overline{A}_t, X_t)\}^T_{t=1}$ with the density given as
\begin{align*}
    &\overline{r}^{\kappa}_P(y, d, x) = \prod_{a\in[K]} \left\{f^a_P(y| x)\kappa(a| x)\right\}^{\mathbbm{1}[d=a]}\zeta_P(x).
\end{align*}
We call it an observed-data distribution.
To avoid the complexity of the notation, we will denote $\{(\overline{Y}_t, \overline{A}_t, X_t)\}^T_{t=1}$ as $\{(Y_t, A_t, X_t)\}^T_{t=1}$ in the following arguments. Let $\mathcal{R} = \big\{\overline{R}_P: P \in \mathcal{P}\big\}$ be a set of all observed-data bandit models $\overline{R}_P$. For $P_0\in\mathcal{P}$, let $\overline{R}^{\kappa}_{P_0} = \overline{R}^{\kappa}_0$, and $\overline{r}^{\kappa}_{P_0} = \overline{r}^{\kappa}_{0}$.

\subsection{Parametric Submodels for the Full-Data Bandit Models}
\label{sec:para_sub_full}
The purpose of this section is to introduce parametric submodels for the true full-data bandit model $P_0\in\mathcal{P}$, which is indexed by a real-valued parameter and a set of distributions contained in the larger set $\mathcal{P}$, and define the derivative of the parametric submodels. 

In Section~\ref{sec:para_sub_obs}, we define parametric submodels for observed-data bandit models under the true full-data bandit model, which is a set of distributions contained in the larger set $\mathcal{R}_0$, by using the parametric submodels for full-data bandit models. These definitions of parametric submodels are preparations for the series expansion of the log-likelihood; that is, we consider approximation of the log-likelihood $L_T = \sum^T_{t=1} \sum_{a\in[K]}\mathbbm{1}[ A_t = a] \log \left(\frac{f^a_{Q}(Y^a_{t}| X_t)\zeta_{Q}(X_t)}{f^a_{P_0}(Y^a_{t}| X_t)\zeta_{P_0}(X_t)}\right)$ using $\mu^*_0 - \mu^a_0$, where $Q\in\mathcal{P}$ is an alternative bandit model. 

This section consists of the following two parts. 
In the first part, we define parametric submodels as \eqref{eq:parametric_submodel} with condition~\eqref{eq:const_ate}. Then, in the following part, we confirm the differentiability \eqref{eq:trans} and define score functions. 

\paragraph{Definition of parametric submodels for the observed-data distribution}
First, we define parametric submodels for the true full-data bandit model $P_0$ with the density function $p_{P_0}(y^1, \dots, y^K, x)$ by introducing a parameter $\bm{\varepsilon} = (\varepsilon^a)_{a\in[K]\backslash\{a^*_0\}}$  $\varepsilon^a \in \Theta$ with some compact space $\Theta$. We construct our parametric submodels so that the parameter can be interpreted as the gap parameter of a parametric submodel. 
For $P\in\mathcal{P}$, we define a set of parametric submodels $\left\{P_{\bm{\varepsilon}}: \bm{\varepsilon}\in\Theta^{K-1}\right\} \subset \mathcal{P}$ as follows: for a set of some functions $(g^a)_{a\in[K]\backslash\{a^*_0\}}$ such that $g^a:\mathbb{R}\times \mathbb{R}\times \mathcal{X} \to \mathbb{R}$, a parametric submodel $P_{\bm{\varepsilon}}$ has a density such that for each $a\in[K]\backslash\{a^*_0\}$, $g^a(\phi^*_\tau(y, x), \phi^a_\tau(y, x), x) = 0$, and 
\begin{align}
\label{eq:parametric_submodel}
&p_{\bm{\varepsilon}}(y^*, y^a, x) = 
\left( 1 + \varepsilon^a g^a\left(\phi^*_\tau(y, x), \phi^a_\tau(y, x), x\right) \right)p_{P_0}(y^*, y^a, x),
\end{align}
where for a constant $\tau > 0$ and each $d\in[K]$, $\phi^d_\tau:\mathbb{R}\times\mathcal{X}\to (-\tau, \tau)$ is a truncation function such that for $\varepsilon^a < c(\tau)$,
\begin{align*}
    &\phi^d_\tau(y, x) = y\mathbbm{1}[|y| < \tau] - \mathbb{E}_{P_0}[Y^d_t\mathbbm{1}[|Y^d_t| < \tau]|X_t = x] + \mu^d_0(x),\qquad |\varepsilon^a g^a\big(\phi^*_\tau(y), \phi^a_\tau(y), x\big) | < 1,
\end{align*}
and $c(\tau)$ is some decreasing scalar function with regard to $\tau$ such that for the inverse $c^{-1}(e) = \tau$, $\tau \to \infty$ as $e \to 0$. Let $\phi^{a^*_0}$ be denoted by $\phi^{*}$. 
This is a standard construction of parametric submodels with unbounded random variables \citep{Hansen2022}. For $a\in[K]\backslash\{a^*_0\}$,
this parametric submodel must satisfy $\mathbb{E}_{P_0}[g^a(\phi^*_\tau(Y_t, X_t), \phi^a_\tau(Y_t, X_t), X_t)] = 0$, $\mathbb{E}_{P_0}[(g^a(\phi^*_\tau(Y_t, X_t), \phi^a_\tau(Y_t, X_t), X_t))^2] < \infty$,
and
\begin{align}
\label{eq:const_ate}
\int \int \left(y^* - y^a\right)p_{\bm{\varepsilon}}(y^*, y^a, x)  \mathrm{d}y^*\mathrm{d}y^a\mathrm{d}x = \mu^*_0 - \mu^a_0 + \varepsilon^a \quad.
\end{align}
In Section~\ref{sec:oberved-data}, we specify functions $(g^a)_{a\in[K]\backslash\{a^*_0\}}$ and confirm that the specified $g^a$ satisfies \eqref{eq:const_ate}. Note that the parametric submodels are usually not unique. For each $a\in[K]\backslash\{a^*_0\}$, the parametric submodel $p_{\bm{\varepsilon}}(y^*, y^a, x)$ is equivalent to $p_{P_0}(y^*, y^a, x)$ when $\varepsilon^a = 0$ for any $(\varepsilon^e)_{e\in[K]\backslash\{a^*_0, a\}}$. 

For each $a\in[K]\backslash\{a^*_0\}$ and a parametric submodel $P_{\bm{\varepsilon}}$, let $f^*_{\bm{\varepsilon}}(y| x)$, $f^a_{\bm{\varepsilon}}(y| x) = f^a_{\varepsilon^a}(y| x)$ and $\zeta_{\bm{\varepsilon}}(x)$ be the conditional densities of $Y^*_t$ and $Y^a_t$ given $X_t = x$ and the density of $X_t$, which satisfies \eqref{eq:parametric_submodel} and \eqref{eq:const_ate} as
\begin{align*}
    &p_{\bm{\varepsilon}}(y^*, y^a, x) = f^*_{\bm{\varepsilon}}(y| x)f^a_{\varepsilon^a}(y| x)\zeta_{\bm{\varepsilon}}(x),\\
    &\int \int \left(y^* - y^a\right) f^*_{\bm{\varepsilon}}(y| x) f^a_{\varepsilon^a}(y| x)\zeta_{\bm{\varepsilon}}(x) \mathrm{d}y^*\mathrm{d}y^a\mathrm{d}x = \mu^*_0 - \mu^a_0 + \varepsilon^a.
\end{align*}
According to the definition of the parametric submodels, $f^*_{\bm{0}}(y| x) = f^*_{P_0}(y| x)$, $f^a_{\bm{0}}(y| x) = f^a_{0}(y| x) = f^a_{P_0}(y| x)$ and $\zeta_{\bm{\varepsilon}}(x) = \zeta_{P_0}(x)$. 

\paragraph{Differentiablity and score functions of the parametric submodels for the observed-data distribution.}
Next, we confirm the differentiablity of $p_{\bm{\varepsilon}}(y^*, y^a, x)$. 
Because $\sqrt{p_{\bm{\varepsilon}}(y^*, y^a, x)}$ is continuously differentiable for every $(y^*, y^a, x)$, and $\int \left( \frac{\dot{p}_{\bm{\varepsilon}}(y^*, y^a, x)}{p_{\bm{\varepsilon}}(y^*, y^a, x)}\right)^2p_{\bm{\varepsilon}}(y^*, y^a, x) \mathrm{d}m$ are well defined and continuous in $\bm{\varepsilon}$, where $m$ is some reference measure on $(y^*, y^a, x)$, from Lemma~7.6 of \citet{Vaart1998}, we see that the parametric submodel has the score function $g^a$ in the $L_2$ sense; that is, the density $p_{\bm{\varepsilon}}(y^*, y^a, x)$ is differentiable in quadratic mean (DQM): for $a\in[K]\backslash\{a^*_0\}$, and any $(\varepsilon^b)_{b\in[K]\backslash\{a^*_0, a\}}$, 
\begin{align}
    &\int\left[p^{1/2}_{\bm{\varepsilon}}(y^*, y^a, x) -  p^{1/2}_{P_0}(y^*, y^a, x)  - \frac{1}{2}\varepsilon^a g^a(\phi^*_\tau(y, x), \phi^a_\tau(y, x), x) p^{1/2}_{P_0}(y^*, y^a, x) \right]^2\mathrm{d}m = o\left(\varepsilon^a\right).
\end{align}
This relationship is derived from 
\begin{align*}
    &\frac{\partial}{\partial \varepsilon^a}\Big|_{\varepsilon^a = 0}\log p_{\bm{\varepsilon}}(y^*, y^a, x) =  \frac{g^a(\phi^*_\tau(y, x), \phi^a_\tau(y, x), x)}{ 1 + \varepsilon^a g^a(\phi^*_\tau(y, x), \phi^a_\tau(y, x), x)}
    \Big|_{\varepsilon^a = 0}
    =
    g^a(\phi^*_\tau(y, x), \phi^a_\tau(y, x), x),
\end{align*}
for any $(\varepsilon^b)_{b\in[K]\backslash\{a^*_0, a\}}$.

To clarify the relationship between $g^a$ and a score function, for each $a\in[K]\backslash\{a^*_0\}$, and any $(\varepsilon^b)_{b\in[K]\backslash\{a^*_0, a\}}$, we express the score function as
\begin{align*}
    g^a(\phi^*_\tau(y, x), \phi^a_\tau(y, x), x) &= \frac{\partial}{\partial \varepsilon^a}\Big|_{\varepsilon^a=0} \log p_{\bm{\varepsilon}}(y^*, y^a, x) = S^{a, a^*_0}_{f}(y|x) +
   S^{a, a}_{f}(y|x) + S^a_{\zeta}(x),
\end{align*}
where
\begin{align*}
    &S^{a, a^*_0}_{f}(y|x) = \frac{\partial}{\partial \varepsilon^a}\Big|_{\varepsilon^a = 0} \log f^*_{\bm{\varepsilon}}(y| x),\quad S^{a, a}_{f}(y|x) = \frac{\partial}{\partial \varepsilon^a}\Big|_{\varepsilon^a = 0} \log f^a_{\varepsilon^a}(y| x),\quad S^a_{\zeta}(x) = \frac{\partial}{\partial \varepsilon^a}\Big|_{\varepsilon^a = 0} \log \zeta_{\bm{\varepsilon}}(x).
\end{align*}

\subsection{Mapping from Observed-Data to Full-Data Bandit Models}
\label{sec:mapping_obs}
According to Section~7.2 of \citet{Tsiatis2007semiparametric}, we define a mapping from full-data to observed-data as $(y, x) = \mathcal{T}^d(y^*, y^a, x)$, where $\mathcal{T}^d:\mathbb{R}^2 \times \mathcal{X} \to  \mathbb{R}\times \mathcal{X}$ is a known many-to-one function, which maps the full-data $(y^*, y^a, x)$ to observed-data bandit models $(y^d, x)$. 
We only consider a case where $(Y^*_t, Y^a_t, X_t)$ is continuous and define a function $V^d:\mathbb{R}^2 \to \mathbb{R}$ as a counterfactual value of the observation; that is, $V^d(Y^*_t, Y^a_t) = ((Y^b_t)_{b\in\{a^*_0, a\}\backslash\{d\}})$. Then, the mapping 
\begin{align*}
    (Y^*_t, Y^a_t, X_t) \mapsto \{\mathcal{T}^{d}(Y^*_t, Y^a_t, X_t), V^{d}(Y^*_t, Y^a_t)\}
\end{align*}
is one-to-one for all $a\in[K]\backslash\{a^*_0\}$ and $d\in\{a^*_0, a\}$. For $a\in[K]\backslash\{a^*_0\}$, $d\in\{a^*_0, a\}$, $\tau^d = (y^d, x)$, and $v^d = ((y^b)_{b\in\{a^*_0, a\}\backslash\{d\}})$, which correspond to $\mathcal{T}^d$ and $V^d$ respectively, we define the inverse transformation as
\begin{align}
    (y^*, y^a, x) = H^d(\tau^d, v^d),
\end{align}
Then, by the standard formula for change of variables, let us define the density of $(\tau^d, v^d)$ under $\mathcal{T}^d$ and $V^d$ as
\begin{align}
\label{eq:p_h_trans}
    p_{\mathcal{T}^d, V^d}(\tau^d, v^d) = p_P(H^d(\tau^d, v^d))J(\tau^d, v^d),
\end{align}
where $J$ is the Jacobian of $H^d$ with respect to $(\tau^d, v^d)$. To find the density of the observed data $\overline{r}^\kappa_P(y, d, x)$, we can use
\begin{align}
\label{eq:main_p_h_trans}
    \overline{r}^\kappa_P(y, d, x) = \int \overline{r}^\kappa_{P, V^d}(\tau^d, d, v^d)dv^d,
\end{align}
where
\begin{align}
\label{eq:p_h_trans_prob}
   \overline{r}^\kappa_{P, V^d}(\tau^d, d, v^d) = \kappa(d|x) p_{\mathcal{T}^d, V^d}(\tau^d, v^d).
\end{align}
Consequently, using \eqref{eq:p_h_trans} and \eqref{eq:p_h_trans_prob}, we can rewrite \eqref{eq:main_p_h_trans} as
\begin{align}
    &\overline{r}^\kappa_{P}(y, d, x)=\int  \kappa\big(d| x\big) p_P(H^d(\tau^d, v^d))J(\tau^d, v^d) \mathrm{d}v^d.
\end{align}

\subsection{Parametric Submodels for the Observed-Data Bandit Models and Tangent Space}
\label{sec:para_sub_obs}
This section consists of the following three parts. 
In the first part, we define parametric submodels as \eqref{eq:parametric_submodel} with condition~\eqref{eq:const_ate}. Then, in the following part, we confirm the differentiability \eqref{eq:trans} and define score functions. Finally, we define a set of score functions, called a tangent set in the final paragraph. 

By using the parametric submodels and tangent set, in Section~\ref{sec:semiparametric_lratio}, we demonstrate the series expansion of the log-likelihood (Lemma~\ref{lem;taylor_exp_semipara}). In this section and Section~\ref{sec:semiparametric_lratio}, we abstractly provide definitions and conditions for the parametric submodels and do not specify them. However, in Sections~\ref{sec:oberved-data} and \ref{sec:specification-score}, we show a concrete form of the parametric submodel by finding score functions satisfying the conditions imposed in this section. 

By using the parametric submodels for the true full-data bandit model $P_0\in\mathcal{P}$ in Section~\ref{sec:para_sub_full}, we define parametric submodels for observed-data bandit models under the true full-data bandit model $P_0\in\mathcal{P}$. Because we define the density functions of the parametric submodel of the true full-data bandit model, the parametric submodels for the observed-data bandit models are given as follows: 
\begin{align*}
    &\overline{r}^{\kappa}_{\bm{\varepsilon}}(y, a, x) = f^a_{\varepsilon^a}(y| x)\kappa(a| x)\zeta_{\bm{\varepsilon}}(x)\qquad \forall a \in [K]\backslash\{a^*_0\},\\
    &\overline{r}^{\kappa}_{\bm{\varepsilon}}(y, a^*_0, x) = f^*_{\bm{\varepsilon}}(y| x)\kappa(a^*_0| x)\zeta_{\bm{\varepsilon}}(x).
\end{align*}

\paragraph{Differentiablity and score functions of the parametric submodels for the observed-data distribution.}
Next, we confirm the differentiablity of $\overline{r}^{\kappa}_{\bm{\varepsilon}}(y, d, x)$. 
Because $\sqrt{\overline{r}^{\kappa}_{\bm{\varepsilon}}(y, d, x)}$ is continuously differentiable for every $y, x$ given $d\in[K]$, and $\int \left( \frac{\dot{\overline{r}}^{\kappa}_{\bm{\varepsilon}}(y, d, x)}{\overline{r}^{\kappa}_{\bm{\varepsilon}}(y, d, x)}\right)^2\overline{r}^{\kappa}_{\bm{\varepsilon}}(y, d, x) \mathrm{d}m$ are well defined and continuous in $\bm{\varepsilon}$, where $m$ is some reference measure on $(y, d, x)$, from Lemma~7.6 of \citet{Vaart1998}, we see that the parametric submodel has the score function $g^a$ in the $L_2$ sense; that is, the density $\overline{r}^{\kappa}_{\bm{\varepsilon}}(y, d, x)$ is differentiable in quadratic mean (DQM): for $a\in[K]\backslash\{a^*_0\}$, $d\in\{a^*_0, a\}$, and any $(\varepsilon^b)_{b\in[K]\backslash\{a^*_0, a\}}$,

Then we show the differentiablity in quadratic mean at $\varepsilon^a = 0$ of $\overline{r}^{\kappa, 1/2}_{\bm{\varepsilon}}$ in the following lemma. 
We show the proof in Appendix~\ref{appdx:proof_lem:trans}. 
\begin{lemma}\label{lem:trans} Under Assumption~\ref{asm:bounded_mean_variance}, for $a\in[K]\backslash\{a^*_0\}$ and $d\in\{a^*_0, a\}$, 
\begin{align}
\label{eq:trans}
    &\int\left[ \overline{r}^{\kappa\ 1/2}_{\bm{\varepsilon}}(y, d, x) -  \overline{r}^{\kappa\ 1/2}_0(y, d, x) - \frac{1}{2}\varepsilon^a S^a(y, d, x)\overline{r}^{\kappa\ 1/2}_0(y, d, x) \right]^2\mathrm{d}m = o\left(\varepsilon^a\right).
\end{align}
where 
\begin{align}
    S^a(y, d, x) = \mathbb{E}_{P_0}\left[g^a\left(\phi^*_\tau(Y^*_t, X_t), \phi^a_\tau(Y^a_t, X_t), X_t\right)  | \mathcal{T}^d(Y^*_t, Y^*_t, X_t) = (y, x)\right].
\end{align}
\end{lemma}

In the following section, we specify a measurable function $S^a$ wigh $g^a$, satisfying the conditions \eqref{eq:parametric_submodel} and \eqref{eq:const_ate}, which corresponds to a score function of $\overline{r}^{\kappa}_{0}(y, a, x)$ and $\overline{r}^{\kappa}_{\bm{\varepsilon}}(y, a^*_0, x) $ for each $a\in[K]\backslash\{a^*_0\}$. To clarify the relationship between $g^a$ and a score function, for each $a\in[K]\backslash\{a^*_0\}$, and any $(\varepsilon^b)_{b\in[K]\backslash\{a^*_0, a\}}$, we denote the score function as
\begin{align*}
    S^a(y, d, x) &= \frac{\partial}{\partial \varepsilon^a}\Big|_{\varepsilon^a=0} \log \overline{r}^{\kappa}_{\bm{\varepsilon}}(y, d, x) = 
    \mathbbm{1}[d = a^*_0]S^{a, a^*_0}_{f}(y|x) + \mathbbm{1}[d = a]S^{a, a}_{f}(y|x) + S^{a}_{\zeta}(x)\quad \forall d\in\{a^*_0, a\},\\
    S^a(y, d, x) &= 0\quad \forall d\in[K]\backslash\{a^*_0, a\}.
\end{align*}
Note that $\frac{\partial}{\partial \varepsilon^a} \log\kappa(a| x) = 0$. 

\paragraph{Definition of the tangent set.} Recall that parametric submodels and corresponding score functions are not unique. Here, we consider a set of score functions. 
For a set of the parametric submodels $\left\{\overline{R}^{\kappa}_{\bm{\varepsilon}}: \bm{\varepsilon}\in\Theta^{K-1}\right\}$, we obtain a corresponding set of score functions $g^a$ in the Hilbert space $L_2(\overline{R}_Q)$, which we call a tangent set of $\mathcal{R}$ at $\overline{R}^{\kappa}_0$ and denote it by $\dot{\mathcal{R}}^a$. Because $\mathbb{E}_{\overline{R}^{\kappa}_0}[(g^a(\phi^{A_t}_\tau(Y_t, X_t), A_t, X_t))^2]$ is automatically finite, the tangent set can be identified with a subset of the Hilbert space $L_2(\overline{R}^{\kappa}_0)$, up to equivalence classes. For our parametric submodels, the tangent set at $\overline{R}^{\kappa}_0$ in $L_2(\overline{R}^{\kappa}_0)$ is given as
\begin{align*}
    \dot{\mathcal{R}}^a = \left\{
    \mathbbm{1}[d = a^*_0]S^{a,  a^*_0}_{f}(y|x) + \mathbbm{1}[d = a]S^{a, a}_{f}(y|x) + S^a_{\zeta}(x)
    \right\}.
\end{align*}

\subsection{Alternative Bandit Model}
\label{sec:alter}
Then, we define a class of alternative hypotheses. To derive a tight lower bound by applying the change-of-measure arguments, we use an appropriately defined alternative hypothesis. Our alternative hypothesis is defined using the parametric submodel of $P_0$ as follows:
\begin{definition}
Let $\Alt(P_0) \subset \mathcal{P}$ be alternative bandit models such that for all $Q \in \Alt(P_0)$, $a^*(Q) \neq a^*_0$, and $\overline{R}^{\kappa_{T, Q}}_{\bm{\varepsilon}} = \overline{R}^{\kappa_{T, Q}}_Q$, where $\bm{\varepsilon} = (\varepsilon^a)_{a\in[K]\backslash\{a^*_0\}}$, $\varepsilon^a =  \left(\mu^{a^*_0}(Q) - \mu^a(Q)\right) - \left(\mu^*_0 - \mu^a_0\right)$. 
\end{definition}
This also implies that for all $Q \in \Alt(P_0)$, for all $a\in[K]\backslash\{a^*_0\}$, $\mu^*_0 - \mu^a_0 > 0$ and there exists $a\in[K]\backslash\{a^*_0\}$ such that $\mu^{a^*_0}(Q) - \mu^a (Q) < 0$. Let $\mu^{a^*_0}(Q)$ be denoted by $\mu^*(Q)$.

\subsection{Derivation of a Lower Bound of the Probability of Misidentification}
\label{sec:deriv_lower}
Here, we derive a lower bound for the probability of misidentification as follows, which is refined later:
\begin{lemma}
\label{lem:semipara_bandit_lower_bound}
Under Assumption~\ref{asm:bounded_mean_variance}, for any $P_0 \in \mathcal{P}$ and $Q \in  \Alt(P_0)$, any consistent and asymptotically invariant strategy satisfies
\begin{align*}
    &\limsup_{T\to\infty}-\frac{1}{T}\log \mathbb{P}_{ P_0 }(\widehat{a}_T \neq a^*_0)\\
    &\leq \sup_{w \in \mathcal{W}}\min_{a\in[K]\backslash\{a^*_0\}}\inf_{\varepsilon^a < - \left(\mu^{*}_0 - \mu^a_0\right)}\sum_{a\in\{a^*_0, a\}}\mathbb{E}_{\overline{R}_{\bm{\varepsilon}}}\left[\mathbb{E}_{\overline{R}_{\bm{\varepsilon}}}\left[\log \frac{f^a_{\bm{\varepsilon}}(Y^a_t| X_t)\zeta_{\bm{\varepsilon}}(X)}{f^a_{P_0}(Y^a_t| X)\zeta_{P_0}(X)}|X_t\right]w(a| X_t)\right].
\end{align*}
\end{lemma}

\begin{proof}[Proof of Lemma~\ref{lem:semipara_bandit_lower_bound}] For each $Q \in  \Alt(P_0)$, $\mathbb{E}_{Q}[L_T] \ge \sup_{\mathcal{E} \in \mathcal{F}_T} d(\mathbb{P}_{Q}(\mathcal{E}),\mathbb{P}_{P_0}(\mathcal{E}))$ holds from Proposition~\ref{lem:data_proc_inequality}. 
Let $\mathcal{E} = \{\widehat{a}_T = a^*_0\}$. 
Because we assume that the strategy is consistent and asymptotically invariant for both models and from the definition of $\Alt(P_0)$, for each $\epsilon_1  \in (0, 1)$ and $\epsilon_2  > 0$,
there exists $t_0 (\epsilon_1, \epsilon_2)$ such that for all $T \geq t_0 (\epsilon_1)$, $\mathbb{P}_{Q}(\mathcal{E}) \le \epsilon_1 \le\mathbb{P}_{ P_0 }(\mathcal{E})$, and $\kappa_{T, Q}(a| X_t) \leq \kappa_{T, P}(a| X_t) + \epsilon_2$. 
Then, for all $T \ge t_0(\epsilon_1, \epsilon_2)$, $\mathbb{E}_{Q}[L_T]
\ge d (\epsilon_1, 1 -\mathbb{P}_{P_0}(\widehat{a}_T \neq a^*_0)) =  \epsilon \log \frac{\epsilon}{1 -\mathbb{P}_{ P_0 }(\widehat{a}_T \neq a^*_0)} + (1 - \epsilon_1) \log \frac{1 - \epsilon_1}{\mathbb{P}_{P_0}(\widehat{a}_T \neq a^*_0)}$. 
Then, taking the limsup and letting $\epsilon_1, \epsilon_2 \to 0$, 
\begin{align*}
    &\limsup_{T\to\infty}-\frac{1}{T}\log \mathbb{P}_{ P_0 }(\widehat{a}_T \neq a^*_0)\leq \inf_{Q \in \Alt(P_0)}\limsup_{T\to\infty}\frac{1}{T}\mathbb{E}_{Q}[L_T]\\
    &\leq \inf_{Q \in \Alt(P_0)}\limsup_{T\to\infty}\sum_{a\in[K]}\mathbb{E}_{Q}\left[\mathbb{E}_{Q}\left[\log \frac{f^a_{Q}(Y^a_t| X_t)\zeta_{Q}(X)}{f^a_{P_0}(Y^a_t| X)\zeta_{P_0}(X)}|X_t\right]\kappa_{T, P}(a| X_t)\right]\\
    &\leq \inf_{Q \in \Alt(P_0)}\limsup_{T\to\infty}\sup_{w\in\mathcal{W}}\sum_{a\in[K]}\mathbb{E}_{Q}\left[\mathbb{E}_{Q}\left[\log \frac{f^a_{Q}(Y^a_t| X_t)\zeta_{Q}(X)}{f^a_{P_0}(Y^a_t| X)\zeta_{P_0}(X)}|X_t\right]w(a| X_t)\right]\\
    &\leq \sup_{w \in \mathcal{W}}\inf_{Q \in \Alt(P_0)}\sum_{a\in[K]}\mathbb{E}_{Q}\left[\mathbb{E}_{Q}\left[\log \frac{f^a_{Q}(Y^a_t| X_t)\zeta_{Q}(X)}{f^a_{P_0}(Y^a_t| X)\zeta_{P_0}(X)}|X_t\right]w(a| X_t)\right]\\
    &= \sup_{w \in \mathcal{W}}\min_{a\in[K]\backslash\{a^*_0\}}\inf_{\substack{Q \in \mathcal{P}\\ \mu^*(Q) - \mu^a(Q) < 0} }\sum_{a\in[K]}\mathbb{E}_{Q}\left[\mathbb{E}_{Q}\left[\log \frac{f^a_{Q}(Y^a_t| X_t)\zeta_{Q}(X)}{f^a_{P_0}(Y^a_t| X)\zeta_{P_0}(X)}|X_t\right]w(a| X_t)\right].
\end{align*}
By using $\varepsilon^a = \left(\mu^*(Q) - \mu^a(Q)\right) - \left(\mu^*_0 - \mu^a_0\right) < - \left(\mu^*_0 - \mu^a_0\right)$ for the parametric submodel,
\begin{align*}
    &\sup_{w \in \mathcal{W}}\min_{a\in[K]\backslash\{a^*_0\}}\inf_{\substack{Q \in \mathcal{P}\\ \mu^*(Q) - \mu^a(Q) < 0} }\sum_{a\in[K]}\mathbb{E}_{Q}\left[\mathbb{E}_{Q}\left[\log \frac{f^a_{Q}(Y^a_t| X_t)\zeta_{Q}(X)}{f^a_{P_0}(Y^a_t| X)\zeta_{P_0}(X)}|X_t\right]w(a| X_t)\right]\nonumber\\
    &= \sup_{w \in \mathcal{W}}\min_{a\in[K]\backslash\{a^*_0\}}\inf_{\substack{\varepsilon^a < - \left(\mu^*_0 - \mu^a_0\right)\nonumber\\
    \forall b\in[K]\backslash\{a^*_0, a\}\ \varepsilon^b = 0} }\sum_{a\in[K]}\mathbb{E}_{\overline{R}_{\bm{\varepsilon}}}\left[\mathbb{E}_{\overline{R}_{\bm{\varepsilon}}}\left[\log \frac{f^a_{\bm{\varepsilon}}(Y^a_t| X_t)\zeta_{\bm{\varepsilon}}(X)}{f^a_{P_0}(Y^a_t| X)\zeta_{P_0}(X)}|X_t\right]w(a| X_t)\right]\nonumber\\
    &= \sup_{w \in \mathcal{W}}\min_{a\in[K]\backslash\{a^*_0\}}\inf_{\varepsilon^a < - \left(\mu^{*}_0 - \mu^a_0\right)}\sum_{a\in\{a^*_0, a\}}\mathbb{E}_{\overline{R}_{\bm{\varepsilon}}}\left[\mathbb{E}_{\overline{R}_{\bm{\varepsilon}}}\left[\log \frac{f^a_{\bm{\varepsilon}}(Y^a_t| X_t)\zeta_{\bm{\varepsilon}}(X)}{f^a_{P_0}(Y^a_t| X)\zeta_{P_0}(X)}|X_t\right]w(a| X_t)\right].
\end{align*}
The proof is complete.
\end{proof}

\subsection{Semiparametric Likelihood Ratio}
\label{sec:semiparametric_lratio}
For $a\in [K]\backslash\{a^*_0\}$, let $\bm{\varepsilon}$ be $(0,\dots, 0, \varepsilon^a, 0, \dots, 0)$. Let us also define 
\begin{align*}
    L^{a}_T &= \sum^T_{t=1} \left\{\mathbbm{1}[ A_t = a^*_0] \log \left(\frac{f^*_{\bm{\varepsilon}}(Y^*_{t}| X_t)}{f^*_{P_0}(Y^*_{t}| X_t)}\right) + \mathbbm{1}[ A_t = a] \log \left(\frac{f^a_{\bm{\varepsilon}}(Y^a_{t}| X_t)}{f^a_{P_0}(Y^a_{t}| X_t)}\right) + \log \left(\frac{\zeta_{\bm{\varepsilon}}(X_t)}{\zeta_{P_0}(X_t)}\right)\right\}.
\end{align*}
We consider series expansion of the log-likelihood $L^a_T$ defined between $P_0\in\mathcal{P}$ and $Q\in\Alt(P_0)$, where $\mathbb{E}_{Q}\left[L_T\right]$ works as a lower bound for the probability of misidentification as shown in Section~\ref{sec:deriv_lower}. We consider an approximation of $L^a_T$ under a small-gap regime (small $\mu^*_0 - \mu^a_0$), which is upper-bounded by the variance of the score function. Our argument is 
inspired by that in \citet{Murphy1997}.

Then, we prove the following lemma:
\begin{lemma}
\label{lem;taylor_exp_semipara}
Suppose that Assumption~\ref{asm:bounded_mean_variance} holds. For $P_0\in\mathcal{P}$, $Q\in\Alt(P_0)$, and each $a\in[K]\backslash\{a^*_0\}$, 
\begin{align*}
    \frac{1}{T}\mathbb{E}_{Q}\left[L^{a}_T\right] = \frac{\left(\varepsilon^a\right)^2}{2}\mathbb{E}_{P_0}\left[ \left(S^a(Y_t, A_t, X_t)\right)^2\right]  + o\left(\left(\varepsilon^a\right)^2\right). 
\end{align*}
\end{lemma}
To prove this lemma, for $a\in[K]\backslash\{a^*_)\}$ and $d\in[K]$, we define
\begin{align*}
    \ell^{a}_{\bm{\varepsilon}}(y, d, x) &= \mathbbm{1}[d = a^*_0]\log f^*_{\bm{\varepsilon}}(y|x) + \mathbbm{1}[d = a]\log f^a_{\varepsilon^a}(y|x) + \log \zeta_{\bm{\varepsilon}}(x).
\end{align*}
Note that if $\varepsilon^a = 0$, then
\begin{align*}
    \ell^{a}_{\bm{\varepsilon}}(y, d, x) &= \mathbbm{1}[d = a^*_0]\log f^*_{P_0}(y|x) + \mathbbm{1}[d = a]\log f^a_{P_0}(y|x) + \log \zeta_{P_0}(x).
\end{align*}

\begin{proof}[Proof of Lemma~\ref{lem;taylor_exp_semipara}]
By using the parametric submodel defined in the previous section, from the series expansion, 
\begin{align*}
   L^{a}_T &= \sum^T_{t=1} \left\{\mathbbm{1}[ A_t = a^*_0] \log \left(\frac{f^*_{\bm{\varepsilon}}(Y^*_{t}| X_t)}{f^*_{P_0}(Y^*_{t}| X_t)}\right) + \mathbbm{1}[ A_t = a] \log \left(\frac{f^a_{\bm{\varepsilon}}(Y^a_{t}| X_t)}{f^a_{P_0}(Y^a_{t}| X_t)}\right) + \log \left(\frac{\zeta_{\bm{\varepsilon}}(X_t)}{\zeta_{P_0}(X_t)}\right)\right\}\\
   &= \sum^T_{t=1}\left\{\frac{\partial}{\partial \varepsilon^a} \Big|_{\varepsilon^a=0} \ell^a_{\bm{\varepsilon}}(Y_t, A_t, X_t)\varepsilon^a + \frac{\partial^2}{\partial (\varepsilon^a)^2}\Big|_{\varepsilon^a=0} \ell^a_{\bm{\varepsilon}}(Y_t, A_t, X_t)\frac{\left(\varepsilon^a\right)^2}{2} + C\left(\varepsilon^a\right)^3\right\},
\end{align*}
where $C$ is a constant, independent from $\varepsilon^a$.
Here, we fix $(\varepsilon^b)_{b\in[K]\backslash\{a^*_0, a\}}$, where $\varepsilon^b = 0$. Note that
\begin{align*}
    &\frac{\partial}{\partial \varepsilon^a } \Big|_{\varepsilon^a = 0} \ell^a_{\bm{\varepsilon}}(y, d, x) = S^a(y, d, x)\\
    &\frac{\partial}{\partial (\varepsilon^a)^2}\Big|_{\varepsilon^a = 0} \ell^a_{\bm{\varepsilon}}(y, d, x) = - \left(S^a(y, d, x)\right)^2.
\end{align*}
Let $\overline{R}^{\kappa_{T, Q}}_{\bm{\varepsilon}} = \overline{R}_{\bm{\varepsilon}}$, $\overline{r}^{\kappa_{T, Q}}_{\bm{\varepsilon}}(y, d, x) = \overline{r}_{\bm{\varepsilon}}(y, d, x)$, and  $\overline{r}^{\kappa_{T, Q}}_{0}(y, d, x) = \overline{r}_{0}(y, d, x)$. Then,
\begin{align*}
    &\mathbb{E}_{Q}\left[ S^a(Y_t, A_t, X_t)\right] = \mathbb{E}_{\overline{R}_{\bm{\varepsilon}}}\left[ S^a(Y_t, A_t, X_t)\right]\\
    &=\mathbb{E}_{\overline{R}_{\bm{\varepsilon}}}\left[ S^a(Y_t, A_t, X_t)\right] - \sum_{d\in[K]}\int S^a(y, d, x) \left(1 + \frac{1}{2}\varepsilon^a S^a(y, d, x)\right)^2\overline{r}_{0}(y, d, x) \mathrm{d}y\mathrm{d}x\\
    &\ \ \ +  \sum_{d\in[K]}\int S^a(y, d, x) \left(1 + \frac{1}{2}\varepsilon^a S^a(y, d, x)\right)^2\overline{r}_{0}(y, d, x) \mathrm{d}y\mathrm{d}x\\
    &= \sum_{d\in[K]}\int S^a(y, d, x) \left\{\overline{r}_{\bm{\varepsilon}}(y, d, x) - \left(1 + \frac{1}{2}\varepsilon^a S^a(y, d, x)\right)^2\overline{r}_{0}(y, d, x)\right\} \mathrm{d}y\mathrm{d}x\\
    &\ \ \ +  \sum_{d\in[K]}\int S^a(y, d, x) \left(1 + \frac{1}{2}\varepsilon^a S^a(y, d, x)\right)^2\overline{r}_{0}(y, d, x) \mathrm{d}y\mathrm{d}x\\
    &= \sum_{d\in[K]}\int S^a(y, d, x) \left\{\overline{r}_{\bm{\varepsilon}}(y, d, x) - \left(1 + \frac{1}{2}\varepsilon^a S^a(y, d, x)\right)^2\overline{r}_{0}(y, d, x)\right\} \mathrm{d}y\mathrm{d}x\\
    &\ \ \ + \mathbb{E}_{P_0}\left[ S^a(Y_t, A_t, X_t)\right] + \varepsilon^a \mathbb{E}_{P_0}\left[\left( S^a(Y_t, A_t, X_t)\right)^2\right] + \frac{1}{4}(\varepsilon^a)^2 \mathbb{E}_{P_0}\left[\left( S^a(Y_t, A_t, X_t)\right)^2\right] ,
\end{align*}
where we used 
\begin{align*}
    &\sum_{d\in[K]}\int S^a(y, d, x)\overline{r}_{0}(y, d, x) \mathrm{d}y\mathrm{d}x\\
    &= \sum_{d\in[K]}\int\left\{\mathbbm{1}[d = a^*_0]S^{a, a^*_0}_{f}(y|x) + \mathbbm{1}[d = a]S^{a, a}_{f}(y|x) + S^{a}_{\zeta}(x)\right\}\overline{r}_{0}(y, d, x) \mathrm{d}y\mathrm{d}x.
\end{align*}
Then, because the density $\overline{r}_{\bm{\varepsilon}}(y, d, x)$ is DQM \eqref{eq:trans}, as $\varepsilon^a \to 0$, 
\begin{align*}
    \mathbb{E}_{Q}\left[ S^a(Y_t, A_t, X_t)\right] - \mathbb{E}_{P_0}\left[ S^a(Y_t, A_t, X_t)\right] -  \varepsilon^a \mathbb{E}_{P_0}\left[\left( S^a(Y_t, A_t, X_t)\right)^2\right] =  o(\varepsilon^a).
\end{align*}
Similarly, 
\begin{align*}
    -\mathbb{E}_{Q}\left[\left( S^a(Y_t, A_t, X_t)\right)^2\right]  + \mathbb{E}_{P_0}\left[\left( S^a(Y_t, A_t, X_t)\right)^2\right] - \varepsilon^a \mathbb{E}_{P_0}\left[\left( S^a(Y_t, A_t, X_t)\right)^3\right]= o(\varepsilon^a).
\end{align*}

By using these expansions, we approximate $\mathbb{E}_{Q}\left[L_T\right]$. Here, by definition, $\mathbb{E}_{P_0}\left[ S^a(Y_t, A_t, X_t)\right] = 0$. Then, we approximate the likelihood ratio as follows:
\begin{align*}
    \frac{1}{T}\mathbb{E}_{Q}[L^{a}_T] - \frac{\left(\varepsilon^a\right)^2}{2}\mathbb{E}_{P_0}\left[ \left(S^a(Y_t, A_t, X_t)\right)^2\right] = o\left(\left( \varepsilon^a\right)^2\right).
\end{align*}
\end{proof}

\subsection{Observed-Data Semiparametric Efficient Influence Function}
\label{sec:oberved-data}
Our remaining task is to specify the score function $S^a$. Because there can be several score functions for our parametric submodel due to directions of the derivative, we find a parametric submodel that has a score function with the largest variance, called a least-favorable parametric submodel \citep{Vaart1998}. 

In this section, instead of the original observed-data bandit model $\overline{R}^{\kappa_{T, Q}}_{\bm{\varepsilon}}$, we consider an alternative observed-data bandit model  $\overline{R}^{\kappa_{T, Q}\,\dagger}_{0}$, which is a distribution of $\{(\phi^{A_t}_\tau(Y_t, X_t), A_t, X_t)\}^T_{t=1}$. Let $\overline{R}^{\kappa_{T, Q}\,\dagger}_{\bm{\varepsilon}}$ be parametric submodel defined as well as Section~\ref{sec:para_sub_obs}, $\mathcal{R}^{\kappa_{T,Q}\,\dagger}_{\bm{\varepsilon}}$ be a set of all $\overline{R}^{\kappa_{T, Q}\,\dagger}_{\bm{\varepsilon}}$, and $\overline{r}^{\kappa_{T,Q}\, \dagger}_{\bm{\varepsilon}}(y, d, x) = f^{d, \dagger}_{\varepsilon^d}(y| x)\kappa_{T, Q}(d| x)\zeta_{\bm{\varepsilon}}(x)$. For each $a\in[K]\backslash\{a^*_0\}$, let $S^{a\,\dagger}(y, d, x)$ and  $\dot{\mathcal{R}}^{a\,\dagger}$ be a corresponding score function and tangent space, respectively. 

As a preparation, we define a parameter $\mu^*(Q) - \mu^a(Q)$ as a function $\psi^a: \mathcal{R}^{\kappa_{T,Q}\,\dagger}_{\bm{\varepsilon}}\to \mathbb{R}$ such that $\psi^a(\overline{R}^{\kappa_{T,Q}\,\dagger}_{\bm{\varepsilon}})=\mu^*_0 - \mu^a_0 + \varepsilon^a$.
The information bound for $\psi^a\left(\overline{R}^{\kappa_{T,Q}\,\dagger}_{\bm{\varepsilon}}\right)$ of interest is called semiparametric efficiency bound. 
Let $ \overline{\mathrm{lin}}\dot{\mathcal{R}}^{a\,\dagger}$ be the closure of the tangent space. 
Then, $\psi^a\left(\overline{R}^{\kappa_{T,Q}\,\dagger}_{\bm{\varepsilon}}\right) = \mu^*_0 - \mu^a_0 + \varepsilon^a$ is pathwise differentiable relative to the tangent space $\dot{\mathcal{R}}^{a\,\dagger}$ if and only if there exists a function $\widetilde{\psi}^a\in\overline{\mathrm{lin}}\dot{\mathcal{R}}^{a\,\dagger}$ such that
\begin{align*}
    &\frac{\partial}{\partial \varepsilon^a}\Big|_{\varepsilon^a=0} \psi^a\left(\overline{R}^{\kappa_{T,Q}\,\dagger}_{\bm{\varepsilon}}\right) \left(= \frac{\partial}{\partial \varepsilon^a}\Big|_{\varepsilon^a=0}\Big\{\mu^*_0 - \mu^a_0 + \varepsilon^a\Big\} = 1\right) = \mathbb{E}_{\overline{R}^{\kappa_{T,Q}\,\dagger}_{\bm{\varepsilon}}}\left[ \widetilde{\psi}^a(Y_t, A_t, X_t) S^{a\,\dagger}(Y_t, A_t, X_t) \right].
\end{align*}
This function $\widetilde{\psi}^a$ is called the \emph{semiparametric influence function}. 

Then, we prove the following lemma on the lower bound for $\mathbb{E}_{P_0}\left[ \left(S^a(Y_t, A_t, X_t)\right)^2 \right]$, which is called the semiparametric efficiency bound:
\begin{lemma}
\label{lem:upperbound_semipara}
Any score function $S^{a\,\dagger}\in\dot{\mathcal{R}}^{a\,\dagger}$ satisfies
\begin{align*}
    \mathbb{E}_{P_0}\left[ \left(S^{a\,\dagger}(Y_t, A_t, X_t)\right)^2 \right]\geq \frac{1}{\mathbb{E}_{P_0}\left[\left(\widetilde{\psi}^a(Y_t, A_t, X_t)\right)^2\right]}.
\end{align*}
\end{lemma}
\begin{proof}
From the Cauchy-Schwartz inequality, we have
\begin{align*}
    &1 = \mathbb{E}_{P_0}\left[ \widetilde{\psi}^a(Y_t, A_t, X_t) S^{a\,\dagger}(Y_t, A_t, X_t) \right]\leq \sqrt{\mathbb{E}_{P_0}\left[\left (\widetilde{\psi}^a(Y_t, A_t, X_t)\right)^2 \right]}\sqrt{\mathbb{E}_{P_0}\left[  \left(S^{a\,\dagger}(Y_t, A_t, X_t)\right)^2 \right]}.
\end{align*}
Therefore, 
\begin{align*}
    &\sup_{ S^{a\,\dagger} \in\dot{\mathcal{R}}^{a\,\dagger}} \frac{1}{\mathbb{E}_{P_0}\left[ \left(S^{a\,\dagger}(Y_t, A_t, X_t)\right)^2 \right]} \leq \mathbb{E}_{P_0}\left[\left(\widetilde{\psi}^a(Y_t, A_t, X_t)\right)^2\right].
\end{align*}
\end{proof}

For $a\in [K]\backslash\{a^*_0\}$ and $d\in[K]\backslash\{a^*_0, a\}$, let us define a \emph{semiparametric efficient score function} $S^{a}_{\mathrm{eff}}(y, d, x) \in \overline{\mathrm{lin}} \dot{\mathcal{R}}^{a\,\dagger}$ as
\begin{align*}
    &S^{a}_{\mathrm{eff}}(y, d, x) = \frac{\widetilde{\psi}^a(y, d, x)}{\mathbb{E}_{P_0}\left[\left(\widetilde{\psi}^a(Y_t, A_t, X_t)\right)^2\right]}.
\end{align*}

Next, we consider finding $\widetilde{\psi}^a \in \overline{\mathrm{lin}} \dot{\mathcal{R}}^{a\,\dagger}$. We can use the result of \citet{hahn1998role}. Let us guess that for each $a\in[K]\backslash\{a^*_0\}$ and $d\in\{a^*_0, a\}$, $\widetilde{\psi}^a(y, d, x)$ is given as follows:
\begin{align}
\label{eq:guess}
    \widetilde{\psi}^{a}(y, d, x) = \frac{\mathbbm{1}[d = a](\phi^*_\tau(y, x) - \mu^*_0(x))}{\kappa_{T, Q}(a^*_0| X)} - \frac{\mathbbm{1}[d = a](\phi^a_\tau(y, x) - \mu^a_0(x))}{\kappa_{T, Q}(a| X)} + \mu^*_0(x) - \mu^a_0(x)  - \left(\mu^*_0 - \mu^a_0\right).
\end{align}
Then, as shown by \citet{hahn1998role}, the condition $1 = \mathbb{E}_{\overline{R}^{\kappa_{T,Q}\,\dagger}_{\bm{\varepsilon}}}\left[ \widetilde{\psi}^a(Y_t, A_t, X_t) S^a(Y_t, A_t, X_t) \right]$ holds under \eqref{eq:guess} when for each $a\in[K]\backslash\{a^*_0\}$ and $d\in\{a^*_0, a\}$, the semiparametric efficient score functions are given as
\begin{align*}
    &S^{a}_{\mathrm{eff}}(y, d, x) = \mathbbm{1}[d = a^*_0]S^{a, a^*_0}_{f, \mathrm{eff}}(y|x) + \mathbbm{1}[d = a]S^{a, a}_{f, \mathrm{eff}}(y|x) + S^{a}_{\zeta, \mathrm{eff}}(x),\\
    &S^{a, a^*_0}_{f, \mathrm{eff}}(y|x) = \frac{(\phi^*_\tau(y, x) - \mu^*_0(x))}{\kappa_{T, Q}(a^*_0| X)}/\widetilde{V}^a_0(\kappa_{T, Q}; \tau),\\
    &S^{a, a}_{f, \mathrm{eff}}(y|x) = \frac{(\phi^a_\tau(y, x) - \mu^a_0(x))}{\kappa_{T, Q}(a| X)}/\widetilde{V}^a_{0}(\kappa_{T, Q}; \tau),\\
    &S^{a}_{\zeta, \mathrm{eff}}(x) = \left(\mu^*_0(x) - \mu^a(x) - \big(\mu^*_0 - \mu^a_0\big)\right)/\widetilde{V}^a_{0}(\kappa_{T, Q}; \tau),
\end{align*}
where 
\begin{align*}
    &\widetilde{V}^a_{0}(\kappa_{T, Q}; \tau) = \mathbb{E}_{P_0}\left[\frac{\left(\sigma^*_0(X_t; \tau)\right)^2}{\kappa_{T, Q}(a^*_0| X_t)} + \frac{\left(\sigma^a_0(X_t; \tau)\right)^2}{\kappa_{T, Q}(a| X_t)} + \left(\big(\mu^*_0(X_t) - \mu^a_0(X_t)\big) - \big(\mu^*_0 - \mu^a_0\big)\right)^2\right],\\
    &\left(\sigma^*_0(X_t; \tau)\right)^2 := \mathbb{E}_{P_0}\left[\left(\phi^*_\tau(Y_t, X_t; \tau) - \mu^*_0(X_t)\right)^2| X_t\right],\\
    &\left(\sigma^a_0(X_t; \tau)\right)^2 := \mathbb{E}_{P_0}\left[\left(\phi^a_\tau(Y_t, X_t; \tau) - \mu^a_0(X_t)\right)^2| X_t\right].
\end{align*}
Here, note that for each $d\in[K]$,
\begin{align*}
    &\mathbb{E}_{P_0}\left[\left(\phi^d_\tau(Y_t, X_t) - \mu^d_0(X_t)\right)^2\right]\\
    &= \mathbb{E}_{P_0}\left[\left(Y^d_t\mathbbm{1}[|Y^d_t| < \tau] - \mathbb{E}_{P_0}[Y^d_t\mathbbm{1}[|Y^d_t| < \tau]|X_t]|X_t] \right)^2\right]\\
    &= \mathbb{E}_{P_0}\left[\mathbb{E}_{P_0}\left[\left(Y^d_t\right)^2\mathbbm{1}[|Y^d_t| < \tau]|X_t\right]- \left(\mathbb{E}_{P_0}[Y^d_t\mathbbm{1}[|Y^d_t| < \tau]|X_t]\right)^2\right].
\end{align*}
We also note that $\mathbb{E}_{\overline{R}^{\kappa_{T,Q}\,\dagger}_{\bm{\varepsilon}}}\left[S^a_{\mathrm{eff}}(Y_t, A_t, X_t)\right] = 0$ and  \[\mathbb{E}_{\overline{R}^{\kappa_{T,Q}\,\dagger}_{\bm{\varepsilon}}}\left[\Big(S^a_{\mathrm{eff}}(Y_t, A_t, X_t)\Big)^2\right] = \widetilde{V}^a_{0}(\kappa_{T, Q}; \tau) =  \left(\mathbb{E}_{\overline{R}^{\kappa_{T,Q}\,\dagger}_{\bm{\varepsilon}}}\left[\Big(\widetilde{\psi}^a(Y_t, A_t, X_t)\Big)^2\right]\right)^{-1}.\] 

Summarizing the above arguments, we obtain the following lemma.
\begin{lemma} 
\label{lem:semipara_efficient}
For $a\in[K]\backslash\{a^*_0\}$ and $d\in[K]\backslash\{a^*_0, a\}$, the semiparametric efficient influence function is
\begin{align*}
    \widetilde{\psi}^a(y, d, x)
    &= \frac{\mathbbm{1}[d = a^*_0](\phi^*_\tau(y, x) - \mu^{*}_0(x))}{\kappa_{T, Q}(a^*_0| X)} - \frac{\mathbbm{1}[d = a](\phi^a_\tau(y, x) - \mu^a_0(x))}{\kappa_{T, Q}(a| X)} + \mu^*_0(x) - \mu^a_0(x)  - \left(\mu^*_0 - \mu^a_0\right).
\end{align*}
\end{lemma}

We also define the limit of the semiparametric efficient influence function when $\tau \to \infty$ and the variance as
\begin{align*}
    \widetilde{\psi}^a_{\infty}(y, d, x) &=  \frac{\mathbbm{1}[d = a^*_0](Y^*_t - \mu^{*}_0(x))}{\kappa_{T, Q}(a^*_0| X)} - \frac{\mathbbm{1}[d = a](Y^a_t - \mu^a_0(x))}{\kappa_{T, Q}(a| X)} + \mu^*_0(x) - \mu^a_0(x)  - \left(\mu^*_0 - \mu^a_0\right),\\
    \widetilde{V}^a_{0}(\kappa_{T, Q})  &= \mathbb{E}_{P_0}\left[\left(\widetilde{\psi}^a_{\infty}(Y_t, A_t, X_t)\right)^2\right]\\
    &=  \mathbb{E}_{P_0}\left[\frac{\left(\sigma^*_0(X_t)\right)^2}{\kappa_{T, Q}(a^*_0| X_t)} + \frac{\left(\sigma^a_0(X_t)\right)^2}{\kappa_{T, Q}(a| X_t)} + \left(\big(\mu^*_0(X_t) - \mu^a_0(X_t)\big) - \big(\mu^*_0 - \mu^a_0\big)\right)^2\right] \\
    &\geq \Omega^{a}_0(\kappa_{T, Q}),
\end{align*}
where $C > 0$ is a constant. 

\subsection{Specification of the Observed-Data Score Function}
\label{sec:specification-score}
According to Lemma~\ref{lem:upperbound_semipara}, we can conjecture that if we use the semiparametric efficient score function for our score function, we can obtain a tight upper bound for $\mathbb{E}_{P_0}[L^{a}_T]$, which is related to a lower bound for the probability of misidentification. Note that the variance of the semiparametric efficient score function is equivalent to the lower bound in Lemma~\ref{lem:upperbound_semipara}. However, we cannot use the semiparametric efficient score function because it is derived for $\overline{R}^{\kappa_{T, Q}\,\dagger}_{\bm{\varepsilon}}$, rather than $\overline{R}^{\kappa_{T, Q}}_{\bm{\varepsilon}}$.
Furthermore, if we use the semiparametric efficient score function for our score function, the constant \eqref{eq:const_ate} is not satisfied. Therefore, based on our obtained result, we specify our score function, which differs from the semiparametric efficient score function, but they match when $\tau\to \infty$. 

We specify our score function $S^a(y, d, x) = \mathbbm{1}[d = a^*_0]S^{a,  a^*_0}_{f}(y|x) + \mathbbm{1}[d = a]S^{a, a}_{f}(y|x) + S^a_{\zeta}(x)$ as follows:
\begin{align*}
    &S^{a, a^*_0}_{f}(y|x) = \frac{(\phi^*_\tau(y, x) - \mu^*_0(x))}{\kappa_{T, Q}(a^*_0| X)}/V^a_0(\kappa_{T, Q}; \tau) =  S^{a, a^*_)}_{f, \mathrm{eff}}(y|x){\widetilde{V}^a_{0}(\kappa_{T, Q}; \tau)}/{V^a_{0}(\kappa_{T, Q}; \tau)},\\
    &S^{a, a}_{f}(y|x) = \frac{(\phi^a_\tau(y, x) - \mu^a_0(x))}{\kappa_{T, Q}(a| X)}/V^a_{0}(\kappa_{T, Q}; \tau) =  S^{a, a}_{f, \mathrm{eff}}(y|x){\widetilde{V}^a_{0}(\kappa_{T, Q}; \tau)}/{V^a_{0}(\kappa_{T, Q}; \tau)},\\
    &S^{a}_{\zeta}(x) = \big(\mu^*_0(x) - \mu^a_0(x) - \big(\mu^*_0 - \mu^a_0\big)\big)/V^a_{0}(\kappa_{T, Q}; \tau) = S^{a}_{\zeta, \mathrm{eff}}(x){\widetilde{V}^a_{0}(\kappa_{T, Q}; \tau)}/{V^a_{0}(\kappa_{T, Q}; \tau)},
\end{align*}
where
\begin{align}
     &V^a_0(\kappa_{T, Q}; \tau) = \widetilde{V}^a_{0}(\kappa_{T, Q}; \tau) + \sum_{d\in\{a^*_0, a\}}\mathbb{E}_{P_0}\left[\frac{ \mu^d_0(X_t)\mathbb{E}_{P_0}[Y^d_t\mathbbm{1}[|Y^d_t| < \tau]|X_t] - \left(\mathbb{E}_{P_0}[Y^d_t\mathbbm{1}[|Y^d_t| < \tau]|X_t]\right)^2}{\kappa_{T, Q}(d| X_t)}\right]\nonumber\\
    \label{eq:variance_form}
    &=\mathbb{E}_{P_0}\left[\frac{Y^*_t\left(\phi^*_\tau(Y_t, X_t) - \mu^*_0(X_t)\right)}{\kappa_{T, Q}(a^*_0| X_t)} + \frac{Y^a_t\left(\phi^a_\tau(Y_t, X_t) - \mu^a_0(X_t)\right)}{\kappa_{T, Q}(a| X_t)} + \left(\big(\mu^*_0(X_t) - \mu^a_0(X_t)\big) - \big(\mu^*_0 -  \mu^a_0\big)\right)^2\right].
\end{align}
Here, note that for $d\in[K]$,
\begin{align*}
    \mathbb{E}_{P_0}\left[Y^d_t\left(\phi^d_\tau(Y_t, X_t) - \mu^d_0(X_t)\right)\right]&= \mathbb{E}_{P_0}\left[\left(\left(Y^d_t\right)^2\mathbbm{1}[|Y^d_t| < \tau] - Y^d_t\mathbb{E}_{P_0}[Y^d_t\mathbbm{1}[|Y^d_t| < \tau]|X_t]|X_t] \right)\right]\nonumber\\
    &= \mathbb{E}_{P_0}\left[\mathbb{E}_{P_0}\left[\left(Y^d_t\right)^2\mathbbm{1}[|Y^d_t| < \tau]|X_t\right]- \mu^d_0(X_t)\mathbb{E}_{P_0}[Y^d_t\mathbbm{1}[|Y^d_t| < \tau]|X_t]\right].
\end{align*}
We note that $V^a_0(\kappa_{T, Q}; \tau) \to \widetilde{V}^a_{0}(\kappa_{T, Q})$ as $\varepsilon^a \to 0$ and $\tau \to \infty$,.

From the definition of the parametric submodel, we have
\begin{align*}
    &g^a(\phi^*_\tau(y, x), \phi^a_\tau(y, x), x) = S^{a, a^*_0}_{f}(y|x) + S^{a, a}_{f}(y|x) + S^{a}_{\zeta}(x)\\
    &= \left\{\frac{(\phi^*_\tau(y, x) - \mu^*_0(x))}{\kappa_{T, Q}(a^*_0| X)} - \frac{(\phi^a_\tau(y, x) - \mu^a_0(x))}{\kappa_{T, Q}(a| X)} + \big(\mu^*_0(x) - \mu^a_0(x) - \big(\mu^*_0 - \mu^a_0\big)\big)\right\}/V^a_0(\kappa_{T, Q}; \tau).
\end{align*}
Then, we can also confirm that condition \eqref{eq:const_ate} holds for our specified $g^a$:
\begin{align*}
    &\int \int \left(y^* - y^a\right) \left(1 + \varepsilon^a g^a(\phi^*_\tau(y, x), \phi^a_\tau(y, x), x)\right)p_{P_0}(a^*_0, a, x) \mathrm{d}y^*\mathrm{d}y^a\mathrm{d}x\\
    &=\mu^*_0 - \mu^a_0 + \varepsilon^a\left\{\int \int \left(y^* - y^a\right)  g^a(\phi^*_\tau(y, x), \phi^a_\tau(y, x), x)\overline{r}_0(y, a^*_0, x) \mathrm{d}y^*\mathrm{d}y^a\mathrm{d}x\right\}\\
    &=  \mu^*_0 - \mu^a_0 + \varepsilon^a,
\end{align*}
where we used the definition of the variance \eqref{eq:variance_form}.

In summary, from Lemmas~\ref{lem;taylor_exp_semipara}, under our specified score function, we obtain the following lemma:
\begin{lemma}
\label{lem:expansion}
Suppose that Assumption~\ref{asm:bounded_mean_variance} holds. For $P_0\in\mathcal{P}$ and $Q\in\Alt(P_0)$, as $\varepsilon^a \to 0$,
\begin{align*}
    \frac{1}{T}\mathbb{E}_{\overline{R}_{0}}[L^{a}_T] - \frac{\left(\varepsilon^a\right)^2}{2V^a_{0}(\kappa_{T, Q}; \tau)} = o\left(\left(\varepsilon^a\right)^2\right).
\end{align*}
\end{lemma}

\subsection{Proof of Theorem~\ref{thm:semipara_bandit_lower_bound}}
\label{sec:final_step}
Combining above arguments and refining the lower bound in Lemma~\ref{lem:semipara_bandit_lower_bound}, we prove Theorem~\ref{thm:semipara_bandit_lower_bound}. 
\begin{proof}[Proof of Theorem~\ref{thm:semipara_bandit_lower_bound}]
Let $\Delta^a_0$ be a positive value such that $\mu^* - \mu^a_0 \leq \Delta_0$.
From the inequality in Lemma~\ref{lem:semipara_bandit_lower_bound},
\begin{align*}
    &\limsup_{T\to\infty}-\frac{1}{\Delta^2_0T}\log \mathbb{P}_{ P_0 }(\widehat{a}_T \neq a^*_0)\\
    &\leq \sup_{w \in \mathcal{W}}\min_{a\in[K]\backslash\{a^*_0\}}\inf_{\substack{\varepsilon^a < - \left(\mu^*_0 - \mu^a_0\right)\nonumber\\
    \forall b\in[K]\backslash\{a^*_0, a\}\ \varepsilon^b = 0} }\frac{1}{\Delta^2_0}\sum_{a\in[K]}\mathbb{E}_{\overline{R}_{\bm{\varepsilon}}}\left[\mathbb{E}_{\overline{R}_{\bm{\varepsilon}}}\left[\log \frac{f^a_{\bm{\varepsilon}}(Y^a_t| X_t)\zeta_{\bm{\varepsilon}}(X)}{f^a_{P_0}(Y^a_t| X)\zeta_{P_0}(X)}|X_t\right]w(a| X_t)\right]\nonumber\\
    &\leq \sup_{w \in \mathcal{W}}\min_{a\in[K]\backslash\{a^*_0\}}\inf_{\varepsilon^a < - \left(\mu^{*}_0 - \mu^a_0\right)}\frac{\left(\varepsilon^a\right)^2}{\Delta^2_0}\left\{\frac{1}{2}\mathbb{E}_{P_0}\left[ \left(S^a(Y_t, A_t, X_t)\right)^2\right]  + o\left(1\right)\right\}\nonumber\\
    &\leq  \sup_{w \in \mathcal{W}}\min_{a\in[K]\backslash\{a^*_0\}}\frac{\left(\mu^*_0 - \mu^a_0\right)^2}{\Delta^2_0}\left\{\frac{1}{2V^a_{0}(w; \tau)} + o\left(1\right)\right\}\nonumber\\
    &\leq  \sup_{w \in \mathcal{W}}\min_{a\in[K]\backslash\{a^*_0\}}\left\{\frac{1}{2V^a_{0}(w; \tau)} + o\left(1\right)\right\}\nonumber.
\end{align*}
Here, for $\inf_{\varepsilon^a < - \left(\mu^*_0 - \mu^a_0\right)}\frac{\left(\varepsilon^a\right)^2}{2}\mathbb{E}_{P_0}\left[ \left(S^a(Y_t, A_t, X_t)\right)^2\right] $, we set $\varepsilon^a = - \left(\mu^*_0 - \mu^a_0\right)$, which indicates a situation where $\mu^*(Q) - \mu^a(Q)$ is sufficiently close to $0$. Then, after $\mu^*_0 - \mu^a_0 \to 0$, by letting $\tau\to\infty$, we obtain $V^a_{0}(w; \tau) \to \widetilde{V}^a_{0}(\kappa_{T, Q})$, which is the semiparametric efficiency bound in Lemmas~\ref{lem:upperbound_semipara} and \ref{lem:semipara_efficient}. From $\widetilde{V}^a_{0}(\kappa_{T, Q}) \geq \Omega^{a}_0(w)$, the proof is complete.
\end{proof}

\section{Proof of Lemma~\ref{lem:kauf_lemma_extnd_infinite}}
\label{appdx:proof:lem_extnd_infinite}
\begin{proof}
\begin{align*}
    &\mathbb{E}_{Q}[L_T]  = \sum^T_{t=1}\mathbb{E}_{Q}\left[\sum_{a \in [K]} \mathbbm{1}\{A_t = a\} \log \frac{f^a_{Q}(Y^a_{t}| X_t)\zeta_{Q}(X_t)}{f^a_{P_0}(Y^a_{t}| X_t) \zeta_{P_0}(X_t)}\right]
    \\
    & = \sum^T_{t=1}\mathbb{E}^{X_t, \mathcal{F}_{t-1}}_{Q}\left[\sum_{a \in [K]}\mathbb{E}^{Y^a_t, A_t}_{Q}\left[ \mathbbm{1}[A_t = a] \log \frac{f^a_{Q}(Y^a_{t}| X_t)\zeta_{Q}(X_t)}{f^a_{P_0}(Y^a_{t}| X_t)\zeta_{P_0}(X_t)}|X_t, \mathcal{F}_{t-1}\right]\right]\\
    & = \sum^T_{t=1}\mathbb{E}^{X_t, \mathcal{F}_{t-1}}_{Q}\left[\sum_{a \in [K]}\mathbb{E}_{Q}\left[ \mathbbm{1}[A_t = a]| X_t, \mathcal{F}_{t-1}\right] \mathbb{E}^{Y^a_t}_{Q}\left[\log \frac{f^a_{Q}(Y^a_{t}| X_t)\zeta_{Q}(X_t)}{f^a_{P_0}(Y^a_{t}| X_t)\zeta_{P_0}(X_t)}| X_t, \mathcal{F}_{t-1}\right]\right]\\
    & = \sum^T_{t=1}\mathbb{E}^{X_t}_{Q}\left[\mathbb{E}^{\mathcal{F}_{t}}_{Q}\left[\sum_{a \in [K]}\mathbb{E}_{Q}\left[ \mathbbm{1}[A_t = a]|X_t,  \mathcal{F}_{t-1}\right]\mathbb{E}^{Y^a_t}_{Q}\left[\log \frac{f^a_{Q}(Y^a_t| X_t)\zeta_{Q}(X_t)}{f^a_{P_0}(Y^a_t| X_t)\zeta_{P_0}(X_t)}|X_t\right]\right]\right]\\
    & = \sum^T_{t=1}\int\left(\sum_{a \in [K]}\mathbb{E}^{\mathcal{F}_{t}}_{Q}\left[\mathbb{E}_{Q}\left[ \mathbbm{1}[A_t = a]|X_t = x,  \mathcal{F}_{t-1}\right]\right]\mathbb{E}^{Y^a_t}_{Q}\left[\log \frac{f^a_{Q}(Y^a_t| X_t)\zeta_{Q}(X_t)}{f^a_{P_0}(Y^a_t| X_t)\zeta_{P_0}(X_t)}|X_t = x\right]\right)\zeta_{Q}(x)\mathrm{d}x\\
    & = \int\sum_{a \in [K]}\left(\mathbb{E}^{Y^a}_{Q}\left[\log \frac{f^a_{Q}(Y^a| X)\zeta_{Q}(X)}{f^a_{P_0}(Y^a| X)\zeta_{P_0}(X)}|X = x\right]\sum^T_{t=1}\mathbb{E}^{\mathcal{F}_{t}}_{Q}\left[\mathbb{E}_{Q}\left[ \mathbbm{1}[A_t = a]|X_t = x,  \mathcal{F}_{t-1}\right]\right]\right)\zeta_{Q}(x)\mathrm{d}x\\
    & = \mathbb{E}^{X}_{Q}\left[\sum_{a \in [K]}\mathbb{E}^{Y^a}_{Q}\left[\log \frac{f^a_{Q}(Y^a| X)\zeta_{Q}(X)}{f^a_{P_0}(Y^a| X)\zeta_{P_0}(X)}|X\right]\sum^T_{t=1}\mathbb{E}^{\mathcal{F}_{t-1}}_{Q}\left[\mathbb{E}_{Q}\left[ \mathbbm{1}[A_t = a]| X, \mathcal{F}_{t-1}\right]\right]\right],
\end{align*}
where $\mathbb{E}^{Z}_{Q}$ denotes an expectation of random variable $Z$ over the distribution $Q$. We used that the observations $(Y^1_t, \dots, Y^K_t, X_t)$ are i.i.d. across $t\in\{1,2,\dots, T\}$. 
\end{proof}

\section{Proof of Lemma~\ref{lem:trans}}
\label{appdx:proof_lem:trans}
\begin{proof}
For the parametric submodel of the observed-data bandit models, the log-likelihood for the observed data is
\begin{align}
    &\log \overline{r}^{\kappa}_{\bm{\varepsilon}}(y, d, x)= \log \int \kappa(d|x)p_{{\bm{\varepsilon}}}(H^d(\tau^d, v^d)) J(\tau^d, v^d) \mathrm{d}v^d,
\end{align}
where note that $p_{{\bm{\varepsilon}}}(H^d(\tau^d, v^d)) = p_{{\bm{\varepsilon}}}(y^*, y^a, x)$. 
Then, for $d\in\{a^*_0, a\}$, 
\begin{align}
\label{eq:log_like_comp}
    S^a(y, d, x) &= \frac{\partial}{\partial \varepsilon^a} \left[\log \int \kappa(d|x)p_{{\bm{\varepsilon}}}(H^d(\tau^d, v^d)) J(\tau^d, v^d) \mathrm{d}v^d\right]\Bigg|_{\varepsilon^a = 0}\nonumber\\
    &= \frac{\int \frac{\partial}{\partial \varepsilon^a} \kappa(d|x)p_{{\bm{\varepsilon}}}(H^d(\tau^d, v^d)) J(\tau^d, v^d) \mathrm{d}v^d}{\int \kappa(d|x)p_{{\bm{\varepsilon}}}(y^*, y^a, x)p_{{\bm{\varepsilon}}}(H^d(\tau^d, v^d)) J(\tau^d, v^d) \mathrm{d}v^d}.
\end{align}
Dividing and multiplying by $p_{\bm{\varepsilon}}(H^d(\tau^d, v^d)) J(\tau^d, v^d) \mathrm{d}v^d$ in the integral of the numerator of \eqref{eq:log_like_comp} yields
\begin{align*}
&\frac{\int \frac{\partial}{\partial \varepsilon^a} \kappa(d|x)p_{\bm{\varepsilon}: \varepsilon^a = 0}(H^d(\tau^d, v^d)) J(\tau^d, v^d) \mathrm{d}v^d}{\int \kappa(d|x)p_{\bm{\varepsilon}: \varepsilon^a = 0}(H^d(\tau^d, v^d)) J(\tau^d, v^d) \mathrm{d}v^d}\\
&=\frac{\int g^a(\phi^*_\tau(y, x), \phi^a_\tau(y, x), x) p_{\bm{\varepsilon}: \varepsilon^a = 0}(H^d(\tau^d, v^d)) J(\tau^d, v^d) \mathrm{d}v^d}{\int \kappa(d|x)p_{\bm{\varepsilon}: \varepsilon^a = 0}(H^d(\tau^d, v^d)) J(\tau^d, v^d) \mathrm{d}v^d}\\
&=\frac{\int g^a(\phi^*_\tau(y, x), \phi^a_\tau(y, x), x) p_{P_0}(H^d(\tau^d, v^d)) J(\tau^d, v^d) \mathrm{d}v^d}{\int \kappa(d|x)p_{P_0}(H^d(\tau^d, v^d)) J(\tau^d, v^d) \mathrm{d}v^d}.
\end{align*}
Hence, 
\begin{align*}
    S^a(y, d, x) = \mathbb{E}_{P_0}\left[g^a(\phi^*_\tau(Y^*_t, X_t), \phi^a_\tau(Y^a_t, X_t), X_t) | \mathcal{T}^d(Y^*_t, Y^a_t, X_t) = (y, x)\right]
\end{align*}
This concludes the proof. 
\end{proof}

\section{Proof of Theorem~\ref{thm:semipara_bandit_lower_bound}}
\label{appdx:lower}
\begin{proof}
From Theorem~\ref{thm:semipara_bandit_lower_bound}, if the exists $\Delta_0 > 0$ such that $\mu^*_0 - \mu^a_0$ for all $a\in[K]$, the lower bounds are characterized by
\begin{align*}
\max_{w \in \mathcal{W}}\min_{a\neq a^*_0} \frac{1}{2\mathbb{E}_{P_0}\left[\frac{\left(\sigma^*_0(X)\right)^2}{w(a^*_0| X)} + \frac{\left(\sigma^a_0(X)\right)^2}{w(a| X)} \right]}.
\end{align*}
Solving this problem is equivalent to solve
\begin{align*}
\min_{w \in \mathcal{W}}\min_{a\neq a^*_0} \mathbb{E}_{P_0}\left[\frac{\left(\sigma^*_0(X)\right)^2}{w(a^*_0| X)} + \frac{\left(\sigma^a_0(X)\right)^2}{w(a| X)} \right].
\end{align*}
To solve this problem, it is enough to consider the point-wise optimization problem for each $x\in\mathcal{X}$ given as follows:
\begin{align*}
\min_{w \in \mathcal{W}}\min_{a\neq a^*_0} \frac{\left(\sigma^*_0(x)\right)^2}{w(a^*_0| x)} + \frac{\left(\sigma^a_0(x)\right)^2}{w(a| x)}.
\end{align*}
From the definition of asymptotically invariant strategies, it should not depend on $P_0$. Therefore, we consider the following non-linear programming: 
\begin{align*}
    &\min_{R\in\mathbb{R}, w\in\mathcal{W}}\ \ \ R\\
    \mathrm{s.t.}&\ \ \ R \geq \frac{\left(\sigma^a_0(x)\right)^2}{w(a^a_0|x)} + \frac{\left(\sigma^b_0(x)\right)^2}{w(b|x)} \qquad \forall b \in [K]\backslash\{a\}\quad \forall a \in [K],\\
    &\ \ \ \sum_{c\in[K]}w(c|x) = 1,\\
    &\ \ \ w(c|x) \geq 0 \qquad \forall c\in [K].
\end{align*}
For this problem, we derive the first-order condition, which is sufficient for the global optimality of such a convex programming problem. For Lagrangian multipliers $\lambda^{a, b} \geq 0$ and $\gamma \in\mathbb{R}$, we consider the following Lagrangian function:
\begin{align*}
    L(\lambda) = R + \sum_{a\in [K]} \sum_{b\in [K]\backslash \{a\}} \lambda^{a, b}\left\{\frac{\left(\sigma^a_0(x)\right)^2}{w(a|x)} + \frac{\left(\sigma^b_0(x)\right)^2}{w(b|x)} - R\right\} - \gamma\left\{\sum_{c\in[K]}w(c|x) - 1\right\}
\end{align*}
Then, the optimal solutions $w^*\in\mathcal{W}$, $\lambda^{* a,b}$,  $\gamma^*$, and $R^*$ satisfies
\begin{align}
\label{eq:cond1}
&1 -  \sum_{a\in[K]}\sum_{b\in[K]\backslash\{a\}}\lambda^{*a, b} = 0\\
\label{eq:cond2}
&-2\sum_{b\in [K]\backslash\{a\}}\lambda^{*a,b}\frac{\left(\sigma^b_0(x)\right)^2}{(w^*(b|x))^2}  = \gamma^*\qquad \forall a \in [K]\\
\label{eq:cond3}
&\lambda^{* a, b} \left\{\frac{\left(\sigma^a_0(x)\right)^2}{w^*(a|x)} + \frac{\left(\sigma^b_0(x)\right)^2}{w^*(b|x)} - R^*\right\} = 0\\
&\gamma^* \left\{\sum_{c\in[K]}w^*(c|x) - 1\right\} = 0\qquad \forall a \in [K].\nonumber
\end{align}
Here, \eqref{eq:cond1} implies $\lambda^{*a,b} > 0$ for some $a, b \in [K]\times [K]$ $a\neq b$. With $\lambda^{*a, b} > 0$, since $- \frac{\left(\sigma^a_0(x)\right)^2}{(w^*(a|x))^2} < 0$ for all $a\in[K]$, it follows that $\gamma^* < 0$. This also implies that $\lambda^{*a,b} > 0$ for each $a, b \in [K]\times [K]$ $a\neq b$ from \eqref{eq:cond2}. Then, \eqref{eq:cond3} implies that for all $a \in [K]$,
\begin{align*}
    \frac{\left(\sigma^a_0(x)\right)^2}{w^*(a|x)} + \frac{\left(\sigma^b_0(x)\right)^2}{w^*(b|x)} = R^* \ \ \forall b \in [K]\backslash\{a\}.
\end{align*}
This implies that for each $a\in[K]$ and each $b, c \in ([K]\backslash\{a\}) \times ([K]\backslash\{a\})$,
\begin{align*}
    \frac{\left(\sigma^b_0(x)\right)^2}{w^*(b|x)} = \frac{\left(\sigma^c_0(x)\right)^2}{w^*(c|x)}.
\end{align*}
Finally, we solve
\begin{align*}
    &\min_{w\in\mathcal{W}} \frac{\left(\sigma^1_0(x)\right)^2}{w(1|x)} + \frac{\left(\sigma^2_0(x)\right)^2}{w(2|x)}\qquad \mathrm{s.t.}\quad \frac{\left(\sigma^a_0(x)\right)^2}{w^*(a|x)} = \frac{\left(\sigma^b_0(x)\right)^2}{w^*(b|x)}\quad \forall (a, b)\in[K]^2.
\end{align*}
When $K = 2$, by solving this, we have $w^*(1|x) = \frac{\sigma^1(x)}{\sigma^1(x) + \sigma^2(x)}$ and $w^*(2|x) = \frac{\sigma^2(x)}{\sigma^1(x) + \sigma^2(x)}$ for any $x\in\mathcal{X}$. 
When $K \geq 3$, by solving this, we have $w^*(a|x) = \frac{\left(\sigma^a(x)\right)^2}{\sum_{b\in[K]}\left(\sigma^b(x)\right)^2}$ for all $a\in[K]$. 
\end{proof}

\section{\texorpdfstring{$(\xi^a_t, \mathcal{F}_t)$}{TEXT} is martingale difference sequences}
\label{appdx:martingale}
\begin{proof}
Clearly, $ \mathbb{E}_{P_0}[ |\xi^a_t|] < \infty$. For each $t \in [T]$, 
\begin{align*}
    &\mathbb{E}_{P_0}\left[\xi^a_t| \mathcal{F}_{t-1}\right] = \frac{1}{\sqrt{T } \widetilde{V}^a}\mathbb{E}_{P_0}\left[\varphi^{a^*_0}\Big(Y_t, A_t, X_t; \widehat{\mu}^{a^*_0}_t, \widehat{w}_t\Big) -\varphi^a\Big(Y_t, A_t, X_t; \widehat{\mu}^a_t, \widehat{w}_t\Big) - (\mu^*_0 - \mu^a_0)\big| \mathcal{F}_{t-1}\right]
    \\
    &= \frac{1}{\sqrt{T } \widetilde{V}^a}\mathbb{E}_{P_0}\Bigg[\frac{\mathbb{E}_{P_0}[\mathbbm{1}[A_t = a^*_0|X_t] | \mathcal{F}_{t-1}]\mathbb{E}_{P_0}\left[\big(Y^*_t- \widehat{\mu}^{a^*_0}_t(X_t)\big) | X_t,  \mathcal{F}_{t-1}\right]}{\widehat{w}_t(a^*_0|X_t)}   +\widehat{\mu}^{a^*_0}_t(X_t)
    \\
    & \qquad \qquad \qquad - \frac{\mathbb{E}_{P_0}[\mathbbm{1}[A_t = a|X_t]| \mathcal{F}_{t-1}]\mathbb{E}_{P_0}\left[\big(Y^{a}_t- \widehat{\mu}^a_t(X_t)\big) | X_t, \mathcal{F}_{t-1}\right]}{\widehat{w}_t(a|X_t)}  - \widehat{\mu}^{a}_t(X_t)- (\mu^*_0 - \mu^a_0)\Bigg]
    \\
    & =  0.
\end{align*}
\end{proof}

\section{Proof of Lemma~\ref{lem:condition1}}
\label{appdx:lem:condition1}
\begin{proof} For the simplicity, let us denote $\mathbb{E}_{P_0}$ by $\mathbb{E}$.
Recall that $\varphi^a\Big(Y_t, A_t, X_t; \widehat{\mu}^a_t, \widehat{w}_t\Big)$ is constructed as 
\begin{align*}
    \varphi^a\Big(Y_t, A_t, X_t; \widehat{\mu}^a_t, \widehat{w}_t\Big) = \frac{\mathbbm{1}[A_t = a]\big(Y^{a}_t- \widehat{\mu}^a_t(X_t\big)}{\widehat{w}_{t}(a|X_t)} + \widehat{\mu}^a_t(X_t).
\end{align*}
For each $t = 1, \ldots, T$, we have
\begin{align*}
    & \mathbb{E}\left[\exp\left(C_0 \sqrt{T} |\xi^a_t|\right) \middle| \mathcal{F}_{t-1}\right] 
    \\
    & = \mathbb{E}\left[\exp \left(C_0 \left| \frac{\varphi^{a^*_0}\Big(Y_t, A_t, X_t; \widehat{\mu}^{a^*_0}_t, \widehat{w}_t\Big) - \varphi^a\Big(Y_t, A_t, X_t; \widehat{\mu}^a_t, \widehat{w}_t\Big) -  (\mu^*_0 - \mu^a_0)}{\sqrt{\widetilde{V}^a}}\right| \right) \middle| \mathcal{F}_{t-1}\right] 
    \\
    & \le \mathbb{E}\left[\exp \left( \frac{C_0}{\sqrt{\widetilde{V}^a}} \left|\varphi^{a^*_0}\Big(Y_t, A_t, X_t; \widehat{\mu}^{a^*_0}_t, \widehat{w}_t\Big) - \varphi^a\Big(Y_t, A_t, X_t; \widehat{\mu}^a_t, \widehat{w}_t\Big) \right| + \frac{C_0 (\mu^*_0 - \mu^a_0)}{\sqrt{\widetilde{V}^a}}  \right) \middle| \mathcal{F}_{t-1}\right] 
    \\
    & \le \mathbb{E}\left[\exp \left( \frac{C_0}{\sqrt{\widetilde{V}^a}} \left|\varphi^{a^*_0}\Big(Y_t, A_t, X_t; \widehat{\mu}^{a^*_0}_t, \widehat{w}_t\Big) - \varphi^a\Big(Y_t, A_t, X_t; \widehat{\mu}^a_t, \widehat{w}_t\Big) \right| + \frac{2 C_0 C_\mu}{\sqrt{\widetilde{V}^a}}  \right) \middle| \mathcal{F}_{t-1}\right]
    \\
    & \stackrel{(a)}{=} \widetilde{C}_1 \mathbb{E} \left[\exp \left(\frac{C_0}{\sqrt{\widetilde{V}^a}} |\varphi^{a^*_0}\Big(Y_t, A_t, X_t; \widehat{\mu}^{a^*_0}_t, \widehat{w}_t\Big) - \varphi^a\Big(Y_t, A_t, X_t; \widehat{\mu}^a_t, \widehat{w}_t\Big) | \right) \middle| \mathcal{F}_{t-1}, A_t = a^*_0\right]\cdot \\
& \qquad\qquad \mathbb{P} (A_t = a^*_0 |\mathcal{F}_{t-1}) \\
    & \qquad + \widetilde{C}_1 
\mathbb{E}\left[\exp \left(\frac{C_0}{\sqrt{\widetilde{V}^a}} |\varphi^{a^*_0}\Big(Y_t, A_t, X_t; \widehat{\mu}^{a^*_0}_t, \widehat{w}_t\Big) - \varphi^a\Big(Y_t, A_t, X_t; \widehat{\mu}^a_t, \widehat{w}_t\Big) | \right) \middle| \mathcal{F}_{t-1}, A_t = a\right]\cdot \\
& \qquad\qquad \mathbb{P} (A_t = a |\mathcal{F}_{t-1}) 
\\
& = \widetilde{C}_1 \mathbb{E} \left[\exp \left(\frac{C_0}{\sqrt{\widetilde{V}^a}} \left| \frac{\big(Y^*_t- \widehat{\mu}^{a^*_0}_t(X_t)\big)}{\widehat{w}_t(a^*_0|X_t)} + \widehat{\mu}^{a^*_0}_t(X_t) - \widehat{\mu}^a_t(X_t)\right| \right) \middle| \mathcal{F}_{t-1}, A_t = a^*_0\right] \mathbb{P} (A_t = a^*_0 \; | \mathcal{F}_{t-1})
\\
& \qquad + \widetilde{C}_1\mathbb{E} \left[\exp \left(\frac{C_0}{\sqrt{\widetilde{V}^a}} \left|- \frac{\big(Y^{a}_t- \widehat{\mu}^a_t(X_t)\big)}{\widehat{w}_t(a|X_t)} + \widehat{\mu}^{a^*_0}_t(X_t) - \widehat{\mu}^a_t(X_t)\right| \right) \middle| \mathcal{F}_{t-1}, A_t = a\right] \mathbb{P} (A_t = a | \mathcal{F}_{t-1}),
\end{align*}
where for $(a)$, we denote $\widetilde{C}_1 =  \exp\left(2C_0 C_\mu/\widetilde{V}^a\right)$. Since $ X_{a,t}$ is a sub-exponential random variable (Assumption~\ref{asm:sub_exp}), there exists some universal constant $C>0$ such that for all $ P_0\in \mathcal{P}$, for all $\lambda\ge 0$ such that $0 \le \lambda \le 1/C$, $ \mathbb{E}[\exp(\lambda (X_{a,t} - \mu^a_0))] \le \exp(C^2 \lambda^2)$ (\cite{vershynin2018high}, Proposition 2.7.1).   
Note that from the assumptions that $|{\mu}^a_t| \le C_{ \mu}$, $ \max \{\left(\sigma^a_0\right)^2, 1/\left(\sigma^a_0\right)^2\} \le C_{\sigma^2}$, and $ | \widehat{w}_t(a|X_t)| \ge C_w$ for all $ t \in \{1,\ldots, T\}$), where $C_w > 0$ is a constant that depends on $C_{\sigma^2}$. Therefore, there exists a positive constant $  C_1 ( C_0, C_\mu, C_{\sigma^2})$ such that
\begin{align*}
    \mathbb{E} \left[\exp(C_0 \sqrt{T} |\xi^a_t|) \middle| \mathcal{F}_{t-1}\right]  \le C_1 ( C_0, C_\mu, C_{\sigma^2}).
\end{align*}
This concludes the proof. 

\end{proof}

\section{Proof of Lemma~\ref{lem:condition2}}
\label{appdx:lem:condition2}
Assumptions~\ref{asm:almost_sure_convergence} and the continuity of $w^*$ with respect to $\mathrm{Var}_{P_0}(Y^a|x)$ directly implies the following corollary, which states the almost sure convergence of $\widehat{w}_t(a|X_t)$. 
\begin{lemma}\label{cor:as_conv_w_a}
Under the RS-AIPW strategy, for each $a\in [K]$ and $x\in\mathcal{X}$, 
\begin{align*}
&  \widehat{w}_t(a|x)\xrightarrow{\mathrm{a.s.}} w^*(a|x),
\end{align*}
\end{lemma}
Then, we present the following results on the convergence of the second moment. Recall we defined 
\begin{align*}
    \widetilde{V}^a =  \mathbb{E}_{P_0}\left[\frac{\left(\sigma^*_0(X)\right)^2}{w^*(a^*_0| X)} + \frac{\left(\sigma^a_0(X)\right)^2}{w^*(a| X)} + \left(\mu^*_0(X) - \mu^a_0(X) - (\mu^*_0 - \mu^a_0)\right)^2\right].
\end{align*}

We first show the following lemma.
\begin{lemma}
\label{lem:as_conv}
Suppose that Assumptions~\ref{asm:bounded_mean_variance} and \ref{asm:almost_sure_convergence} hold. Then, with probability one,
\begin{align*}
    \lim_{t\to\infty}t^\alpha \left|\frac{1}{T}\sum^T_{t=1}\mathbb{E}_{P_0}\left[\left(\varphi^{a^*_0}\Big(Y_t, A_t, X_t; \widehat{\mu}^{a^*_0}_t, \widehat{w}_t\Big) - \varphi^a\Big(Y_t, A_t, X_t; \widehat{\mu}^a_t, \widehat{w}_t\Big)- (\mu^*_0 - \mu^a_0)\right)^2\Big| \mathcal{F}_{t-1}\right] - \sqrt{\widetilde{V}^a} \right| = 0.
\end{align*}
\end{lemma}

\begin{proof}
\begin{align*}
&\mathbb{E}_{P_0}\left[\left(\varphi^{a^*_0}\Big(Y_t, A_t, X_t; \widehat{\mu}^{a^*_0}_t, \widehat{w}_t\Big) - \varphi^a\Big(Y_t, A_t, X_t; \widehat{\mu}^a_t, \widehat{w}_t\Big)- (\mu^*_0 - \mu^a_0)\right)^2\Big| \mathcal{F}_{t-1}\right]
\\
&= \mathbb{E}_{P_0}\left[\left(\frac{\mathbbm{1}[A_{t} = a^*_0]\big(Y^*_t- \widehat{\mu}^{a^*_0}_t(X_t)\big)}{\widehat{w}_t(a^*_0|X_t)} - \frac{\mathbbm{1}[A_t = a]\big(Y^a_t- \widehat{\mu}^a_t(X_t)\big)}{\widehat{w}_t(a|X_t)} + \widehat{\mu}^{a^*_0}_t(X_t) - \widehat{\mu}^a_t(X_t) - (\mu^*_0 - \mu^a_0)\right)^2 \Big| \mathcal{F}_{t-1}\right]
\\
&= \mathbb{E}_{P_0}\Bigg[\left(\frac{\mathbbm{1}[A_{t} = a^*_0]\big(Y^*_t- \widehat{\mu}^{a^*_0}_t(X_t)\big)}{\widehat{w}_t(a^*_0|X_t)} - \frac{\mathbbm{1}[A_t = a]\big(Y^a_t- \widehat{\mu}^a_t(X_t)\big)}{\widehat{w}_t(a|X_t)}\right)^2
\\
&\ \ \ \ \ \ + 2\left(\frac{\mathbbm{1}[A_{t} = a^*_0]\big(Y^*_t- \widehat{\mu}^{a^*_0}_t(X_t)\big)}{\widehat{w}_t(a^*_0|X_t)} - \frac{\mathbbm{1}[A_{t} = a]\big(Y^a_t- \widehat{\mu}^a_t(X_t)\big)}{\widehat{w}_t(a|X_t)}\right)\left( \widehat{\mu}^{a^*_0}_t(X_t) - \widehat{\mu}^a_t(X_t) - (\mu^*_0 - \mu^a_0)\right)
\\
&\ \ \ \ \ \ + \left( \widehat{\mu}^{a^*_0}_t(X_t) - \widehat{\mu}^a_t(X_t) - (\mu^*_0 - \mu^a_0)\right)^2 |  \mathcal{F}_{t-1}\Bigg]
\\
&= \mathbb{E}_{P_0}\Bigg[\frac{\mathbbm{1}[A_{t} = a^*_0]\big(Y^*_t- \widehat{\mu}^{a^*_0}_t(X_t)\big)^2}{\widehat{w}_t(a^*_0|X_t)} + \frac{\mathbbm{1}[A_{t} = a]\big(Y^a_t- \widehat{\mu}^a_t(X_t)\big)^2}{\widehat{w}_t(a|X_t)}
\\
&\ \ \ \ \ \ + 2\left(\frac{\mathbbm{1}[A_{t}  = a^*_0]\big(Y^*_t- \widehat{\mu}^{a^*_0}_t(X_t)\big)}{\widehat{w}_t(a^*_0|X_t)} - \frac{\mathbbm{1}[A_t = a]\big(Y^a_t- \widehat{\mu}^a_t(X_t)\big)}{\widehat{w}_t(a|X_t)}\right)\left( \widehat{\mu}^{a^*_0}_t(X_t) - \widehat{\mu}^a_t(X_t) - (\mu^*_0 - \mu^a_0)\right)
\\
&\ \ \ \ \ \ + \left( \widehat{\mu}^{a^*_0}_t(X_t) - \widehat{\mu}^a_t(X_t) - (\mu^*_0 - \mu^a_0)\right)^2 |  \mathcal{F}_{t-1}\Bigg]
\\
&= \mathbb{E}_{P_0}\left[\frac{\big(Y^*_t- \widehat{\mu}^{a^*_0}_t(X_t)\big)^2}{\widehat{w}_t(a^*_0|X_t)}|  \mathcal{F}_{t-1}\right] + \mathbb{E}_{P_0}\left[\frac{\big(Y^a_t- \widehat{\mu}^a_t(X_t)\big)^2}{\widehat{w}_t(a|X_t)}|  \mathcal{F}_{t-1}\right]\\
&\ \ \ \ \ \ -  \mathbb{E}_{P_0}\left[\left(\widehat{\mu}^{a^*_0}_t(X_t) + \widehat{\mu}^a_t(X_t) - (\mu^*_0 - \mu^a_0)\right)^2|  \mathcal{F}_{t-1}\right]. \label{eq:check4}
\end{align*}
Here, we used 
\begin{align*}
    &\mathbb{E}_{P_0}\Bigg[\frac{\mathbbm{1}[A_{t} = a]\big(Y^{a}_t- \widehat{\mu}^a_t(X_t)\big)^2}{(\widehat{w}_t(a|X_t))^2}|  \mathcal{F}_{t-1}\Bigg] = \mathbb{E}_{P_0}\Bigg[\mathbb{E}_{P_0}\Bigg[\frac{\widehat{w}_t(a|X_t)\big(Y^{a}_t- \widehat{\mu}^a_t(X_t)\big)^2}{(\widehat{w}_t(a|X_t))^2}|X_t \mathcal{F}_{t-1}\Bigg]\Bigg]\\
    &= \mathbb{E}_{P_0}\Bigg[\frac{\big(Y^{a}_t- \widehat{\mu}^a_t(X_t)\big)^2}{\widehat{w}_t(a|X_t)}|  \mathcal{F}_{t-1}\Bigg]
\end{align*}
and
\begin{align*}
    &\mathbb{E}_{P_0}\Bigg[ \frac{\mathbbm{1}[A_{t} = a]\big(Y^a_t- \widehat{\mu}^{a}_t(X_t)\big)}{\widehat{w}_t(a|X_t)}\left( \widehat{\mu}^{a^*_0}_t(X_t) - \widehat{\mu}^a_t(X_t) - (\mu^*_0 - \mu^a_0)\right) |  \mathcal{F}_{t-1}\Bigg] 
    \\
    & = \mathbb{E}_{P_0}\left[\left( \widehat{\mu}^{a^*_0}_t(X_t) - \widehat{\mu}^a_t(X_t) - (\mu^*_0 - \mu^a_0)\right)\mathbb{E}_{P_0}\Bigg[\frac{\widehat{w}_t(a|X_t)\big(Y^a_t- \widehat{\mu}^{a}_t(X_t)\big)}{\widehat{w}_t(a|X_t)} |  X_t, \mathcal{F}_{t-1}\Bigg]\mathcal{F}_{t-1}\right].
\end{align*}
We also have 
\begin{align*}
&\mathbb{E}_{P_0}\left[\frac{\big(Y^{a}_t- \widehat{\mu}^a_t(X_t)\big)^2}{\widehat{w}_t(a|X_t)}|X_t,  \mathcal{F}_{t-1}\right]=\frac{\mathbb{E}_{P_0}[(Y^{a}_t)^2|X_t] - 2\mu^a_0(X_t)\widehat{\mu}^a_t(X_t)+ (\widehat{\mu}^a_{ t}(X_t))^2}{\widehat{w}_t(a|X_t)}\\
&=\frac{\mathbb{E}_{P_0}[(Y^{a}_t)^2|X_t] - (\mu^{a}_0(X_t))^2 + (\mu^a_0(X_t) - \widehat{\mu}^a_t(X_t))^2}{\widehat{w}_t(a|X_t)}.
\end{align*}
Then,
\begin{align*}
& \mathbb{E}_{P_0}\left[\frac{\big(Y^*_t- \widehat{\mu}^{a^*_0}_t(X_t)\big)^2}{\widehat{w}_t(a^*_0|X_t)}|  \mathcal{F}_{t-1}\right] + \mathbb{E}_{P_0}\left[\frac{\big(Y^a_t- \widehat{\mu}^a_t(X_t)\big)^2}{\widehat{w}_t(a|X_t)}|  \mathcal{F}_{t-1}\right]\\
&\ \ \ \ \ -  \mathbb{E}_{P_0}\left[\left(\widehat{\mu}^{a^*_0}_t(X_t) + \widehat{\mu}^a_t(X_t) - (\mu^*_0 - \mu^a_0)\right)^2|  \mathcal{F}_{t-1}\right]
\\
&= \mathbb{E}_{P_0}\left[\frac{\mathbb{E}_{P_0}[(Y^*_t)^2|X_t] - (\mu^*_0(X_t))^2 + (\mu^*_0(X_t) - \widehat{\mu}^{a^*_0}_t(X_t))^2}{\widehat{w}_t(a^*_0|X_t)}\right]\\
&\ \ \ \ \ +  \mathbb{E}_{P_0}\left[\frac{\mathbb{E}_{P_0}[(Y^{a}_t)^2|X_t] - (\mu^{a}_0(X_t))^2 + (\mu^a_0(X_t) - \widehat{\mu}^a_t(X_t))^2}{\widehat{w}_t(a|X_t)}\right]\\
&\ \ \ \ \ -  \mathbb{E}_{P_0}\left[\left(\widehat{\mu}^{a^*_0}_t(X_t) + \widehat{\mu}^a_t(X_t) - (\mu^*_0 - \mu^a_0)\right)^2\right].
\end{align*}
Because $\widehat{\mu}^a_t(x)\xrightarrow{\mathrm{a.s.}} \mu^a_0(x)$ and $\widehat{w}_t(a|x)\xrightarrow{\mathrm{a.s.}} w^*(a|x)$, for each $x\in\mathcal{X}$, with probability $1$,
\begin{align*}
&\lim_{t\to\infty}t^\alpha\left|\left(\frac{\mathbb{E}_{P_0}[(Y^*_t)^2|x] - (\mu^*_0(x))^2 + (\mu^*_0(x) - \widehat{\mu}^{a^*_0}_t(x))^2}{\widehat{w}_t(a^*_0|x)}\right)\right.\\
&\ \ \ + \left(\frac{\mathbb{E}_{P_0}[(Y^{a}_t)^2|x] - (\mu^a_0(x))^2 + (\mu^a_0(x) - \widehat{\mu}^a_t(x))^2}{\widehat{w}_t(a|x)}\right) 
\\
& \ \ \  - \left(\widehat{\mu}^{a^*_0}_t(x) + \widehat{\mu}^a_t(x) - (\mu^*_0 - \mu^a_0)\right)^2\\
&\ \ \  \left. -  \left(\frac{\left(\sigma^*_0(x)\right)^2}{w^*(a^*_0| X)} + \frac{\left(\sigma^a_0(X)\right)^2}{w^*(a| X)} + \left(\mu^*_0(x) - \mu^a_0(x) - (\mu^*_0 - \mu^a_0)\right)^2\right)\right|
\\
&\leq t^\alpha\lim_{t\to\infty}\left|\frac{\mathbb{E}_{P_0}[(Y^*_t)^2|x] - (\mu^*_0(x))^2}{\widehat{w}_t(a^*_0|x)} - \frac{\left(\sigma^*_0(x)\right)^2}{w^*(a^*_0| x)}\right| +  t^\alpha\lim_{t\to\infty}\left| \frac{\mathbb{E}_{P_0}[(Y^{a}_t)^2|x] - (\mu^a_0(x))^2}{\widehat{w}_t(a|x)} - \frac{\left(\sigma^a_0(X)\right)^2}{w^*(a| x)}\right|\\
&\ \ \ + t^\alpha\lim_{t\to\infty}\frac{(\mu^*_0(x) - \widehat{\mu}^{a^*_0}_t(x))^2}{\widehat{w}_t(a^*_0|X_t)} + t^\alpha\lim_{t\to\infty}\frac{ (\mu^{a}_0(x) - \widehat{\mu}^{a}_t(x))^2}{\widehat{w}_t(a|X_t)}\\
&\ \ \ + t^\alpha\lim_{t\to\infty}\left|\left(\widehat{\mu}^{a^*_0}_t(x) + \widehat{\mu}^a_t(x) - (\mu^*_0 - \mu^a_0)\right)^2 - \left(\mu^*_0(x) - \mu^a_0(x) - (\mu^*_0 - \mu^a_0)\right)^2\right|\\
&= 0.
\end{align*}
Note that $\mathbb{E}_{P_0}[(Y^{a}_t)^2|x] - (\mu^{a}_0(x))^2 = \left(\sigma^a_0(x)\right)^2$. 
This directly implies the statement.
\end{proof}

By using Lemma~\ref{lem:as_conv}, we prove Lemma~\ref{lem:consistent_second_moment}.
 \begin{lemma}
\label{lem:consistent_second_moment}
 Under the RS-AIPW strategy, for each $a \in [K] \backslash \{a^*_0\}$, with probability~$1$,
\begin{align*}
&\lim_{t\to\infty}\left\{\mathbb{E}_{P_0}\left[\left(\varphi^{a^*_0}\Big(Y_t, A_t, X_t; \widehat{\mu}^{a^*_0}_t, \widehat{w}_t\Big) - \varphi^a\Big(Y_t, A_t, X_t; \widehat{\mu}^a_t, \widehat{w}_t\Big)- (\mu^*_0 - \mu^a_0)\right)^2\Big| \mathcal{F}_{t-1}\right] - \widetilde{V}^a\right\} = 0.
\end{align*}
\end{lemma}

This directly implies the statement.

\begin{proof}
Lemma~\ref{lem:as_conv} implies that
\begin{align*}
& \frac{1}{T}\sum^T_{t=1}\mathbb{E}_{P_0}\left[\left(\varphi^{a^*_0}\Big(Y_t, A_t, X_t; \widehat{\mu}^{a^*_0}_t, \widehat{w}_t\Big) - \varphi^a\Big(Y_t, A_t, X_t; \widehat{\mu}^a_t, \widehat{w}_t\Big)- (\mu^*_0 - \mu^a_0)\right)^2\Big| \mathcal{F}_{t-1}\right] -   \sqrt{\widetilde{V}^a}\\
&\xrightarrow{\mathrm{a.s.}} 0,\\
\Leftrightarrow&\frac{1}{T \sqrt{\widetilde{V}^a}}\sum^T_{t=1}\mathbb{E}_{P_0}\left[\left(\varphi^{a^*_0}\Big(Y_t, A_t, X_t; \widehat{\mu}^{a^*_0}_t, \widehat{w}_t\Big) - \varphi^a\Big(Y_t, A_t, X_t; \widehat{\mu}^a_t, \widehat{w}_t\Big)- (\mu^*_0 - \mu^a_0)\right)^2\Big| \mathcal{F}_{t-1}\right] - 1\\
&\xrightarrow{\mathrm{a.s.}} 0,
\end{align*}
Since $Y^a_{t}$ and $X_t$ are sub-exponential random variables, and the other variables, $\widehat{\mu}^a_t$ and $\widehat{w}_t$, in $\varphi^{a^*_0}\Big(Y_t, A_t, X_t; \widehat{\mu}^{a^*_0}_t, \widehat{w}_t\Big)$ and $\varphi^a\Big(Y_t, A_t, X_t; \widehat{\mu}^a_t, \widehat{w}_t\Big)$ are bounded,  we find that
\begin{align*}
    \frac{1}{T \sqrt{\widetilde{V}^a}}\sum^T_{t=1}\mathbb{E}_{P_0}\left[\left(\varphi^{a^*_0}\Big(Y_t, A_t, X_t; \widehat{\mu}^{a^*_0}_t, \widehat{w}_t\Big) - \varphi^a\Big(Y_t, A_t, X_t; \widehat{\mu}^a_t, \widehat{w}_t\Big)- (\mu^*_0 - \mu^a_0)\right)^2\Big| \mathcal{F}_{t-1}\right] - 1 
\end{align*}
is uniformly integrable from Proposition~\ref{prp:suff_uniint}. Then, from Proposition~\ref{prp:lr_conv_theorem}, for any $\delta$, there exists $T_0$ such that for all $T>T_0$
\begin{align*}
&\mathbb{E}_{P_0}\left[\left| \frac{1}{T \sqrt{\widetilde{V}^a}}\sum^T_{t=1}\mathbb{E}_{P_0}\left[\left(\varphi^{a^*_0}\Big(Y_t, A_t, X_t; \widehat{\mu}^{a^*_0}_t, \widehat{w}_t\Big) - \varphi^a\Big(Y_t, A_t, X_t; \widehat{\mu}^a_t, \widehat{w}_t\Big)- (\mu^*_0 - \mu^a_0)\right)^2\Big| \mathcal{F}_{t-1}\right] - 1 \right|\right]\\
&\le \delta.
\end{align*}
This concludes the proof.
\end{proof}

\section{Proof of Theorem~\ref{thm:fan_refine}: Large Deviation Bound for Martingales}
\label{appdx:proof_large_deviation}

For brevity, let us denote $\mathbb{P}_{P_0}$and  $\mathbb{E}_{P_0}$ by $\mathbb{P}$ and $\mathbb{E}$, respectively. For all $t = 1, \ldots, T$, let us define
\begin{align*}
    r_t(\lambda ) = \frac{\exp\left(\lambda \xi^a_t\right)}{\mathbb{E}\left[\exp\left(\lambda \xi^a_t\right)\right]}
\end{align*}
and
\begin{align*}
    \eta_t(\lambda) = \xi^a_t - b_t(\lambda),
\end{align*}
where 
\begin{align*}
   b_t(\lambda) = \mathbb{E}[r_t(\lambda)\xi^a_t].
\end{align*}
Then, we obtain the following decomposition:
\begin{align*}
    Z^a_T = U_T(\lambda) + B_T(\lambda),
\end{align*}
where
\begin{align*}
    U_T(\lambda) = \sum^T_{t=1} \eta_t(\lambda)
\end{align*}
and
\begin{align*}
    B_T(\lambda) = \sum^T_{t=1}b_t(\lambda).
\end{align*}
Let $\Psi_T(\lambda) = \sum^T_{t=1}\log\mathbb{E}\left[\exp\left(\lambda \xi^a_t\right)\right]$.

Before showing the proof of Theorem~\ref{thm:fan_refine}, we show the following lemmas. In particular, Lemma~\ref{lem:34} in Appendix~\ref{appdx:proof_large_deviation} is our novel result to bound $\mathbb{E}[\exp(\overline{\lambda}(u)\sum^T_{t=1} \xi^a_t)]/(\prod^T_{t=1}\mathbb{E}[\exp(\overline{\lambda}(u) \xi^a_t)])$.  Lemmas~\ref{lem:31}--\ref{lem:33} are modifications of the existing results of \citet{Fan2013,fan2014generalization}.

\begin{lemma}
\label{lem:31}
Under Condition~A, 
\begin{align*}
    \mathbb{E}\left[|\xi^a_t|^k \;\middle| \mathcal{F}_{t-1}\right] \leq k!\left(C_0 T^{1/2}\right)^{-k}C_1,\qquad \text{for all }\quad k \geq 2.
\end{align*}
\end{lemma}
\begin{proof}
Applying the elementary inequality $ x^k/k! \leq \exp(x),  \forall x \ge 0$,
to $x=C_0|\sqrt{T}\xi^a_t|$, for $k\geq 2$, 
\begin{align*}
   |\xi^a_t|^k \leq k!(C_0  T^{1/2})^{-k}\exp(C_0|\sqrt{T}\xi^a_t|).
\end{align*}
Taking expectations on both sides, with Condition~A, we obtain the
desired inequality. Recall that Condition~A is
\[\sup_{1\leq t \leq T}\mathbb{E}_{P_0}\left[\exp\left(C_0 \sqrt{T}\left|\xi^a_t\right|\right) \;\middle|\mathcal{F}_{t-1}\right]\leq C_1\]
for some positive constants $C_0$ and $C_1$.
\end{proof}

\begin{lemma}
\label{lem:32}
Under Condition~A, there exists some constant $C>0$ such that for all $0\leq \lambda \leq \frac{1}{4}C_0 \sqrt{T}$,
\begin{align*}
     \left|B_T(\lambda) - \lambda\right| \leq C\left(\lambda V_T + \lambda^2/\sqrt{T} \right).
\end{align*}
\end{lemma}
\begin{proof}
By definition, for $t = 1,\dots, T$,
\begin{align*}
    b_t(\lambda) &= \frac{\mathbb{E}\left[\xi^a_t\exp\left(\lambda \xi^a_t\right)\right]}{\mathbb{E}\left[\exp\left(\lambda \xi^a_t\right)\right]}.
\end{align*}
Jensen's inequality and $\mathbb{E}[\xi^a_t] = \mathbb{E}[\mathbb{E}[\xi^a_t|\mathcal{F}_{t-1}]] = 0$ implies that $\mathbb{E}[\exp(\lambda \xi^a_t)]\geq 1$ and
\begin{align*}
    &\mathbb{E}\left[\xi^a_t\exp\left(\lambda \xi^a_t\right)\right] = \mathbb{E}\left[\xi^a_t\left(\exp\left(\lambda \xi^a_t\right)- 1\right)\right] \geq 0,\qquad\mathrm{for}\ \lambda \geq 0.
\end{align*}
We find that 
\begin{align*}
B_T(\lambda) &\leq \sum^T_{t=1}\mathbb{E}[\xi^a_t\exp(\lambda \xi^a_t)]\\
&= \lambda \mathbb{E}[W_T] + \sum^T_{t=1}\sum^\infty_{k=2} \mathbb{E}\left[\frac{\xi^a_t (\lambda \xi^a_t)^k}{k!}\right],
\end{align*}
by the series expansion for $\exp(x)$. Recall that $W_T = \sum^T_{t=1}\mathbb{E}_{P_0}\left[\xi^2_{t}| \mathcal{F}_{t-1}\right]$ is the sum of the conditional second moment. Here, using Lemma~\ref{lem:31} and $\mathbb{E}\left[\xi^{k+1}_t\right] = \mathbb{E}\left[\mathbb{E}\left[\xi^{k+1}_t | \mathcal{F}_{t-1}\right]\right]$, for some constant $C_2$,
\begin{align}
    \sum^T_{t=1}\sum^\infty_{k=2}\left| \mathbb{E}\left[\frac{\xi^a_t (\lambda \xi^a_t)^k}{k!}\right]\right| &\leq \sum^T_{t=1}\sum^\infty_{k=2}\left| \mathbb{E}\left[\xi^{k+1}_t\right]\right|\frac{\lambda^k}{k!}\nonumber\\
    \label{eq:bound_lem32}
    &\leq \sum^T_{t=1}\sum^\infty_{k=2}(k+1)!\left(C_0 T^{1/2}\right)^{-(k+1)}C_1\frac{\lambda^k}{k!}\nonumber\\
    &\leq C_2\lambda^2/\sqrt{T}.
\end{align}
Therefore, 
\begin{align*}
B_T(\lambda) \leq \lambda + \lambda V_T + C_2\lambda^2/\sqrt{T}.
\end{align*}
Next, we show the lower bound of $B_T(\lambda)$. First, by using Lemma~\ref{lem:31}, using some constant $C_3>0$, for all $0\leq \lambda \leq \frac{1}{4} C_0 \sqrt{T}$,
\begin{align*}
    \mathbb{E}\left[\exp(\lambda\xi^a_t)\right] &\leq 1 + \sum^\infty_{k=2}\left|\mathbb{E}\left[\frac{(\lambda\xi^a_t)^k}{k!}\right]\right|\\
    &\leq 1 + C_1\sum^\infty_{k=2}\lambda^k(C_0 \sqrt{T})^{-k}\\
    &\leq 1 + C_3\lambda^2 T^{-1}.
\end{align*}
This inequality together with \eqref{eq:bound_lem32} implies the lower bound of $B_T(\lambda)$: for some positive constant $C_4$,
\begin{align*}
    B_T(\lambda) &=\sum^T_{t=1}\frac{\mathbb{E}\left[\xi^a_t\exp\left(\lambda \xi^a_t\right)\right]}{\mathbb{E}\left[\exp\left(\lambda \xi^a_t\right)\right]}\\
    &\geq \left(\sum^T_{t=1}\mathbb{E}[\xi^a_t \exp(\lambda \xi^a_t)]\right)\big(1 + C_3\lambda^2 T^{-1}\big)^{-1}\\
    &= \left(\lambda W_T + \sum^T_{t=1}\sum^\infty_{k=2} \mathbb{E}\left[\frac{\xi^a_t (\lambda \xi^a_t)^k}{k!}\right]\right)\big(1 + C_3\lambda^2 T^{-1}\big)^{-1}\\
    &\geq \left(\lambda W_T - \sum^T_{t=1}\sum^\infty_{k=2}\left| \mathbb{E}\left[\frac{\xi^a_t (\lambda \xi^a_t)^k}{k!}\right]\right|\right)\big(1 + C_3\lambda^2 T^{-1}\big)^{-1}\\
    &\geq \big(\lambda - \lambda V_T - C_2\lambda^2 /\sqrt{T}\big)\big(1 + C_3\lambda^2 T^{-1}\big)^{-1}\\
    &\geq \lambda - \lambda V_T - C_4 \lambda^2 /\sqrt{T}.
\end{align*}
This concludes the proof.
\end{proof}

\begin{lemma}
\label{lem:33}
Assume Condition~A. There exists some constant $C>0$ such that for all $0\leq \lambda \leq \frac{1}{4}C_0\sqrt{T}$,
\begin{align*}
     \left|\Psi_T(\lambda) - \frac{\lambda^2}{2}\right| \leq C\left(\lambda^3/\sqrt{T} + \lambda^2V_T\right).
\end{align*}
\end{lemma}
\begin{proof}
First, we have $\mathbb{E}\left[\exp(\lambda \xi^a_t)\right] \geq 1$ from Jensen's inequality.
Using the series expansion of $\log (1+\varphi)$, $\varphi\ge 0$, there exists $0 \leq \varphi^\dagger_t\leq \mathbb{E}\left[\exp(\lambda \xi^a_t)\right] - 1 \; (\text{for } t=1,\ldots, T)$ such that
\begin{align*}
    \Psi_T(\lambda) & = \log\prod^T_{t=1}\mathbb{E}\left[\exp\left(\lambda \xi^a_t\right)\right]\\
    &= \sum^T_{t=1}\left(\left(\mathbb{E}\left[\exp(\lambda \xi^a_t)\right] - 1\right) -   \frac{1}{2\left(1+\varphi^\dagger_t\right)^2}\left(\mathbb{E}\left[\exp(\lambda \xi^a_t)\right] - 1\right)^2\right).
\end{align*}
Because $(\xi^a_t)$ is a martingale difference sequence, $\mathbb{E}[\xi^a_t] = \mathbb{E}[\mathbb{E}[\xi^a_t|\mathcal{F}_{t-1}]]= 0$. Therefore,
\begin{align*}
     &\Psi_T(\lambda) - \frac{\lambda^2}{2}\mathbb{E}[W_T]
     \\
     &= \sum^T_{t=1}\left(\left(\mathbb{E}\left[\exp(\lambda \xi^a_t)\right] - 1\right) -   \frac{1}{2\left(1+\varphi^\dagger_t\right)^2}\left(\mathbb{E}\left[\exp(\lambda \xi^a_t)\right] - 1\right)^2\right) - \sum^T_{t=1}\left( \lambda \mathbb{E}[\xi^a_t] + \frac{\lambda^2}{2}\mathbb{E}[(\xi^a_t)^2]\right) 
\end{align*}
Then, by using $\mathbb{E}\left[\exp(\lambda \xi^a_t)\right] \geq 1$, we have 
\begin{align*}
    \left|\Psi_T(\lambda) - \frac{\lambda^2}{2}\mathbb{E}[W_T]\right| &\leq \sum^T_{t=1}\left|\mathbb{E}\left[\exp(\lambda \xi^a_t)\right] - 1 - \lambda \mathbb{E}[\xi^a_t] - \frac{\lambda^2}{2}\mathbb{E}[(\xi^a_t)^2]\right| + \frac{1}{2} \sum^T_{t=1}\left(\mathbb{E}\left[\exp(\lambda \xi^a_t)\right] - 1\right)^2\\
    &\leq \sum^T_{t=1}\sum^{+\infty}_{k=3}\frac{\lambda^k}{k!}\left|\mathbb{E}\left[(\xi^a_t)^k\right]\right| + \frac{1}{2} \sum^T_{t=1}\left(\sum^{+\infty}_{k=1}\frac{\lambda^k}{k!}\left|\mathbb{E}\left[(\xi^a_t)^k\right]\right|\right)^2.
\end{align*}
From Lemma~\ref{lem:31}, for a constant $C_3$,
\begin{align*}
    \left|\Psi_T(\lambda) - \frac{\lambda^2}{2}\mathbb{E}[W_T]\right| \leq C_3\lambda^3/\sqrt{T}
\end{align*}
In conclusion, we have
\begin{align*}
    \left|\Psi_T(\lambda) - \frac{\lambda^2}{2}\right| \leq C_3\lambda^3/\sqrt{T} + \frac{\lambda^2}{2}\left(\mathbb{E}[W_T-1]\right)\leq C_3\lambda^3 /\sqrt{T} + \frac{\lambda^2}{2}\mathbb{E}[|W_T-1|].
\end{align*}
Recall that $V_T = \mathbb{E}[|W_T-1|]$. Then, 
\begin{align*}
    \left|\Psi_T(\lambda) - \frac{\lambda^2}{2}\right| \leq C\left(\lambda^3/\sqrt{T} + \lambda^2V_T\right).
\end{align*}
\end{proof}
\begin{lemma}
\label{lem:34}
Assume Condition~A. For any $\varepsilon > 0$ there exists $T_0 > 0$ and some constants $\widetilde{C}_2,\widetilde{C}_3,\widetilde{C}_4 >0$ such that for all $T \geq T_0$ and $0\leq \lambda \leq \frac{1}{4}C_0\sqrt{T}$,
\begin{align*}
\frac{\mathbb{E}\left[\exp\left(\overline{\lambda}\sum^T_{t=1} \xi^a_t\right)\right]}{\prod^T_{t=1}\mathbb{E}\left[\exp\left(\overline{\lambda} \xi^a_t\right)\right]} 
&\leq \exp\left(\widetilde{C}_2\overline{\lambda}^4/T + \widetilde{C}_3\overline{\lambda}^3/\sqrt{T} + \widetilde{C}_4T_0 + \varepsilon \overline{\lambda}^2 \right).
\end{align*}
\end{lemma}
\begin{proof}
Here, we have
\begin{align*}
    \mathbb{E}\left[\exp\left(\overline{\lambda}\sum^T_{t=1} \xi^a_t\right)\right] = \mathbb{E}\left[\prod^T_{t=1}\mathbb{E}\left[\exp\left(\overline{\lambda} \xi^a_t\right)|\mathcal{F}_{t-1}\right]\right].
\end{align*}
Then, by using Lemma~\ref{lem:31}, for each $t =1, \ldots, T$,
\begin{align*}
    \mathbb{E}\left[\exp\left(\overline{\lambda} \xi^a_t\right)|\mathcal{F}_{t-1}\right]&\leq 1 + \frac{\overline{\lambda}^2}{2} \mathbb{E}\left[(\xi^a_t)^2| \mathcal{F}_{t-1}\right]  + \sum^\infty_{k=3}
    \frac{\overline{\lambda}^k\mathbb{E}\left[(\xi^a_t)^k| \mathcal{F}_{t-1}\right]}{k!}\\
    &\leq 1 + \frac{\overline{\lambda}^2}{2} \mathbb{E}\left[(\xi^a_t)^2| \mathcal{F}_{t-1}\right]  + \sum^\infty_{k=3}\overline{\lambda}^k C_1(C_0\sqrt{T})^{-k}\\
    &\leq 1 + \frac{\overline{\lambda}^2}{2} \mathbb{E}\left[(\xi^a_t)^2| \mathcal{F}_{t-1}\right]  + O\left(\overline{\lambda}^3/T^{3/2}\right).
\end{align*}
Therefore,
\begin{align*}
    \mathbb{E}\left[\exp\left(\overline{\lambda}\sum^T_{t=1} \xi^a_t\right)\right] &\leq \mathbb{E}\left[\prod^T_{t=1}\left(1 + \frac{\overline{\lambda}^2}{2} \mathbb{E}\left[(\xi^a_t)^2| \mathcal{F}_{t-1}\right] +  O\left(\overline{\lambda}^3/T^{3/2}\right)\right)\right]
    \\
    & \leq \mathbb{E}\left[\prod^T_{t=1}\exp\left(\frac{\overline{\lambda}^2}{2} \mathbb{E}\left[(\xi^a_t)^2| \mathcal{F}_{t-1}\right] +  O\left(\overline{\lambda}^3/T^{3/2}\right)\right)\right].
\end{align*}
Similarly, by using Lemma~\ref{lem:31} and constants $c, \widetilde{c}>0$, we have
\begin{align*}
    &\mathbb{E}\left[\exp\left(\overline{\lambda} \xi^a_t\right)\right]
    \\
    &= \exp\left(\log \mathbb{E}\left[\exp\left(\overline{\lambda} \xi^a_t\right)\right]\right)
    \\
    &= \exp\left(\log \left(1 +  \sum^\infty_{k=2}\mathbb{E}\left[\frac{(\overline{\lambda}\xi^a_t)^k}{k!}\right]\right)\right)
    \\
    &= \exp\left(\frac{\overline{\lambda}^2}{2} \mathbb{E}\left[(\xi^a_t)^2\right] + \sum^\infty_{k=3}\mathbb{E}\left[\frac{(\overline{\lambda}\xi^a_t)^k}{k!}\right] - \frac{1}{2}\left( \sum^\infty_{k=2}\mathbb{E}\left[\frac{(\overline{\lambda}\xi^a_t)^k}{k!}\right]\right)^2 + \frac{1}{3}\left( \sum^\infty_{k=2}\mathbb{E}\left[\frac{(\overline{\lambda}\xi^a_t)^k}{k!}\right]\right)^3 + \cdots\right)
    \\
     & \stackrel{(a)}{\geq} \exp\left(\frac{\overline{\lambda}^2}{2} \mathbb{E}\left[(\xi^a_t)^2\right] - \sum^\infty_{k=3}\mathbb{E}\left[\frac{|\overline{\lambda}\xi^a_t|^k}{k!}\right] - \frac{1}{2}\left( \sum^\infty_{k=2}\mathbb{E}\left[\frac{|\overline{\lambda}\xi^a_t|^k}{k!}\right]\right)^2 - \frac{1}{3}\left( \sum^\infty_{k=2}\mathbb{E}\left[\frac{|\overline{\lambda}\xi^a_t|^k}{k!}\right]\right)^3 + \cdots\right)
     \\
     & \stackrel{(b)}{\ge} \exp\left(\frac{\overline{\lambda}^2}{2} \mathbb{E}\left[(\xi^a_t)^2\right] - c\overline{\lambda}^3/T^{3/2} - \frac{1}{2}\left( \frac{4 C_1\overline{\lambda}^2}{3 C_0^2 T}\right)^2 - \frac{1}{3}\left( \frac{4 C_1\overline{\lambda}^2}{3 C_0^2 T}\right)^3 - \frac{1}{4}\left( \frac{4 C_1 \overline{\lambda}^2}{3 C_0^2 T}\right)^4 -  \cdots\right)
     \\
     & \ge \exp\left(\frac{\overline{\lambda}^2}{2} \mathbb{E}\left[(\xi^a_t)^2\right] - c\overline{\lambda}^3/T^{3/2} - \left( \frac{4 C_1\overline{\lambda}^2}{3 C_0^2 T}\right)^2 - \left( \frac{4 C_1\overline{\lambda}^2}{3 C_0^2 T}\right)^3 - \left( \frac{4 C_1\overline{\lambda}^2}{3 C_0^2 T}\right)^4 - \cdots\right)
     \\
     & \ge \exp\left(\frac{\overline{\lambda}^2}{2} \mathbb{E}\left[(\xi^a_t)^2\right] - c\overline{\lambda}^3/T^{3/2} - \left(\frac{4 C_1\overline{\lambda}^2}{3 C_0^2 T}\right)^2\frac{1}{1- \frac{1}{2}}\right)
     \\
     & \stackrel{(c)}{\geq} \exp\left(\frac{\overline{\lambda}^2}{2} \mathbb{E}\left[(\xi^a_t)^2\right] - c \left(\overline{\lambda}^3/\sqrt{T}\right)^3- \widetilde{c} \overline{\lambda}^4/T^2\right).
\end{align*}
For $(a)$, we used Jensen's inequality for $m = 2,3,\dots$ as
\begin{align*}
 - (-1)^{m} \frac{1}{m}\left( \sum^\infty_{k=2}\mathbb{E}\left[\frac{(\overline{\lambda}\xi^a_t)^k}{k!}\right]\right)^m \geq - \frac{1}{m}\left( \sum^\infty_{k=2}\left|\mathbb{E}\left[\frac{(\overline{\lambda}\xi^a_t)^k}{k!}\right]\right|\right)^m \ge   - \frac{1}{m}\left(\sum^\infty_{k=2}\mathbb{E}\left[\frac{|\overline{\lambda}\xi^a_t|^k}{k!}\right]\right)^m.
\end{align*}
For $(b)$, we used the fact there exist a constant $c>0$ such that
\begin{align}
    \mathbb{E} \left[\sum^\infty_{k=2}\frac{|\overline{\lambda}\xi^a_t|^k}{k!} \right]& \stackrel{(c)}{\le}  \sum^\infty_{k=2}\frac{\overline{\lambda}^k}{k!} \cdot k! C_1 \frac{1}{(C_0 \sqrt{T})^k} = C_1 \sum^\infty_{k=2} \left(\frac{\overline{\lambda}}{C_0\sqrt{T}}\right)^k = \frac{C_1\overline{\lambda}^2}{C_0^2T} \frac{1}{1 - \frac{\overline{\lambda}}{C_0 \sqrt{T}}}
    \nonumber\\
    & \stackrel{(d)}{\le}  \frac{ C_1\overline{\lambda}^2}{C_0^2T}  \frac{1}{1 - \frac{1}{4}} = \frac{4 C_1\overline{\lambda}^2}{3 C_0^2 T}\stackrel{(d)}{\le} \frac{1}{2}\nonumber,
\end{align}
and
\begin{align*}
    \mathbb{E}\left[\sum^\infty_{k=3}\frac{|\overline{\lambda}\xi^a_t|^k}{k!} \right]&  \stackrel{(c)}{\le} \sum^\infty_{k=3}\frac{\overline{\lambda}^k}{k!} \cdot k! C_1 \frac{1}{(C_0 \sqrt{T})^k} \leq c \left(\frac{\overline{\lambda}}{\sqrt{T}} \right)^3.
\end{align*}
Here, for $(c)$, we used Lemma~\ref{lem:31}, and for $(d)$, we used \eqref{eq:34}.
Then, by combining the above upper and lower bounds, with some constant $\widetilde{C}_0, \widetilde{C}_1>0$,
\begin{align}
&\frac{\mathbb{E}\left[\exp\left(\overline{\lambda}\sum^T_{t=1} \xi^a_t\right)\right]}{\prod^T_{t=1}\mathbb{E}\left[\exp\left(\overline{\lambda} \xi^a_t\right)\right]}\nonumber\\
&\leq \frac{ \mathbb{E}\left[\prod^T_{t=1}\exp\left(\frac{\overline{\lambda}^2}{2} \mathbb{E}\left[(\xi^a_t)^2| \mathcal{F}_{t-1}\right] + O\left( \left(\overline{\lambda}/\sqrt{T} \right)^3\right)\right)\right]}{\prod^T_{t=1}\exp\left(\frac{\overline{\lambda}^2}{2} \mathbb{E}\left[(\xi^a_t)^2\right]  - c \left(\overline{\lambda}^3/\sqrt{T}\right)^3- \widetilde{c} \overline{\lambda}^4/T^2 \right)}\nonumber
\\
&= \exp\left(\widetilde{C}_0\overline{\lambda}^4/T + \widetilde{C}_1\overline{\lambda}^3/\sqrt{T} \right)\mathbb{E}\left[\prod^T_{t=1}\exp\left(\overline{\lambda}^2\left(\mathbb{E}[(\xi^a_t)^2| \mathcal{F}_{t-1}] - \mathbb{E}[(\xi^a_t)^2]\right)/2 \right)\right].\nonumber
\end{align}
Using Hölder's inequality,
\begin{align}
&\frac{\mathbb{E}\left[\exp\left(\overline{\lambda}\sum^T_{t=1} \xi^a_t\right)\right]}{\prod^T_{t=1}\mathbb{E}\left[\exp\left(\overline{\lambda} \xi^a_t\right)\right]}\nonumber\\
&\leq \exp\left(\widetilde{C}_0\overline{\lambda}^4/T + \widetilde{C}_1\overline{\lambda}^3/\sqrt{T} \right)\mathbb{E}\left[\prod^T_{t=1}\exp\left(\overline{\lambda}^2\left(\mathbb{E}[(\xi^a_t)^2| \mathcal{F}_{t-1}] - \mathbb{E}[(\xi^a_t)^2]\right)/2 \right)\right]
\nonumber\\
\label{eq:jensen}
&\leq \exp\left(\widetilde{C}_0\overline{\lambda}^4/T + \widetilde{C}_1\overline{\lambda}^3/\sqrt{T} \right) \prod^T_{t=1} \left(\mathbb{E}\left[\exp\left(T\overline{\lambda}^2\left(\mathbb{E}[(\xi^a_t)^2| \mathcal{F}_{t-1}] - \mathbb{E}[(\xi^a_t)^2]\right)/2 \right)\right] \right)^{\frac{1}{T}}.
\end{align}
Note that the term
\begin{align*}
&\frac{\overline{\lambda}^2}{2}\left(\mathbb{E}[(\xi^a_t)^2| \mathcal{F}_{t-1}] - \mathbb{E}[(\xi^a_t)^2]\right)\\
&=\frac{\overline{\lambda}^2}{2T}\Bigg(\mathbb{E}\left[\left(\varphi^{a^*_0}\Big(Y_t, A_t, X_t; \widehat{\mu}^{a^*_0}_t, \widehat{w}_t\Big) - \varphi^a\Big(Y_t, A_t, X_t; \widehat{\mu}^a_t, \widehat{w}_t\Big)- (\mu^*_0 - \mu^a_0)\right)^2\Big| \mathcal{F}_{t-1}\right]\\
&\ \ \ \ \ \ - \mathbb{E}\left[\left(\varphi^{a^*_0}\Big(Y_t, A_t, X_t; \widehat{\mu}^{a^*_0}_t, \widehat{w}_t\Big) - \varphi^a\Big(Y_t, A_t, X_t; \widehat{\mu}^a_t, \widehat{w}_t\Big)- (\mu^*_0 - \mu^a_0)\right)^2 \right]\Bigg)
\end{align*}
is bounded by some constant because $\widehat{w}_t(a|X_t)$ and $\widehat{\mu}^a_t$ are bounded and $\overline{\lambda} \le \sqrt{T}\min\left\{\frac{1}{4} C_0, \sqrt{\frac{3 C_0^2}{8 C_1}}\right\}$. Then, Lemma~\ref{lem:as_conv} and Proposition~\ref{prp:lr_conv_theorem}, with probability one, as $t\to \infty$,
\begin{align*}
    & t^\alpha\Bigg|\mathbb{E}\left[\left(\varphi^{a^*_0}\Big(Y_t, A_t, X_t; \widehat{\mu}^{a^*_0}_t, \widehat{w}_t\Big) - \varphi^a\Big(Y_t, A_t, X_t; \widehat{\mu}^a_t, \widehat{w}_t\Big)- (\mu^*_0 - \mu^a_0)\right)^2\Big| \mathcal{F}_{t-1}\right]\\
    &\ \ \ \ \ \ - \mathbb{E}\left[\left(\varphi^{a^*_0}\Big(Y_t, A_t, X_t; \widehat{\mu}^{a^*_0}_t, \widehat{w}_t\Big) - \varphi^a\Big(Y_t, A_t, X_t; \widehat{\mu}^a_t, \widehat{w}_t\Big)- (\mu^*_0 - \mu^a_0)\right)^2 \right]\Bigg| 
    \\
    & \le t^\alpha\Bigg|\mathbb{E}\left[\left(\varphi^{a^*_0}\Big(Y_t, A_t, X_t; \widehat{\mu}^{a^*_0}_t, \widehat{w}_t\Big) - \varphi^a\Big(Y_t, A_t, X_t; \widehat{\mu}^a_t, \widehat{w}_t\Big)- (\mu^*_0 - \mu^a_0)\right)^2\Big| \mathcal{F}_{t-1}\right]- \sqrt{\widetilde{V}^a}\Bigg|\\
    &\ \ \ \ \ \ +  t^\alpha\left|\mathbb{E}\left[\left(\varphi^{a^*_0}\Big(Y_t, A_t, X_t; \widehat{\mu}^{a^*_0}_t, \widehat{w}_t\Big) - \varphi^a\Big(Y_t, A_t, X_t; \widehat{\mu}^a_t, \widehat{w}_t\Big)- (\mu^*_0 - \mu^a_0)\right)^2\right] -   \sqrt{\widetilde{V}^a}\right| 
    \\
    & = t^\alpha\left|\mathbb{E}\left[\left(\varphi^{a^*_0}\Big(Y_t, A_t, X_t; \widehat{\mu}^{a^*_0}_t, \widehat{w}_t\Big) - \varphi^a\Big(Y_t, A_t, X_t; \widehat{\mu}^a_t, \widehat{w}_t\Big)- (\mu^*_0 - \mu^a_0)\right)^2\Big| \mathcal{F}_{t-1}\right] - \sqrt{\widetilde{V}^a}\right|
    \\
    & \ \ \ \ \ \ +  t^\alpha\left|\mathbb{E}\left[\mathbb{E}\left[\left(\varphi^{a^*_0}\Big(Y_t, A_t, X_t; \widehat{\mu}^{a^*_0}_t, \widehat{w}_t\Big) - \varphi^a\Big(Y_t, A_t, X_t; \widehat{\mu}^a_t, \widehat{w}_t\Big)- (\mu^*_0 - \mu^a_0)\right)^2 \middle| \mathcal{F}_{t-1}\right] -  \sqrt{\widetilde{V}^a}\right]\right| 
    \\
    & \to  0,
\end{align*}
where we used the boundedness of $\mathbb{E}\left[\left(\varphi^{a^*_0}\Big(Y_t, A_t, X_t; \widehat{\mu}^{a^*_0}_t, \widehat{w}_t\Big) - \varphi^a\Big(Y_t, A_t, X_t; \widehat{\mu}^a_t, \widehat{w}_t\Big)- (\mu^*_0 - \mu^a_0)\right)^2 \middle| \mathcal{F}_{t-1}\right] -  \sqrt{\widetilde{V}^a}$ to derive the mean convergence.

Here, let us define an event $\mathcal{E}$ such that 
\begin{align*}
    \mathcal{E} = \left\{T\left(\mathbb{E}[(\xi^a_t)^2| \mathcal{F}_{t-1}] - \mathbb{E}[(\xi^a_t)^2]\right) \to 0\ \mathrm{as}\ t\to \infty\right\}, 
\end{align*}
which occurs with probability one. Without loss of generality, we assume that $\alpha \neq 1$. On the event $\mathcal{E}$, for all $\varepsilon > 0$, there exists $T_0 \geq 0$ such that for all $T > T_0$,
\begin{align*}
    \exp\left(T\overline{\lambda}^2\left(\mathbb{E}[(\xi^a_t)^2| \mathcal{F}_{t-1}] - \mathbb{E}[(\xi^a_t)^2]\right)/2 \right) \leq \exp\left(\overline{\lambda}^2\varepsilon / t^\alpha \right).
\end{align*}
Because this event occurs with probability one and $\mathbb{E}[(\xi^a_t)^2| \mathcal{F}_{t-1}] - \mathbb{E}[(\xi^a_t)^2]$ is bounded, for all $\varepsilon > 0$, there exists $T_0$ such that for all $t > T_0$,
\begin{align*}
    \mathbb{E}\left[\exp\left(T\overline{\lambda}^2\left(\mathbb{E}[(\xi^a_t)^2| \mathcal{F}_{t-1}] - \mathbb{E}[(\xi^a_t)^2]\right)/2 \right)\right] \leq \exp\left(\overline{\lambda}^2\varepsilon \right). 
\end{align*}
From this result, for all $t > T_0$,
\begin{align*}
    \mathbb{E}\left[\exp\left(T\overline{\lambda}^2\left(\mathbb{E}[(\xi^a_t)^2| \mathcal{F}_{t-1}] - \mathbb{E}[(\xi^a_t)^2]\right)/2 \right)\right]^{1/T} \leq \exp\left(\overline{\lambda}^2\varepsilon/(Tt^\alpha) \right). 
\end{align*}
Therefore, in \eqref{eq:jensen}, from the boundedness of the random variables, for a constant $C>0$
\begin{align*}
&\prod^T_{t=1} \left(\mathbb{E}\left[\exp\left(T\overline{\lambda}^2\left(\mathbb{E}[(\xi^a_t)^2| \mathcal{F}_{t-1}] - \mathbb{E}[(\xi^a_t)^2]\right)/2 \right)\right] \right)^{\frac{1}{T}}\\
&\leq \prod^T_{t=T_0 + 1}\exp\left(\overline{\lambda}^2\varepsilon/(Tt^\alpha)  \right)\prod^{T_0}_{t=1} \mathbb{E}\left[\exp\left(T\overline{\lambda}^2\left(\mathbb{E}[(\xi^a_t)^2| \mathcal{F}_{t-1}] - \mathbb{E}[(\xi^a_t)^2]\right)/2 \right)\right]^{1/T}\\
&\leq \prod^T_{t=T_0 + 1}\exp\left(\overline{\lambda}^2\varepsilon/(Tt^\alpha)  \right)\prod^{T_0}_{t=1} \exp(C)\\
&\leq \exp\left(\overline{\lambda}^2\varepsilon  T^{-\alpha} / (1 - \alpha) + CT_0 \right).
\end{align*}
In summary, for any $\varepsilon > 0$ and some constants $\widetilde{C}_2, \widetilde{C}_3, \widetilde{C}_4>0$, there exists $T_0 > 0$ such that for all $T \geq T_0$, 
\begin{align*}
\frac{\mathbb{E}\left[\exp\left(\overline{\lambda}\sum^T_{t=1} \xi^a_t\right)\right]}{\prod^T_{t=1}\mathbb{E}\left[\exp\left(\overline{\lambda} \xi^a_t\right)\right]} 
&\leq \exp\left(\widetilde{C}_2\overline{\lambda}^4/T + \widetilde{C}_3\overline{\lambda}^3/\sqrt{T} + \widetilde{C}_4T_0 + \varepsilon \overline{\lambda}^2 \right).
\end{align*}
\end{proof}

By using Lemmas~\ref{lem:31}--\ref{lem:34}, we show the proof of Theorem~\ref{thm:fan_refine}.
\begin{proof}[Proof of Theorem~\ref{thm:fan_refine}]  
There exists some constant $C>0$ such that for all $1\leq u \leq \sqrt{T}\min\left\{\frac{1}{4} C_0, \sqrt{\frac{3 C_0^2}{8 C_1}}\right\}$,
\begin{align*}
&\mathbb{P}\left(Z^a_T > u\right)
\\
&=
\int\left(\prod^T_{t=1}\frac{\exp\left(\lambda \xi^a_t\right)}{\mathbb{E}\left[\exp\left(\lambda \xi^a_t\right)\right]}\right)\left(\prod^T_{t=1}\frac{\exp\left(\lambda \xi^a_t\right)}{\mathbb{E}\left[\exp\left(\lambda \xi^a_t\right)\right]}\right)^{-1}\mathbbm{1}[Z^a_T > u]\mathrm{d}\mathbb{P}
\\
&=
\int\left(\prod^T_{t=1}\frac{\exp\left(\lambda \xi^a_t\right)}{\mathbb{E}\left[\exp\left(\lambda \xi^a_t\right)\right]}\right) \exp\left(-\lambda \sum_{t=1}^T\xi^a_t + \log \left(\prod_{t=1}^T\mathbb{E}[\exp(\lambda \xi^a_t)]\right)\right)\mathbbm{1}[Z^a_T > u]\mathrm{d}\mathbb{P}
\\
\\
&=
\int\left(\prod^T_{t=1}\frac{\exp\left(\lambda \xi^a_t\right)}{\mathbb{E}\left[\exp\left(\lambda \xi^a_t\right)\right]}\right)\exp\left(-\lambda Z^a_T + \Psi_T(\lambda)\right)\mathbbm{1}[Z^a_T > u]\mathrm{d}\mathbb{P}
\\
&=
\int\left(\prod^T_{t=1}\frac{\exp\left(\lambda \xi^a_t\right)}{\mathbb{E}\left[\exp\left(\lambda \xi^a_t\right)\right]}\right)\exp\left(-\lambda U_T(\lambda) - \lambda B_T(\lambda) + \Psi_T(\lambda)\right)\mathbbm{1}[U_T(\lambda) + B_T(\lambda) > u]\mathrm{d}\mathbb{P},
\\
& \le \int\left(\prod^T_{t=1}\frac{\exp\left(\lambda \xi^a_t\right)}{\mathbb{E}\left[\exp\left(\lambda \xi^a_t\right)\right]}\right)\exp\left(-\lambda U_T(\lambda)  -\frac{\lambda^2}{2} + C(\lambda^3 /\sqrt{T}+ \lambda^2 V_T)\right)
\\
& \qquad \qquad \qquad \cdot \mathbbm{1} \left[U_T(\lambda) + \lambda + C(\lambda V_T + \lambda^2 /\sqrt{T}) > u\right]\mathrm{d}\mathbb{P},
\end{align*}
where for the last inequality, we used Lemma~\ref{lem:32} and Lemma~\ref{lem:33}.
Let $\overline{\lambda} = \overline{\lambda}(u)$ be the largest solution of the equation
\begin{align*}
    \lambda + C(\lambda V_T + \lambda^2 /\sqrt{T}) = u.
\end{align*}
The definition of $\overline{\lambda}$ implies that there exist $C' > 0$ such that, for all $1\leq u \leq \sqrt{T}\min\left\{\frac{1}{4} C_0, \sqrt{\frac{3 C_0^2}{8 C_1}}\right\}$,
\begin{align}
\label{eq:34}
    C' u \leq \overline{\lambda}(u) = \frac{2u}{\sqrt{(1+CV_T)^2 + 4C u /\sqrt{T}} + C V_T +1} \leq u
\end{align}
and there exists $\theta \in (0, 1]$ such that 
\begin{align}
    \overline{\lambda}(u) & = u - C  (\overline{\lambda} V_T + \overline{\lambda}^2/\sqrt{T}) 
    \nonumber\\
    \label{eq:35}
    & = u - C \theta (u V_T + u^2 /\sqrt{T}) \in \left[C', \sqrt{T}\min\left\{\frac{1}{4} C_0, \sqrt{\frac{3 C_0^2}{8 C_1}}\right\}\right].
\end{align}
Then, we obtain for all $1\leq u \leq \sqrt{T}\min\left\{\frac{1}{4} C_0, \sqrt{\frac{3 C_0^2}{8 C_1}}\right\}$,
\begin{align*}
    & \mathbb{P}\left(Z^a_T > u\right) 
    \\
    & \leq \exp\left(C\left(\overline{\lambda}^3T^{-1/2} + \overline{\lambda}^2 V_T \right) - \overline{\lambda}^2/2\right)\int\left(\prod^T_{t=1}\frac{\exp\left(\overline{\lambda} \xi^a_t\right)}{\mathbb{E}\left[\exp\left(\overline{\lambda} \xi^a_t\right)\right]}\right)\exp\left(-\overline{\lambda} U_T(\overline{\lambda})\right)\mathbbm{1}[U_T(\overline{\lambda}) > 0]\mathrm{d}\mathbb{P}.
\end{align*}
Here, we have
\begin{align*}
    &\int\left(\prod^T_{t=1}\frac{\exp\left(\overline{\lambda} \xi^a_t\right)}{\mathbb{E}\left[\exp\left(\overline{\lambda} \xi^a_t\right)\right]}\right)\exp\left(-\overline{\lambda} U_T(\overline{\lambda})\right)\mathbbm{1}[U_T(\overline{\lambda}) > 0]\mathrm{d}\mathbb{P}\\
    &=\mathbb{E}\left[ \prod^T_{t=1}\frac{\exp\left(\overline{\lambda} \xi^a_t\right)}{\mathbb{E}\left[\exp\left(\overline{\lambda} \xi^a_t\right)\right]}\exp\left(-\overline{\lambda} U_T(\overline{\lambda})\right)\mathbbm{1}[U_T(\overline{\lambda}) > 0]\right].
\end{align*}
We also define another measure $\widetilde{\mathbb{P}}_\lambda $ as
\begin{align*}
    \mathrm{d}\widetilde{\mathbb{P}}_\lambda = \frac{\prod^T_{t=1}\exp\left(\lambda \xi^a_t\right)}{\mathbb{E}\left[\exp\left(\lambda\sum^T_{t=1} \xi^a_t\right)\right]}\mathrm{d}\mathbb{P}= \frac{\exp\left(\lambda \sum^T_{t=1}\xi^a_t\right)}{\mathbb{E}\left[\exp\left(\lambda\sum^T_{t=1} \xi^a_t\right)\right]}\mathrm{d}\mathbb{P}.
\end{align*}
Note that $\widetilde{\mathbb{P}}_\lambda $ is a probability measure, as the following holds
\begin{align*}
\int \mathrm{d}\widetilde{\mathbb{P}}_\lambda & = \int \frac{\exp\left(\lambda \sum^T_{t=1}\xi^a_t\right)}{\mathbb{E}\left[\exp\left(\lambda\sum^T_{t=1} \xi^a_t\right)\right]}\mathrm{d}\mathbb{P}
\\
& = \frac{1}{\mathbb{E}\left[\exp\left(\lambda\sum^T_{t=1} \xi^a_t\right)\right]} \int \exp\left(\lambda \sum^T_{t=1}\xi^a_t\right)  \mathrm{d}\mathbb{P}
\\
& = \frac{1}{\mathbb{E}\left[\exp\left(\lambda\sum^T_{t=1} \xi^a_t\right)\right]} \cdot \mathbb{E}\left[\exp\left(\lambda\sum^T_{t=1} \xi^a_t\right)\right]
\\
& = 1.
\end{align*}

We further denote $\widetilde{\mathbb{E}}_\lambda $ as the expectation under the measure $\widetilde{\mathbb{P}}_\lambda $.
In the same way as (37) and (38) in \citet{Fan2013}, it is easy to see that
\begin{align}
    &\mathbb{E}\left[ \prod^T_{t=1}\frac{\exp\left(\overline{\lambda} \xi^a_t\right)}{\mathbb{E}\left[\exp\left(\overline{\lambda} \xi^a_t\right)\right]}\exp\left(-\overline{\lambda} U_T(\overline{\lambda})\right)\mathbbm{1}[U_T(\overline{\lambda}) > 0]\right]
    \nonumber\\
    & = \frac{\mathbb{E}[\exp(\overline{\lambda}\sum_{t=1}^T \xi^a_t)]}{\prod^T_{t=1} \mathbb{E}\left[\exp\left(\overline{\lambda} \xi^a_t\right)\right]} 
    \mathbb{E}\left[ \frac{ \prod^T_{t=1}\exp\left(\overline{\lambda} \xi^a_t\right)}{\mathbb{E}[\exp(\overline{\lambda}\sum_{t=1}^T \xi^a_t)]}\exp\left(-\overline{\lambda} U_T(\overline{\lambda})\right)\mathbbm{1}[U_T(\overline{\lambda}) > 0]\right]
    \nonumber\\
    & = \frac{\mathbb{E}[\exp(\overline{\lambda}\sum_{t=1}^T \xi^a_t)]}{\prod^T_{t=1} \mathbb{E}\left[\exp\left(\overline{\lambda} \xi^a_t\right)\right]} \widetilde{\mathbb{E}}_{\overline{\lambda}}[\exp\left(-\overline{\lambda} U_T(\overline{\lambda})\right)\mathbbm{1}[U_T(\overline{\lambda}) > 0]]
    \nonumber\\
    & = \frac{\mathbb{E}\left[\exp\left(\lambda\sum^T_{t=1} \xi^a_t\right)\right]}{\prod^T_{t=1}\mathbb{E}\left[\exp\left(\lambda \xi^a_t\right)\right]} \int^\infty_0\overline{\lambda} \exp(-\overline{\lambda}y)\widetilde{\mathbb{P}}_{\overline{\lambda}}(0< U_T(\overline{\lambda}) < y) \mathrm{d}y, \label{eq:fans_37}.
\end{align}
Besides, for a standard Gaussian random variable $\mathcal{N}$, 
\begin{align}
        &\mathbb{E}\left[\exp\left(-\overline{\lambda} \mathcal{N}\right)\mathbbm{1}[\mathcal{N} > 0]\right] = \int^\infty_0\overline{\lambda} \exp(-\overline{\lambda}y)\mathbb{P}(0< \mathcal{N} < y) \mathrm{d}y.\label{eq:fans_38}
\end{align}
Then, from \eqref{eq:fans_37} and \eqref{eq:fans_38}, 
\begin{align*}
    \left| \widetilde{\mathbb{E}}_{\overline{\lambda}}[\exp\left(-\overline{\lambda} U_T(\overline{\lambda})\right)\mathbbm{1}[U_T(\overline{\lambda}) > 0]] - \mathbb{E}\left[\exp\left(-\overline{\lambda} \mathcal{N}\right)\mathbbm{1}[\mathcal{N} > 0]\right]\right| \le 2 \sup_{g} \left|\widetilde{\mathbb{P}}_{\overline{\lambda}}\left(U_T(\overline{\lambda})\leq g\right) - \Phi(g)\right|
\end{align*}
Therefore,
\begin{align*}
    &\mathbb{P}\left(Z^a_T > u\right)
    \\
    & \le \frac{\mathbb{E}\left[\exp\left(\lambda\sum^T_{t=1} \xi^a_t\right)\right]}{\prod^T_{t=1}\mathbb{E}\left[\exp\left(\lambda \xi^a_t\right)\right]} \exp\left(C\left(\overline{\lambda}^3/\sqrt{T} + \overline{\lambda}^2 V_T \right) - \overline{\lambda}^2/2\right)
    \widetilde{\mathbb{E}}_{\overline{\lambda}}[\exp\left(-\overline{\lambda} U_T(\overline{\lambda})\right)\mathbbm{1}[U_T(\overline{\lambda}) > 0]]
    \\
    &\leq \frac{\mathbb{E}\left[\exp\left(\lambda\sum^T_{t=1} \xi^a_t\right)\right]}{\prod^T_{t=1}\mathbb{E}\left[\exp\left(\lambda \xi^a_t\right)\right]} \exp\left(C\left(\overline{\lambda}^3/\sqrt{T} + \overline{\lambda}^2 V_T \right) - \overline{\lambda}^2/2\right) 
    \\
    & \qquad \qquad \qquad \times \left(\mathbb{E}\left[\exp\left(-\overline{\lambda} \mathcal{N}\right)\mathbbm{1}[\mathcal{N} > 0]\right] + 2\sup_{g}\left|\widetilde{\mathbb{P}}_{\overline{\lambda}}\left(U_T(\overline{\lambda})\leq g\right) - \Phi(g)\right|\right)
    \\
    &\leq \frac{\mathbb{E}\left[\exp\left(\lambda\sum^T_{t=1} \xi^a_t\right)\right]}{\prod^T_{t=1}\mathbb{E}\left[\exp\left(\lambda \xi^a_t\right)\right]}\exp\left(C\left(\overline{\lambda}^3/\sqrt{T} + \overline{\lambda}^2 V_T \right) - \overline{\lambda}^2/2\right)\left(\mathbb{E}\left[\exp\left(-\overline{\lambda} \mathcal{N}\right)\mathbbm{1}[\mathcal{N} > 0]\right] + 2 \right).
\end{align*}
Here,
\begin{align*}
\exp\left(- \overline{\lambda}^2/2\right)\mathbb{E}\left[\exp\left(-\overline{\lambda} \mathcal{N}\right)\mathbbm{1}[\mathcal{N} > 0]\right] = \frac{1}{\sqrt{2\pi}} \int^\infty_0 \exp\left(-(y+ \overline{\lambda})^2\right)\mathrm{d}y = 1 - \Phi(\overline{\lambda}).
\end{align*}
From (41) of \citet{Fan2013}, for all $\overline{\lambda} \ge C'$, we have
\begin{align*}
\frac{C'}{\sqrt{2 \pi}(1 + C')}\frac{1}{\overline{\lambda}}\exp\left( - \frac{\overline{\lambda}^2}{2}\right) \leq 1 - \Phi(\overline{\lambda}).
\end{align*}
Therefore, with some constant $\widetilde{C}$, for all $1\leq u \leq \sqrt{T}\min\left\{\frac{1}{4} C_0, \sqrt{\frac{3 C_0^2}{8 C_1}}\right\}$,
\begin{align}
    \mathbb{P}\left({Z}_T > u\right) &\leq\frac{\mathbb{E}\left[\exp\left(\lambda\sum^T_{t=1} \xi^a_t\right)\right]}{\prod^T_{t=1}\mathbb{E}\left[\exp\left(\lambda \xi^a_t\right)\right]} \left\{
    \Big(1-\Phi(\overline{\lambda})\Big) + \overline{\lambda} \Big(1-\Phi(\overline{\lambda})\Big) c \right\}\exp\left(C\left(\overline{\lambda}^3/\sqrt{T} + \overline{\lambda}^2 V_T \right)\right)
   \nonumber\\
    &\leq \frac{\mathbb{E}\left[\exp\left(\lambda\sum^T_{t=1} \xi^a_t\right)\right]}{\prod^T_{t=1}\mathbb{E}\left[\exp\left(\lambda \xi^a_t\right)\right]}
    \widetilde{C}\overline{\lambda}
    \Big(1-\Phi(\overline{\lambda})\Big)\exp\left(C\left(\overline{\lambda}^3/\sqrt{T} + \overline{\lambda}^2 V_T \right)\right),\label{eq:upper_bound_wo_frac}
\end{align}
where $c = \sqrt{2\pi} (1+C') / C'$, and $\widetilde{C}$ is chosen to be $\widetilde{C} \overline{\lambda} \geq (1 + \overline{\lambda} c)$ (Note that $\overline{\lambda} \ge C'$ from \eqref{eq:34}).

From Lemma~\ref{lem:34}, for any $\varepsilon > 0$, there exists $T_0 > 0$ such that for all $T \geq T_0$, 
\begin{align}
\label{eq:upper_bound_frac_fin}
\frac{\mathbb{E}\left[\exp\left(\overline{\lambda}\sum^T_{t=1} \xi^a_t\right)\right]}{\prod^T_{t=1}\mathbb{E}\left[\exp\left(\overline{\lambda} \xi^a_t\right)\right]} 
&\leq \exp\left(\widetilde{C}_2\overline{\lambda}^4/T + \widetilde{C}_3\overline{\lambda}^3/\sqrt{T} + \widetilde{C}_4T_0 + \varepsilon \overline{\lambda}^2 \right).
\end{align}

In summary, by \eqref{eq:upper_bound_wo_frac} and \eqref{eq:upper_bound_frac_fin}, for all $1\leq u \leq \sqrt{T}\min\left\{\frac{1}{4} C_0, \sqrt{\frac{3 C_0^2}{8 C_1}}\right\}$,
\begin{align}
\label{eq:42}
    \frac{\mathbb{P}\left({Z}_T > u\right)}{1-\Phi(\overline{\lambda})} &\leq  \widetilde{C}\overline{\lambda} \exp\left(\widetilde{C}_2\overline{\lambda}^4/T + \widetilde{C}_3\overline{\lambda}^3/\sqrt{T}  + C\left(\overline{\lambda}^3/\sqrt{T} + \overline{\lambda}^2 V_T + T_0\right) + \varepsilon \overline{\lambda}^2\right).
\end{align}
Next, we compare $1-\Phi(\overline{\lambda})$ with $1-\Phi(u)$.
Recall the following upper bound and lower bound on $1 - \Phi(x) = \Phi(-x)$:
\begin{align*}
    \frac{1}{\sqrt{2 \pi} (1+x)} \exp\left(-\frac{x^2}{2}\right) \le \Phi(-x) \le \frac{1}{\sqrt{ \pi} (1+x)} \exp\left(-\frac{x^2}{2}\right), \; x\ge 0.
\end{align*}
For all $1\leq u \leq \sqrt{T}\min\left\{\frac{1}{4} C_0, \sqrt{\frac{3 C_0^2}{8 C_1}}\right\}$,
\begin{align*}
1 & \leq \frac{\int^\infty_{\overline{\lambda}}\exp(-t^2/2)\mathrm{d}t}{\int^\infty_u \exp(-t^2/2)\mathrm{d}t} 
\\
& \le \frac{\frac{1}{\sqrt{\pi}(1+ \overline{\lambda})} \exp(- \overline{\lambda}^2/2)}{\frac{1}{\sqrt{2\pi}(1+ u)} \exp(- u^2/2)}
\\
& = \sqrt{2}\frac{1+u}{1+ \overline{\lambda}}\exp((u^2 - \overline{\lambda}^2)/2).
\end{align*}
From \eqref{eq:35}, we have
\begin{align*}
    u^2 - \overline{\lambda}^2 & = (u + \overline{\lambda})(u - \overline{\lambda})
    \\
    &  \le  2u  (C \theta (u V_T + u^2 /\sqrt{T}))
    \\
    &  = 2C \theta (u^2 V_T + u^3 /\sqrt{T}).
\end{align*}
Therefore, with some constant $\widetilde{C}_4 > 0$
\begin{align*}
\frac{\int^\infty_{\overline{\lambda}}\exp(-t^2/2)\mathrm{d}t}{\int^\infty_u \exp(-t^2/2)\mathrm{d}t} 
&\leq \exp\left(\widetilde{C}_4 \left(u^2 V_T + u^3 /\sqrt{T}\right)\right).
\end{align*}
We find that
\begin{align}
\label{eq:44}
1 - \Phi(\overline{\lambda}) \le \big(1 - \Phi(u)\big)\exp\left( \widetilde{C}_4\left(u^2 V_T + u^3 /\sqrt{T}\right)\right).
\end{align}
By combining \eqref{eq:42}, \eqref{eq:44}, and \eqref{eq:34}, for any $\varepsilon > 0$ all $1\leq u \leq \sqrt{T}\min\left\{\frac{1}{4} C_0, \sqrt{\frac{3 C_0^2}{8 C_1}}\right\}$, there exist $T_0 > 0$ and $\widetilde{C}_5>0$ such that for all $T \geq T_0$,
\begin{align*}
    &\frac{\mathbb{P}\left({Z}_T > u\right)}{1-\Phi(u)}\\
    &\leq  \widetilde{C} \overline{\lambda} \exp\left( C\left(\overline{\lambda}^3/\sqrt{T} + \overline{\lambda}^2 V_T \right) + \widetilde{C}_2\overline{\lambda}^4/T + \widetilde{C}_3\overline{\lambda}^3/\sqrt{T} + \widetilde{C}_4 \left( u^2 V_T + u^3 /\sqrt{T}+ T_0\right) + \varepsilon u^2 \right)
    \\
    &\le \widetilde{C}u \exp\left( \widetilde{C}_5 \left( u^2 V_T + u^3 /\sqrt{T} + u^4/T+ T_0\right) + \varepsilon u^2 \right)\\
    &= \widetilde{C}u \exp\left( \widetilde{C}_5 \left( u^2 (V_T + \varepsilon) + u^3 /\sqrt{T} + u^4/T+ T_0\right) \right).
\end{align*}
Applying the same argument to the martingale $-Z^a_T$, we conclude the proof.
\end{proof}

\end{document}